%% file: bare_jrnl_compsoc.tex
\documentclass[10pt,journal,compsoc]{IEEEtran}

\ifCLASSOPTIONcompsoc
  \usepackage[nocompress]{cite}
\else
  \usepackage{cite}
\fi

\ifCLASSINFOpdf

\else

\fi

\usepackage{amsmath,amssymb,amsfonts}
\usepackage{BOONDOX-cal}
\usepackage{mathrsfs}
\usepackage{bm}
\usepackage{soul}
\usepackage{algorithm}
\usepackage{algorithmic}
\usepackage{multirow}
\usepackage{verbatimbox}
\usepackage{booktabs}
\usepackage{graphicx}
\usepackage{subfigure}
\usepackage{siunitx}
\usepackage{empheq}
\usepackage{amsthm}
\usepackage{lscape}
\usepackage{graphicx}
\usepackage{textcomp}
\usepackage{xcolor}
\usepackage[hidelinks]{hyperref}
\def\BibTeX{{\rm B\kern-.05em{\sc i\kern-.025em b}\kern-.08em
    T\kern-.1667em\lower.7ex\hbox{E}\kern-.125emX}}
\newtheorem{theorem}{Theorem}
\newtheorem{corollary}[theorem]{Corollary}
\newtheorem{lemma}[theorem]{Lemma}
\newtheorem{definition}{Definition}
\newtheorem{example}{Example}

\newtheorem{proposition}[theorem]{Proposition}
\newtheorem{assumption}{Assumption}

\begin{document}
\title{Robust Orthogonal Machine Learning\\ of Treatment Effects}

%\author{Anonymous Authors}

\author{Yiyan Huang, Cheuk Hang Leung, Qi Wu, and Xing Yan
\IEEEcompsocitemizethanks{\IEEEcompsocthanksitem Yiyan Huang, Cheuk Hang Leung, Qi Wu are with the School of Data Science, City University of Hong Kong, 
Kowloon, Hong Kong SAR (e-mail: yiyhuang3-c@my.cityu.edu.hk, \{chleung87,qiwu55\}@cityu.edu.hk).
%\IEEEcompsocthanksitem J. Doe and J. Doe are with Anonymous University.}% <-this % stops an unwanted space
\IEEEcompsocthanksitem Xing Yan is with the Institute of Statistics and Big Data, Renmin University of China, Beijing, China, 100872 (e-mail: xingyan@ruc.edu.cn).
}
}

\markboth{Robust Orthogonal Machine Learning of Treatment Effects}%
{Shell \MakeLowercase{\textit{et al.}}: Bare Demo of IEEEtran.cls for Computer Society Journals}

\IEEEtitleabstractindextext{%
\begin{abstract}
Causal learning is the key to obtaining stable predictions and answering \textit{what if} problems in decision-makings. In causal learning, it is central to seek methods to estimate the average treatment effect (ATE) from observational data. The Double/Debiased Machine Learning (DML) is one of the prevalent methods to estimate ATE. However, the DML estimators can suffer from an \textit{error-compounding issue} and even give extreme estimates when the propensity scores are close to 0 or 1. Previous studies have overcome this issue through some empirical tricks such as propensity score trimming, yet none of the existing works solves it from a theoretical standpoint. In this paper, we propose a \textit{Robust Causal Learning (RCL)} method to offset the deficiencies of DML estimators. Theoretically, the RCL estimators i) satisfy the (higher-order) orthogonal condition and are as \textit{consistent and doubly robust} as the DML estimators, and ii) get rid of the error-compounding issue. Empirically, the comprehensive experiments show that: i) the RCL estimators give more stable estimations of the causal parameters than DML; ii) the RCL estimators outperform traditional estimators and their variants when applying different machine learning models on both simulation and benchmark datasets, and a mimic consumer credit dataset generated by WGAN.
\end{abstract}

}

\maketitle

\IEEEdisplaynontitleabstractindextext
\IEEEpeerreviewmaketitle

\input{mainContent.tex}

\appendices
\input{appendix1.tex}

%\section{}

%\ifCLASSOPTIONcompsoc
%  % The Computer Society usually uses the plural form
%\else
%  % regular IEEE prefers the singular form
%\fi

\ifCLASSOPTIONcaptionsoff
  \newpage
\fi

\bibliographystyle{IEEEtran}
\bibliography{Bibliography-MM-MC}

\clearpage
\onecolumn
\input{appendix.tex}

\end{document}

%% file: mainContent.tex
\IEEEraisesectionheading{\section{Introduction}\label{sec:introduction}}
\IEEEPARstart{H}{ow} to construct counterfactuals and estimate treatment effects correctly is the key to obtaining stable predictions and answering \textit{what if} questions when one only has observational data. It is ubiquitous in many applications such as in healthcare
%\cite{belloni2014inference, farrell2015robust, chernozhukov2018double} 
and business 
%\cite{glass2013causal, hill2013assessing, alaa2017bayesian} 
when decision makers do not have the luxury to conduct randomized controlled trials (RCTs) because they are either costly, or time-consuming, or simply can not be done. There are three challenges in utilizing observational data to estimate the treatment-outcome relation. The first one is the selection bias. Selection bias arises during the selection process when subjects are differentially included or excluded from treatments. The control of selection bias requires data on all factors relevant to the selection mechanism for treatment. The second one is the confounding bias. Confounding bias results from inadequate adjustment of covariates that are simultaneously predictive of treatment and outcome. The control of confounding bias requires data on all covariates relevant to the outcome and the treatment assignment mechanism.  The third one is the unknown and potentially nonlinear relationship between treatment and outcome, between treatment and covariates, and between outcome and covariates.  
% pi equation is for control of selection bias, using covariates specific for the selection mechanisum
% g equation is the outcome quation, when covariates specific for the selection mechanisum are included, control of confounding bias
% allowing pi and g to be of any model is to allow nonlinearity
% moment condition -> ATE estimator is regularization unbaised w.r.t. true pi and true g
% Neymann orthognal condition -> ATE estimator is in-sensitive (robust) to mis-specification
% new score form -> error-compounding at 0, 1.

In healthcare applications, the selection mechanism for including or excluding a patient for treatment in RCTs is likely constrained by medical and ethical, and the outcomes are subject to covariates that are largely overlapping with those used in the selection process. For example, the Infant Health and Development Program (IHDP) \cite{hill2011bayesian} in the United States conducted experiments to evaluate the efficacy of early intervention in reducing the developmental and health problems of low birth weight and premature infants. However, there is a wide range of variables (covariates) that can influence both infants' cognitive outcomes and whether or not they receive treatments, including pregnancy complications, child's gender, household composition, day care arrangements, source of health care, quality of the home environment, parents' race and ethnicity, and maternal age, education, IQ, and employment. They managed to obtain the study sample of infants which was stratified by birth weight (2,000 grams or less, 2,001-2,500 grams) and randomized to the Intervention Group or the Follow-Up Group. The goal is to study the treatment effect of the specialist visits (binary treatment) on the cognitive scores of children (continuous-valued outcome). 

In an example of business applications, many e-commerce platforms provide shopping credit in the form of consumer loans to shoppers who frequent their marketplaces. In this new paradigm of consumer financing, platform lenders would like to know the precise impact of adjusting credit policies (the treatment), such as credit lines and loan interests, on shoppers' spending behavior. ``It is an open question to what extent FinTech lenders are using randomization (or other techniques to recover causal effects) in order to improve the accuracy of their prediction models out-of-sample", as the authors of \cite{berg2021fintech} put it. This is because conducting RCTs by boosting shoppers' credit lines is costly and waiting for changes in spending behavior to reveal is time-consuming. However, from observational data, it is difficult to assess whether the changes in spending amount (the outcome) come from covariates such as customer's age, gender, region, indebtedness, or seasonal effects, external shocks, or are indeed the results of credit policy changes (the treatment), not to mention that, in practice, similar covariates are also used by algorithms to select customers for adjusting credit lines and by how much. 

In the recent literature, various machine learning-based methods have been widely adopted to estimate treatment effects. For example, researchers have suggested (i) classical linear models, e.g., Ordinary Least Square (OLS), and Least Absolute Shrinkage and Selection Operator (LASSO) \cite{belloni2014inference, bloniarz2016lasso}; (ii) tree models and their ensembles, e.g.,  Regression Tree (RT) \cite{athey2016recursive} and Causal Forests (CF) \cite{wager2018estimation, athey2019generalized}; (iii) neural network models, e.g., TARNet \cite{shalit2017estimating} and Dragonnet \cite{shi2019adapting}. Other advanced approaches involve representation learning \cite{louizos2017causal, assaad2021counterfactual, yao2018representation, zhang2020learning}, adversarial learning \cite{yoon2018ganite, du2021adversarial}, and Bayesian learning \cite{hill2011bayesian, taddy2016nonparametric, alaa2017bayesian, ray2019debiased}.

In general, causal machine learning methods concerning ATE estimations mainly include regression adjustment methods and re-weighting methods (see more details in \cite{10.1145/3444944}). Regression adjustment methods require an estimated feature-outcome relation (the outcome model) and directly average the predicted potential outcomes over the whole population to estimate ATE, so the associated estimator is called the direct regression (DR) estimator. The shortcoming of the DR estimator is that it overlooks the probabilistic impact of the covariates on the treatment assignment (i.e., the propensity score) and hence often results in biased estimations of ATE unless the outcome model is estimated accurately. Re-weighting methods mimic the principle of RCTs to make the re-weighted instances look like they receive alternative treatment. The Inverse Probability Weighting (IPW) is one of the prevalent re-weighting strategies. It involves the propensity scores rather than the outcome model. However, the IPW estimator is sensitive to the estimation of propensity scores, thus easily leading to high variance or even infinite estimates. Such occasion often occurs when the estimated propensity scores are close to 0 or 1. This is called the \textit{error-compounding issue}.

The Debiased Machine Learning (DML) method, which is exploited by \cite{chernozhukov2018double} based on \cite{neyman1979c}, has become more and more popular and been applied in different fields since its appearance. Examples include the use of the DML estimators when the data is of panel type in \cite{semenova2018orthogonal}, the application of DML for programme evaluation in \cite{knaus2020double},
and the study of the action-response effects under the credit risk context in \cite{huang2021causal}. Besides, researchers elaborate and investigate the DML estimators further. For example, DML has been extended for continuous treatment effects \cite{klosin2021automatic}, dynamic treatment effects \cite{lewis2021double},  and instrumental variable quantile regressions \cite{chen2021debiased}.
%the authors in \cite{chernozhukov2018double} give out the orthogonal score functions for the treatment effect when the underlying model is partially linear or fully non-linear with binary treatment. In \cite{chernozhukov2018double2}, the authors give a methodology to construct the orthogonal score functions when the causal parameter is of the expected type. 
In addition, authors in \cite{oprescu2019orthogonal} combine the DML estimators with the generalized random forest given in \cite{athey2019generalized} to estimate the heterogeneous effects (or the conditional average treatment effects (CATEs)). %Finally, authors in \cite{shi2019adapting} adapt the neural networks to estimate the average treatment effect. They improve the network architecture proposed by \cite{shalit2017estimating} and introduce a novel loss using the DML estimators.
Recently, \cite{mackey2018orthogonal} firstly proposed orthogonal machine learning where the treatment effect estimator satisfying the higher-order orthogonality than in DML was constructed and nice theoretical properties were presented (our paper will follow the line of orthogonal machine learning works).

DML offsets the shortcomings of classical causal learning methods. The DML method combines the two classical approaches to ensure that the corresponding estimator is accurate as long as either the outcome model or the propensity score model, but not necessarily both, is correctly specified (see \cite{robins2017minimax,robins2008higher,mukherjee2017semiparametric,kang2007demystifying,van2014higher} and the references therein). This notable merit is well known as the \textit{doubly robust} property. However, the DML estimator can still suffer from the error compounding issue since the inverse term of the propensity score is still present. For example, assuming the true propensity score $\pi^i(\mathbf{z}_m)=0.01$ and the estimated one $\hat{\pi}^i(\mathbf{z}_m)=0.05$, they only differ with $|0.01-0.05|=0.04$. However, the error of their inverses can attain up to $|1/0.01-1/0.05|=80$. Inevitably, such situation occurs commonly in practice, especially when the distribution of the treated group is substantially different from that of the controlled group (see, for example, \cite{li2018balancing, busso2014new, knaus2020double, austin2015moving, dudik2011doubly}). In practice, two variants of the DML estimator, Augmented Inverse Propensity Weighted (AIPW) \cite{robins1994estimation, rotnitzky1998semiparametric, bang2005doubly} and DML-trim \cite{chernozhukov2018double}, are often used to tackle this problem.

AIPW has the same mathematical form as the DML estimator, but can avoid an infinite estimate. When IPW is infinite, the two terms in AIPW that contain IPW will cancel each other out \cite{linden2016estimating}. DML-trim removes the individuals with very high or low estimated propensity scores \cite{crump2009dealing} or re-adjusts the propensity scores based on a specific range (e.g., in experiments of this paper, following \cite{chernozhukov2018double}, we trim the estimated propensity scores with $0.01$ and $0.99$). Nevertheless, such incomplete solutions are trick-based, not from a theoretical or methodological perspective.
Although we can deal with the error-compounding issue with them from a practical perspective, stabilizing the compounded error caused by IPW from a theoretical standpoint still remains a challenge. This observation motivates us to go beyond the DML estimator and construct estimators that are more robust to the extremes of the propensity scores.

%Since the initiation of the DML approach, it has been receiving extensive interest in both the theoretical communities and practitioners. Some researches focus on improving the DML theory. For example, \cite{semenova2018orthogonal} extend the DML to the estimation of heterogeneous effects. 
%\cite{tibshirani2016exact} and \cite{berk2013valid} investigate the post-selection inference problem of DML, and \cite{zhang2014confidence, van2014asymptotically} analyze the confidence intervals of causal parameters using DML.
%Others aim to incorporate the DML theory in training the regressors such as random forests (RF) \cite{oprescu2019orthogonal} and neural networks (NN) (\cite{huang2021causal, shi2019adapting}).

In this paper, we propose a Robust Causal Learning (RCL) method to establish the RCL estimators of ATE. The contributions are summarized as follows: 
\begin{enumerate}
	\item [1.] Our RCL estimators robustly ease the error-compounding issue exhibited by the DML estimator since the propensity scores in the RCL estimators are no longer in an inverse form.
	\item [2.] The RCL estimators inherit the consistency and doubly robust property of the DML estimator, and satisfy the higher-order orthogonality.
	\item [3.] The RCL methodology to construct an ATE estimator can also be applied to establish other prevalent treatment effect estimators.
	\item [4.] Extensive experiments show that the proposed RCL method achieves superior performance than DR, IPW, DML, and their variants AIPW and DML-trim across various combinations of machine learning regressors and classifiers.
\end{enumerate}
%The code for reproducing the experimental results is uploaded in the link\footnote{\href{https://github.com/xingyan-fml/rcl}{https://github.com/xingyan-fml/rcl}}.

The rest of the paper is organized as follows. Section \ref{sec:background} introduces the problem setup and the background of orthogonal scores. Section \ref{sec:$k^{th}$-order condition and assumptions} presents the main theoretical results, including the RCL score in Theorem \ref{thm:$k^{th}$-order orthogonal condition score function} and the RCL estimator in Corollary \ref{corollary:$k^{th}$-order orthogonal condition estimator}. In Section \ref{sec:theory}, we theoretically prove the consistency of the RCL estimator. Section \ref{sec:experiment} reports extensive experimental results on simulation datasets, benchmark datasets, and a mimic consumer credit dataset generated by WGAN. We defer some proof details %and the code for reproducing our experiments to the full paper version.
in the appendix.%the code for reproducing our experiment results in the supplementary material.
%\footnote{{\color{blue} \url{https://github.com/PaperandCode/ijcnn-251}}}

%\input{The_Problem_Setup/The_Problem_Setup}
\section{Problem Formulation and Assumptions}\label{sec:background}
\subsection{The Problem Setup}
In this paper, we consider the potential outcome framework \cite{rubin1974estimating, rubin2005causal} to study the ATE estimation. Let {$\mathbf{Z}$} be the covariates (confounders), and {$D$} be the treatment variable which can take values from {$\{d^{1},\dots, d^{n}\}$}. We denote {$Y$} as the outcome variable (the response), and {$Y^{i}$} represents the \textit{potential outcome} under the treatment {$d^{i}$}. If the observed treatment is {$d^{i}$}, then the \textit{factual outcome} {$Y^{F}$} equals {$Y^{i}$}. We denote {$\{w_m=(y_m, d_m, \mathbf{z}_m)\}_{m=1}^{N}$} as the observed $N$ realizations of i.i.d. random variables {$\{W_m=(Y_m, D_m, \mathbf{Z}_m)\}_{m=1}^{N}$}. Given the true causal parameter {$\theta^{i}:=\mathbb{E}[Y^{i}]$}, the target quantity ATE between treatment $d^{i}$ and treatment $d^{j}$ is defined as 
%We suppose the observational dataset contains $N$ individuals and the $m^{th}$ individual is observed as $(\mathbf{z}_m, d_m, y_m)$. 
\begin{equation}
{
\begin{aligned}
\theta^{i,j}=\theta^{i}-\theta^{j}.
\end{aligned}
}
\end{equation}
Identifying ATE {$\theta^{i,j}$} under the potential outcome framework requires some fundamental assumptions to ensure that {$\theta^{i}$} and {$\theta^{j}$} are identifiable. Thus, we impose the following assumptions as stated in existing causal inference literature.

\begin{assumption}[Stable Unit Treatment Value Assumption] (SUTVA)
	The potential outcomes for any individual do not vary regardless of the treatment status of other individuals.
\end{assumption}

\begin{assumption}[Ignorability]
	Given the covariates {$\mathbf{Z}$}, the potential outcomes are independent of the treatment assignment {$D$}, i.e., {$(Y^{1},\cdots, Y^{n})\perp\!\!\!\perp D\mid \mathbf{Z}$}.
\end{assumption}

\begin{assumption}[Positivity]
	Treatment assignment is not deterministic regardless of the values the covariates {$\mathbf{Z}$} take, i.e., {$0<\mathbb{P}\{D=d^{i}\mid \mathbf{Z}=\mathbf{z}\}<1$}, {$\forall\;i$} and {$\forall\;\mathbf{z}$}.
\end{assumption}

\begin{assumption}[Consistency]
If an individual receives treatment {$d^{i}$}, his factual outcome {$Y^{F}$} is equal to the potential outcome under treatment {$d^{i}$}, i.e., {$Y^{F}=Y^{i}$} if {$D=d^{i}$}.
\end{assumption}

These assumptions guarantee that ATE can be inferred if we specify the relation {$\mathbb{E}\left[Y \mid D, \mathbf{Z}\right]$}, which is equivalent to estimating {$g^{i}(\mathbf{Z})$} for each {$i \in \{1, \dots, n\}$} in the generalized propensity score model setting\footnote{This model setting allows {$D$} to be multi-valued. It can be reduced to the \textit{iteractive model} in \cite{chernozhukov2018double} once the treatment {$D$} takes binary values.} \cite{imbens2000role,tu2013using}:
\begin{subequations}
	\begin{align}		
		Y^{i}&=g^{i}(\mathbf{Z})+\xi^{i},\quad \mathbb{E}\left[\xi^{i}\mid D, \mathbf{Z}\right]=0 \quad a.s.,	\\	
	\mathbf{1}_{\{D=d^{i}\}}&=\pi^{i}(\mathbf{Z})+\nu^{i},\quad \mathbb{E}\left[\nu^{i}\mid\mathbf{Z}\right]=0 \quad a.s..
\end{align}	
\end{subequations}\noindent
Here, {$g^i(\cdot)$} and {$\pi^i(\cdot)$} are true \textit{nuisance parameters}. {$\xi^{i}$} and {$\nu^{i}$} are the noise terms. {$\pi^i(\mathbf{Z})=\mathbb{E}\left[\mathbf{1}_{\{D=d^{i}\}} \mid \mathbf{Z}\right]$} is known as the generalized propensity score (GPS) with multi-valued treatment variable. 
Notice that $\xi^{i} \perp\!\!\!\perp D\mid \mathbf{Z}$, so $\mathbb{E}\left[\xi^{i}\mid \mathbf{Z}\right]=0.$
Finally, the true causal parameter {$\theta^i$} for {$i \in \{1, \dots, n\}$} can be computed by {$\theta^i:=\mathbb{E}\left[Y^{i}\right]=\mathbb{E}\left[g^i(\mathbf{Z})\right]$}  and the true ATE can be computed by {$\theta^{i,j}=\theta^{i}-\theta^{j}$}.

\subsection{Non-Orthogonal Scores and Orthogonal Scores}

We aim to estimate the true causal parameters $\theta^{i}$ given $N$ i.i.d. samples $\{W_m=(Y_m, D_m, \mathbf{Z}_m)\}^{N}_{m=1}$. According to \cite{chernozhukov2018double}, the standard procedure to obtain the estimated causal parameter $\hat{\theta}^i$ is: 1) getting the estimated nuisance parameters $\hat{\rho}$, e.g., $\hat{\rho}=(\hat{g}^i, \hat{\pi}^i)$; 2) constructing a \textit{score} that satisfies the \textit{moment condition} (Definition \ref{def:moment condition}); 3) establishing the estimator of $\theta^i$, which is solved from the moment condition \eqref{eqt:moment condition}.
\begin{definition}[\textit{\textbf{Moment Condition}}]\label{def:moment condition}
	Let {$W=(Y,D,\mathbf{Z})$} and {$\theta$} be the true causal parameter with {$\vartheta$} being a causal parameter that lies in the causal parameter set. Denoting the nuisance parameters as {$\varrho$} and the true nuisance parameters as {$\rho$}, we say a score {$\psi(W,\vartheta,\varrho)$} satisfies the \textit{moment condition} if
	%Suppose $\vartheta$ and $\theta$ are the causal parameter and the authentic causal parameter lying in a convex set $\Theta$ while $\varrho$ and $\rho$ are the nuisance parameter and the authentic nuisance parameter lying in a convex set $\mathcal{T}$. We say that a score function $\psi(W,\vartheta,\varrho)$ satisfies \textit{moment condition} if
	\begin{equation}
		\begin{aligned}\label{eqt:moment condition}
			\mathbb{E}\left[\psi(W,\vartheta,\varrho)|_{\vartheta=\theta,\;\varrho=\rho}\right]=0.
		\end{aligned}
	\end{equation}
\end{definition}
%The moment condition guarantees that the estimator derived from the score is unbiased if the nuisance parameters equal the true ones. Here, we give examples about the previously mentioned two classical causal learning methods, DR and IPW, whose scores satisfy the moment condition.
The moment condition guarantees that the estimator derived from the score is unbiased if the nuisance parameters equal the true ones. Here, we give the scores which satisfy the moment condition of two classical causal learning methods (DR and IPW) introduced before.

\begin{example}[\textit{The Score and Estimator for DR}]\label{example:DR}
Let {$\varrho=\mathcal{g}^{i}$} and $\rho=g^{i}$. In the DR method, the score {$\psi_{DR}^{i}(W,\vartheta,\varrho)$} satisfying the moment condition and the associated estimator {$\hat{\theta}_{DR}^{i}$} are
{	\begin{equation*}
		\begin{aligned}
			\psi_{DR}^{i}(W,\vartheta,\varrho)=\vartheta-\mathcal{g}^{i}(\mathbf{Z}); \quad 
			\hat{\theta}_{DR}^{i}=\frac{1}{N}\underset{m=1}{\overset{N}{\sum}}\hat{g}^{i}(\mathbf{Z}_{m}).
		\end{aligned}
	\end{equation*}}
\end{example}
\begin{example}[\textit{The Score and Estimator for IPW}]\label{example:IPW}	
Let {$\varrho=a_{i}$} and {$\rho=\pi^{i}$}. In the IPW method, the score {$\psi_{IPW}^{i}(W,\vartheta,\varrho)$} satisfying the moment condition and the associated estimator {$\hat{\theta}_{IPW}^{i}$} are
		{	\begin{equation*}
				\begin{aligned}
					\psi_{IPW}^{i}(W,\vartheta,\varrho)=\vartheta-\frac{Y\mathbf{1}_{\{D=d^{i}\}}}{a_{i}(\mathbf{Z})}; \\
					\hat{\theta}_{IPW}^{i}=\frac{1}{N}\underset{m=1}{\overset{N}{\sum}}\frac{Y_{m}\mathbf{1}_{\{D_{m}=d^{i}\}}}{\hat{\pi}^{i}(\mathbf{Z}_{m})}.
				\end{aligned}
		\end{equation*}}
	%In the IPW method, {$\frac{1}{N}\underset{m=1}{\overset{N}{\sum}}\frac{Y_{m}\mathbf{1}_{\{D_{m}=d^{i}\}}}{\hat{\pi}^{i}(\mathbf{Z}_{m})}$} is the estimator of {$\theta^{i}$}, where {$\hat{\pi}^{i}$} is the estimate of {$\pi^{i}$}. Denote the nuisance parameter and the true one as {$a_{i}(\cdot)$} and {$\pi^{i}(\cdot)$}, the score is {$\psi_{IPW}^{i}(W,\vartheta,\varrho)=\vartheta-\frac{Y\mathbf{1}_{\{D=d^{i}\}}}{a_{i}(\mathbf{Z})}$}. This score incorporates the inverse term of the propensity score which may give large estimates.
\end{example}
Generally, the estimators established from the scores in Example \ref{example:DR} and \ref{example:IPW} might be invalid unless {$\hat{g}^{i}$} and {$\hat{\pi}^{i}$} estimate {$g^{i}$} and {$\pi^{i}$} well. To obtain robust estimators, \cite{chernozhukov2018double,mackey2018orthogonal} suggest that we should construct scores which satisfy the (higher-order) Orthogonal Condition (Definition \ref{def:orthogonal condition}) apart from the moment condition.
%Although the scores of classical methods satisfy the moment condition, they still violate the \textit{Orthogonal Condition}. Such scores are called \textit{Non-Orthogonal Scores}. The causal parameters derived from non-orthogonal scores might be invalid unless {$\hat{g}^{i}$} and {$\hat{\pi}^{i}$} estimate {$g^{i}$} and {$\pi^{i}$} well. To obtain an estimator that is robust to the estimations of nuisance parameters, \cite{chernozhukov2018double} suggest that we should construct \textit{Orthogonal Scores} that satisfy the moment condition and the \textit{Orthogonal Condition} (Definition \ref{def:orthogonal condition}) synchronously.

\begin{definition}[\textit{\textbf{Orthogonal Condition}}]\label{def:orthogonal condition}
	Suppose that the nuisance parameters and the true nuisance parameters are {$\gamma$}-dimensional tuples, i.e., {$\varrho=(\mathcal{h}_{1},\cdots, \mathcal{h}_{\gamma})$} and {$\rho=(h_{1},\cdots, h_{\gamma})$}. The Gateaux $\alpha$-differential of a score function $\psi(W,\vartheta,\varrho)$ w.r.t. the nuisance parameters, denoted as $\mathrm{\mathbf{D}}^{\alpha}\psi(W,\vartheta,\varrho)$, is defined as
\begin{equation}
{
\begin{aligned}\label{eqt:higher-order differential}
\mathrm{\mathbf{D}}^{\alpha}\psi(W,\vartheta,\varrho)&=\partial^{\alpha_{1}}_{\mathcal{h}_{1}}\partial^{\alpha_{2}}_{\mathcal{h}_{2}}\cdots\partial^{\alpha_{\gamma}}_{\mathcal{h}_{\gamma}}\psi(W,\vartheta,\varrho)\\
&=\partial^{\alpha_{1}}_{\mathcal{h}_{1}}\partial^{\alpha_{2}}_{\mathcal{h}_{2}}\cdots\partial^{\alpha_{\gamma}}_{\mathcal{h}_{\gamma}}\psi(W,\vartheta,\mathcal{h}_{1},\dots,\mathcal{h}_{\gamma}).
\end{aligned}
}
\end{equation}	
Given {$S\subseteq \mathbb{Z}_{\geq 0}^{\gamma}$} where $\mathbb{Z}_{\geq 0}$ is the set of all nonnegative integers, we say a score {$\psi(W,\vartheta,\varrho)$} satisfies the \textit{orthogonal condition} if
	%Suppose that {$\vartheta$} and {$\theta$} lie in a convex set {$\Theta$}, while {$\varrho$} and {$\rho$} are {$\gamma$}-dimensional tuples which lie in a convex set {$\mathfrak{T}$} such that {$\varrho=(\mathcal{h}_{1},\cdots, \mathcal{h}_{\gamma})$} and {$\rho=(h_{1},\cdots, h_{\gamma})$}. Let {$S\subseteq \mathbb{Z}_{\geq 0}^{\gamma}$}. The score {$\psi(W,\vartheta,\varrho)$} should satisfy
\begin{equation}
{
\begin{aligned}
\mathbb{E}\left[\mathrm{\mathbf{D}}^{\alpha}\psi(W,\vartheta,\varrho)\mid_{\vartheta=\theta,\;\varrho=\rho}\mid\mathbf{Z}\right]=0,\quad \forall \alpha\in S. \label{eqt:$k^{th}$-order orthogonal condition}
\end{aligned}
}
\end{equation}
{$S$} can be any proper subset of {$\mathbb{Z}_{\geq 0}^{\gamma}$}. Throughout the paper, for some positive integer {$k$}, we define {$S$} as
\begin{equation}
{
\begin{aligned}\label{eqt:S definition}
S_k=\{\alpha\in\mathbb{Z}_{\geq 0}^{\gamma}\mid \left\|\alpha\right\|_{1}\leq k\}.
\end{aligned}
}
\end{equation}
%and {$\mathrm{\mathbf{D}}^{\alpha}\psi(W,\vartheta,\varrho)=\partial^{\alpha_{1}}_{\mathcal{h}_{1}}\partial^{\alpha_{2}}_{\mathcal{h}_{2}}\cdots\partial^{\alpha_{\gamma}}_{\mathcal{h}_{\gamma}}\psi(W,\vartheta,\mathcal{h}_{1},\cdots, \mathcal{h}_{\gamma})$}.
\end{definition}

The orthogonal condition ensures that the established estimators can still be valid even though some nuisance parameters are misspecified (see \cite{chernozhukov2018double, robins2008higher, mackey2018orthogonal} for more details). Below we demonstrate how to utilize {$\mathbf{D}^{\alpha}\psi(W,\vartheta,\varrho)$} to justify that the scores in Example \ref{example:DR} and  \ref{example:IPW} violate the orthogonal condition. Suppose $k$ in \eqref{eqt:S definition} is $1$, then
{\begin{equation*}
	\begin{aligned}
	&\mathbf{D}^{(1)}\psi_{DR}^{i}(W,\vartheta,\varrho)=\partial_{\mathcal{g}^{i}}\psi_{DR}^{i}(W,\vartheta,\varrho)=-1,\\
	&\mathbf{D}^{(1)}\psi_{IPW}^{i}(W,\vartheta,\varrho)=\partial_{a_{i}}\psi_{IPW}^{i}(W,\vartheta,\varrho)=\frac{Y\mathbf{1}_{\{D=d^{i}\}}}{a_{i}(\mathbf{Z})^{2}},\\
	&\mathbb{E}\left[\mathbf{D}^{(1)}\psi_{DR}^{i}(W,\vartheta,\varrho)\mid_{\vartheta=\theta^i,\;\varrho=g^i} \mid \mathbf{Z}\right]=-1 \neq 0,\\
	&\mathbb{E}\left[\mathbf{D}^{(1)}\psi_{IPW}^{i}(W,\vartheta,\varrho)\mid_{\vartheta=\theta^i,\;\varrho=\pi^i} \mid \mathbf{Z}\right]\\
	=&\mathbb{E}\left[\frac{Y\mathbf{1}_{\{D=d^{i}\}}}{\pi^{i}(\mathbf{Z})^{2}} \mid \mathbf{Z}\right] \neq 0.\\
%	&\partial_{a_{i}}\psi_{IPW}^{i}(W,\theta,\varrho)=\frac{Y\mathbf{1}_{\{D=d^{i}\}}}{a_{i}(\mathbf{Z})^{2}}
	\end{aligned}
\end{equation*}}\noindent
The above calculations show that {$\psi_{DR}^{i}(W,\vartheta,\varrho)$} and {$\psi_{IPW}^{i}(W,\vartheta,\varrho)$} do not satisfy the orthogonal condition. Such scores are usually termed as the \textit{non-orthogonal scores}. As a consequence, their associated estimators are not \textit{doubly robust}. To obtain a doubly robust estimator, \cite{chernozhukov2018double} propose the DML method to construct the DML score as we summarize below.

\begin{example}[\textit{The Score and Estimator for DML}]\label{example:DML}	
Let {$\varrho=(\mathcal{g}^{i}, a_{i})$} and {$\rho=(g^{i},\pi^{i})$}, in the DML method, the score {$\psi_{DML}^{i}(W,\vartheta,\varrho)$} that satisfies both the moment condition and orthogonal condition and the associated estimator {$\hat{\theta}_{DML}^{i}$} are
	{
		\begin{equation*}
						\begin{aligned}
				&\psi_{DML}^{i}(W,\vartheta,\varrho)=\vartheta-\mathcal{g}^{i}(\mathbf{Z})-\frac{\bm{1}_{\{D=d^{i}\}}}{a_{i}(\mathbf{Z})}(Y-\mathcal{g}^{i}(\mathbf{Z})),\\
				&\hat{\theta}_{DML}^{i}=\frac{1}{N}\overset{N}{\underset{m=1}{\sum}}\hat{g}^{i}(\mathbf{Z}_{m})+\frac{1}{N}\overset{N}{\underset{m=1}{\sum}}\frac{\bm{1}_{\{D_{m}=d^{i}\}}(Y_{m}-\hat{g}^{i}(\mathbf{Z}_{m}))}{\hat{\pi}^{i}(\mathbf{Z}_{m})}.
			\end{aligned}
		\end{equation*}
}
\end{example}
We can prove that {$\psi_{DML}^{i}(W,\vartheta,\varrho)$} satisfies the orthogonal condition when {$k=1$} in Equation \eqref{eqt:S definition} (see \cite{chernozhukov2018double} for detailed derivations) following similar calculation processes for DR and IPW. {$\psi_{DML}^{i}(W,\vartheta,\varrho)$} is therefore termed as an \textit{orthogonal score}. The orthogonal condition assures that the DML estimator is doubly robust, i.e., the estimator is locally unbiased and consistent as long as either {$g^{i}$} or {$\pi^{i}$} is correctly specified. However, in spite of the doubly robust property, the DML estimator still suffers from the error-compounding issue once some encompassed inverse propensity scores are extreme. In real applications, one seldom encounters a situation that propensity scores are correctly estimated for all individuals and with no extreme values. 

This dilemma motivates us to construct scores such that: 1) the scores are orthogonal scores, i.e., they satisfy the moment condition (Definition \ref{def:moment condition}) and the orthogonal condition (Definition \ref{def:orthogonal condition}); 2) the estimators established from the scores can stabilize the estimation error due to the extremes of propensity scores. In the upcoming section, we will introduce a novel method, the Robust Causal Learning (RCL) method, to achieve this goal and overcome the difficulties encountered by DR, IPW, and DML methods.

\section{The Proposed Method}\label{sec:$k^{th}$-order condition and assumptions}
This section introduces our main contribution, the RCL method proposed by us. First, Section \ref{sec:Construction of the RCL Score} shows the constructed RCL score. Then Section \ref{sec:Establishment of the RCL estimators} presents the detailed construction of the RCL estimator with an algorithm that describes how to obtain such an estimator of {$\theta^{i}$} from observational data.

%First, we state the RCL score\footnote{Our RCL scores are closely related to the higher order influence functions in nonparametric statistics. Compared to the higher order influence functions, our scores are expressed in terms of basis functions, and we do not require the knowledge of the density or the inverse covariance matrix of {$\mathbf{Z}$} (see, e.g., \cite{mukherjee2017semiparametric}, \cite{robins2008higher} and \cite{van2014higher}).} in Theorem \ref{thm:$k^{th}$-order orthogonal condition score function}.

\subsection{Construction of The RCL Score}\label{sec:Construction of the RCL Score}

In this subsection, we construct an orthogonal score, the RCL score, to derive an estimator of {$\theta^{i}$} along the line of orthogonal machine learning works (e.g.,  \cite{chernozhukov2018double,mackey2018orthogonal}). The main result is stated in Theorem \ref{thm:$k^{th}$-order orthogonal condition score function}.
\begin{theorem}[\textbf{\textit{RCL Score}}]\label{thm:$k^{th}$-order orthogonal condition score function}
Suppose {$\varrho$} and {$\rho$} are {$2$}-dimensional tuples such that {$\varrho=(\mathcal{g}^{i},a_{i})$} and {$\rho=(g^{i},\pi^{i})$}. Let {$r$}, {$k$} be integers s.t. $1\leq k\leq r$, $r\ge 2$, and assume the local moments {$\mathbb{E}\left[(\nu^{i})^{r}\mid\mathbf{Z}\right]\neq 0$}. % and {$\mathbb{E}\left[(\nu^{i})^{q}\mid\mathbf{Z}\right]$} is a constant function of {$\mathbf{Z}$}, {$\forall\;1\leq q\leq r$}.
%{$\left|\mathbb{E}\left[(\nu^{i})^{q}\mid\mathbf{Z}\right]\right|<\infty\;a.s.,\;\forall\;1\leq q\leq r$}. 
Under the assumptions on nuisance parameters and noise terms stated in \cite{chernozhukov2018double} and \cite{mackey2018orthogonal}, the proposed RCL score {$\psi_{RCL}^{i}(W,\vartheta,\varrho)$} that satisfies the moment condition and the orthogonal condition for {$S_k$} in Equation (\ref{eqt:S definition}) is
{
\begin{align}
\psi_{RCL}^{i}(W,\vartheta,\varrho)&=\vartheta-\mathcal{g}^{i}(\mathbf{Z})-(Y^{i}-\mathcal{g}^{i}(\mathbf{Z}))A(D,\mathbf{Z};a_{i}). \label{eqt:$k^{th}$-order orthogonal condition score function}
\end{align}
}\noindent
Given the integers {$r$} and {$k$},
\begin{subequations}
\begin{equation}
{
\begin{aligned}
&A(D,\mathbf{Z};a_{i})=\bar{b}_{r}\left[\mathbf{1}_{\{D=d^{i}\}}-a_{i}(\mathbf{Z})\right]^{r}+\\
&\;\;\mathbf{1}_{\{k\neq 1\}}\left[\underset{q=1}{\overset{k-1}{\sum}}b_{q}\left(\left[\mathbf{1}_{\{D=d^{i}\}}-a_{i}(\mathbf{Z})\right]^{q}-\mathbb{E}\left[(\nu^{i})^{q}\mid\mathbf{Z}\right]\right)\right],
\label{eqt:$k^{th}$-order orthogonal condition score function 2}
\end{aligned}
}
\end{equation}
where {$\bar{b}_{r}=\frac{1}{\mathbb{E}\left[(\nu^{i})^{r}\mid\mathbf{Z}\right]}$} and the coefficient {$b_{q}$} is computed by descending order for {$q \in \{k-1,\dots,1\}$}:
%Starting from {$b_{k-1}=-\bar{b}_{r}{r \choose k-1}\mathbb{E}\left[(\nu^{i})^{r-(k-1)}\mid\mathbf{Z}\right]$}, the coefficient {$b_{q}$} is computed by descending order for {$q \in \{k-1,\dots,1\}$}:
\begin{equation}
{
\begin{aligned}
b_{q}&=-\bar{b}_{r}\binom{r}{q}\mathbb{E}\left[(\nu^{i})^{r-q}\mid\mathbf{Z}\right]\\
&-\overset{k-1-q}{\underset{u=1}{\sum}}b_{q+u}\binom{q+u}{q}\mathbb{E}\left[(\nu^{i})^{u}\mid\mathbf{Z}\right].
\end{aligned}\label{eqt:kth_order orthogonal condition score function 2 coefficient}
}
\end{equation}
\end{subequations}
\end{theorem}

We might denote the consequent estimator constructed later using the above score as RCL{$_{r,k}$}. %{\color{red}From \eqref{eqt:$k^{th}$-order orthogonal condition score function 2}, we have the following observations: First, $\bar{b}_{r}$ does not incur a high variance as a constant coefficient. Second, the expectation of $\bar{b}_{r}\left[\mathbf{1}_{\{D=d^{i}\}}-a_{i}(\mathbf{Z})\right]$ equals 1 no matter what value $a_{i}(\mathbf{Z})$ reaches. Third, and more importantly, $a_{i}(\cdot)$, the nuisance parameter of the propensity score, is no longer in an inverse form in the RCL score.} 
From \eqref{eqt:$k^{th}$-order orthogonal condition score function 2}, we can observe that $a_{i}(\cdot)$, the nuisance parameter of the propensity score, is no longer in an inverse form in the RCL score.
As a consequence, the established RCL estimator from \eqref{eqt:$k^{th}$-order orthogonal condition score function} can avoid the error-compounding issue. Simultaneously, the RCL score is an orthogonal score, so the RCL estimator is as doubly robust as the DML one. The proof of Theorem \ref{thm:$k^{th}$-order orthogonal condition score function} can be found in Appendix \ref{appendix:A}. From the proof, one can find some intuition behind the construction of such a score.

\subsection{Establishment of The RCL Estimator}\label{sec:Establishment of the RCL estimators}

In this part, we will go into detail about the establishment of the RCL estimator. To begin with, we can solve the estimator $\tilde{\theta}^{i}$ from \eqref{eqt:moment condition} using the emprical version of the moment condition for the RCL score \eqref{eqt:$k^{th}$-order orthogonal condition score function}:
%\begin{equation}
%{
%\begin{aligned}\label{eqt:theta_raw}
%\theta^{i}&=\underbrace{\mathbb{E}\left[g^{i}(\mathbf{Z})\right]}_{(a)}+\underbrace{\mathbb{E}\left[(Y^{i}-g^{i}(\mathbf{Z}))A(D,\mathbf{Z};\pi^{i})\right]}_{(b)}.
%\end{aligned}
%}
%\end{equation}
\begin{subequations}
	{
		\begin{align}
			\tilde{\theta}^{i}=&\frac{1}{N}\underset{m=1}{\overset{N}{\sum}}g^{i}(\mathbf{Z}_{m})\label{tilde_theta_1} \\
			&+\frac{1}{N}\underset{m=1}{\overset{N}{\sum}}(Y^{i}_{m}-g^{i}(\mathbf{Z}_{m}))A(D_{m},\mathbf{Z}_{m};\pi^{i}). \label{tilde_theta_2}
		\end{align}
	}\noindent
\end{subequations}
Equation \eqref{tilde_theta_1} is referred to as the DR estimator when the true nuisance parameter {$g^i$} is replaced by the estimated one {$\hat{g}^i$}. Equation \eqref{tilde_theta_2} can then be divided into two parts:
\begin{subequations}
	{
		\begin{align}
\eqref{tilde_theta_2} &=\frac{1}{N}\underset{m\in\mathscr{I}}{\overset{}{\sum}}(Y^{i}_{m}-g^{i}(\mathbf{Z}_{m}))A(D_{m},\mathbf{Z}_{m};\pi^{i})\label{tilde_theta_2_1} \\
&+\frac{1}{N}\underset{m\in\mathscr{I}^{c}}{\overset{}{\sum}}(Y^{i}_{m}-g^{i}(\mathbf{Z}_{m}))A(D_{m},\mathbf{Z}_{m};\pi^{i})\label{tilde_theta_2_2},
		\end{align}
	}\noindent
\end{subequations}
where {$\mathscr{I}$} is the sample set in which the units are all treated with {$d^{i}$} while {$\mathscr{I}^c$} is the sample set in which the units are not treated with {$d^{i}$}. It is obvious that \eqref{tilde_theta_1} and \eqref{tilde_theta_2_1} can be directly calculated from observational data, whereas \eqref{tilde_theta_2_2} that contains the counterfactual outcomes is unavailable to compute in a direct manner. Instead of pursuing the unobservable counterfactuals, we realize that given {$i \in \{ 1,\dots,n \}$},
{\begin{equation*}
	\begin{aligned}
			&\mathbb{E}\left[(Y^i-g^i(\mathbf{Z}))A(D,\mathbf{Z};\pi^{i}) \mid D=d^j\right]\\
%			=\mathbb{E}\left[\xi^{i} A(D,\mathbf{Z};\pi^{i}) \mid D=d^j\right]\\
			&=\mathbb{E}\left[\mathbb{E}\left[\xi^{i} A(D,\mathbf{Z};\pi^{i}) \mid D=d^j, \mathbf{Z}\right]\mid D=d^j \right]\\
			&=\mathbb{E}\left[A(d^j,\mathbf{Z};\pi^{i})\mathbb{E}\left[\xi^{i} \mid D=d^j, \mathbf{Z}\right]\mid D=d^j \right]=0
	\end{aligned}
\end{equation*}}\noindent
%\begin{equation}
%{
%\begin{aligned}\label{eqt:decomposition sum}
%&\underbrace{\frac{1}{N}\underset{m\in\mathscr{I}}{\overset{}{\sum}}(Y^{i}_{m}-g^{i}(\mathbf{Z}_{m}))A(D_{m},\mathbf{Z}_{m};\pi^{i})}_{(a)}\\
%&\;+\underbrace{\frac{1}{N}\underset{m\in\mathscr{I}^{c}}{\overset{}{\sum}}(Y^{i}_{m}-g^{i}(\mathbf{Z}_{m}))A(D_{m},\mathbf{Z}_{m};\pi^{i})}_{(b)},
%\end{aligned}
%}
%\end{equation}
holds for {$\forall j \in \{ 1,\dots,n \}$}. Thus, the sample means of {$Y^i-g^i(\mathbf{Z})$} and {$(Y^i-g^i(\mathbf{Z}))A(D,\mathbf{Z};\pi^{i})$} equal zero regardless the samples come from {$\mathscr{I}$} or {$\mathscr{I}^c$}. This observation allows us to replace the sample mean of the counterfactuals in \eqref{tilde_theta_2_2} with that of the factual ones. To be specific, we first define the set {$\mathcal{A}$} such that
\begin{equation}
{
\begin{aligned}
\mathcal{A}=\{Y^{i}_{m}-g^{i}(\mathbf{Z}_{m})\mid m\in\mathscr{I}\}. \label{eqt:random picking realization set random variable}
\end{aligned}
}
\end{equation}
Then, a replaced estimator of \eqref{tilde_theta_2_2} is obtained as follows:
\begin{enumerate}
	\item For the {$m^{\mathrm{th}}$} unit in the set {$\mathscr{I}^{c}$}, pick an element from {$\mathcal{A}$}, denoted as $\xi_{m}^{i}$, and multiply it by {$A(D_{m},\mathbf{Z}_{m};\pi^{i})$}. Repeat the process until we go through all the individuals in the set {$\mathscr{I}^{c}$};
	\item Compute $\frac{1}{N}\underset{m\in\mathscr{I}^{c}}{\overset{}{\sum}}\xi_{m}^{i}A(D_{m},\mathbf{Z}_{m};\pi^{i})$;
	\item Repeat above steps {$R$} times to reduce the randomness brought by the random picking procedure and return the substitute estimator {$\frac{1}{R}\overset{R}{\underset{u=1}{\sum}}\left[\frac{1}{N}\underset{m\in\mathscr{I}^{c}}{\overset{}{\sum}}\xi_{m,u}^{i} A(D_{m},\mathbf{Z}_{m};\pi^{i})\right]$}.
\end{enumerate}
Consequently, \eqref{tilde_theta_2_2} can be inferred indirectly from the observational data. With \eqref{tilde_theta_1} and \eqref{tilde_theta_2_1} as well, the RCL estimator of {$\theta^{i}$} is finally established in Corollary \ref{corollary:$k^{th}$-order orthogonal condition estimator}.

\begin{corollary}[\textbf{\textit{RCL Estimator}}]\label{corollary:$k^{th}$-order orthogonal condition estimator}
Let {$R\in\mathbb{Z}^{+}$}, {$(\hat{g}^{i},\hat{\pi}^{i})$} be the estimates of {$(g^{i},\pi^{i})$}, {$\hat{\mathcal{A}}$} be {$\mathcal{A}$} by replacing {$g^{i}$} with {$\hat{g}^{i}$} in \eqref{eqt:random picking realization set random variable}, and {$\hat{\xi}_{m,u}^{i}$} be the element that is randomly selected from the set {$\hat{\mathcal{A}}$} in the {$u^\mathrm{th}$} of $R$ repeated selections. The RCL estimator {$\hat{\theta}^{i}_{RCL}$} is given by
\begin{equation}
{
\begin{aligned}\label{eqt:$k^{th}$-order orthogonal condition final estimator 3-replaced}
\hat{\theta}^{i}_{RCL}=&\underbrace{\frac{1}{N}\underset{m=1}{\overset{N}{\sum}}\hat{g}^{i}(\mathbf{Z}_{m})}_{(a)}\\
&+\underbrace{\frac{1}{N}\underset{m\in\mathscr{I}}{\overset{}{\sum}}(Y^{i}_{m}-\hat{g}^{i}(\mathbf{Z}_{m}))A(D_{m},\mathbf{Z}_{m};\hat{\pi}^{i})}_{(b)}\\
&+\underbrace{\frac{1}{R}\overset{R}{\underset{u=1}{\sum}}\left[\frac{1}{N}\underset{m\in\mathscr{I}^{c}}{\overset{}{\sum}}\hat{\xi}_{m,u}^{i} A(D_{m},\mathbf{Z}_{m};\hat{\pi}^{i})\right]}_{(c)}.
\end{aligned}
}
\end{equation}
\end{corollary}

The proposed RCL estimator {$\hat{\theta}^{i}_{RCL}$} is a consistent estimator of {$\theta^{i}$} if {$(\hat{g}^{i},\hat{\pi}^{i})$} satisfy the assumptions stated in \cite{mackey2018orthogonal} and \cite{chernozhukov2018double}. The detailed consistency results of {$\hat{\theta}^{i}_{RCL}$} are in the next section and some pivotal proofs are left in the appendix due to the space limit. We also outline the procedure of estimating {$\theta^{i}$} from observational data using the proposed RCL method in Algorithm \ref{alg:$k^{th}$-order orthogonal condition estimate}. Note that if the whole dataset is split into the training set and the test set, Step 2 in Algorithm \ref{alg:$k^{th}$-order orthogonal condition estimate} will be only conducted on the training set, while Step 3-8 can be performed to obtain the estimate {$\hat{\theta}^{i}_{RCL}$} on either the training set or the test set. The running complexity of our algorithm is obviously at most {$O(NR)$}.

\begin{algorithm}[t]
\caption{Algorithm of obtaining the RCL estimator of $\theta^{i}$ using (\ref{eqt:$k^{th}$-order orthogonal condition final estimator 3-replaced}a)-(\ref{eqt:$k^{th}$-order orthogonal condition final estimator 3-replaced}c).}\label{alg:$k^{th}$-order orthogonal condition estimate}
\begin{algorithmic}[1]
\STATE {\bfseries Input:} Observational dataset {$\{(y_{m},d_{m},\mathbf{z}_{m})\}_{m=1}^{N} = \mathscr{I}\cup\mathscr{I}^{c}$}, and {$\mathscr{I}\cap\mathscr{I}^{c}=\emptyset$}.\label{alg:state 1}
\STATE Train {$g^{i}$} and {$\pi^{i}$} using the observed data with some machine learning models to obtain the estimated nuisance parameters {$\hat{g}^{i}$} and {$\hat{\pi}^{i}$}.
\STATE Relabel the observed data point {$(y,d,\mathbf{z})$} as {$(y,\tilde{d},\mathbf{z})$} such that {$\tilde{d}=1$} if {$d=d^{i}$} and {$\tilde{d}=0$} if {$d\neq d^{i}$}. Compute {$\tilde{d}-\hat{\pi}^{i}(\mathbf{z})$} for each observation and obtain a constant local moment estimate of {$\mathbb{E}\left[(\nu^i)^{q}\mid\mathbf{Z}\right]$} in \eqref{eqt:$k^{th}$-order orthogonal condition score function 2} for each {$q$} with the mean of all {$(\tilde{d}-\hat{\pi}^{i}(\mathbf{z}))^{q}$}.\label{alg:state 3}
\STATE Compute (\ref{eqt:$k^{th}$-order orthogonal condition final estimator 3-replaced}a)-(\ref{eqt:$k^{th}$-order orthogonal condition final estimator 3-replaced}b) using the observational data.\label{alg:state 4}
\STATE Compute {$y-\hat{g}^{i}(\mathbf{z})$} for each observation in {$\mathscr{I}$} and store the computed values in {$\hat{\mathcal{A}}^{rlz}$} such that {$\hat{\mathcal{A}}^{rlz}=\{y_{m}-\hat{g}^{i}(\mathbf{z}_{m})\mid m\in \mathscr{I}\}$}.\label{alg:state 5} 
\STATE For the {$m^{\mathrm{th}}$} individual in {$\mathscr{I}^c$}, compute {$A(d_{m},\mathbf{z}_{m};\hat{\pi}^{i})$}. 
\STATE Repeat a random picking procedure {$R$} times: picking an element {$\hat{\xi}^{i;rlz}_{m,u}$} randomly from {$\hat{\mathcal{A}}^{rlz}$} in the {$u^\mathrm{th}$} repeat for the {$m^{\mathrm{th}}$} individual in {$\mathscr{I}^c$}. Then compute {$\frac{1}{R}\overset{R}{\underset{u=1}{\sum}}\left[\frac{1}{N}\underset{m\in\mathscr{I}^{c}}{\overset{}{\sum}}\hat{\xi}_{m,u}^{i;rlz} A(d_{m},\mathbf{z}_{m};\hat{\pi}^{i})\right]$} as an estimate of (\ref{eqt:$k^{th}$-order orthogonal condition final estimator 3-replaced}c).\label{alg:state 6}
\STATE {\bfseries Return:} Use the values in Step \ref{alg:state 4} and \ref{alg:state 6} to get the estimate {$\hat{\theta}^{i}_{RCL}$} from the combination of (\ref{eqt:$k^{th}$-order orthogonal condition final estimator 3-replaced}a)-(\ref{eqt:$k^{th}$-order orthogonal condition final estimator 3-replaced}c).\label{alg:state 8}
\end{algorithmic}
\end{algorithm}

\section{Theoretical Analysis}\label{sec:theory}

To facilitate the upcoming studies, we first introduce some notations. These will assist us in the following consistency analysis of the RCL estimator.
Recall that $Y^{i}$ is the potential outcome under the treatment $d^{i}$, and {$\mathscr{I}$} is the sample set in which the units are all treated with {$d^{i}$} while {$\mathscr{I}^c$} is the sample set in which the units are not treated with {$d^{i}$}. We use $Y^{i;F}$ as the factual outcome if an individual from {$\mathscr{I}$} actually receives $d^i$, and $Y^{i;CF}$ as the counterfactual outcome if an individual from {$\mathscr{I}^c$} receives $d^i$. %Hence, we have
%\begin{equation*}
%\begin{aligned}
%Y^{i}=\begin{cases}
%Y^{i;F} & \text{if $D=d^{i}$}\\
%Y^{i;CF} & \text{if $D\neq d^{i}$}
%\end{cases}.
%\end{aligned}
%\end{equation*}
Based on the introduced notations, we can define two \textit{residual differences} $\xi_{m}^{i;F}$ and $\xi_{\bar{m}}^{i;CF}$ for $m\in\mathscr{I}$ and $\bar{m}\in\mathscr{I}^{c}$. Mathematically, $\xi_m^{i;F}:=Y_m^{i;F}-g^{i}(\mathbf{Z}_m)$ if $m \in \mathscr{I}$ and $\xi_{\bar{m}}^{i;CF}:=Y_{\bar{m}}^{i;CF}-g^{i}(\mathbf{Z}_{\bar{m}})$ if $\bar{m} \in \mathscr{I}^c$. 

We give some assumptions on statistical properties between $\xi_{m}^{i;F}$ and $\xi_{\bar{m}}^{i;CF}$. To be precise, we study $\xi_{m}^{i;F}$ and $\xi_{\bar{m}}^{i;CF}$ conditioning on $\mathbf{Z}$. It is an abbreviation of conditioning on $(\mathbf{Z_m},\mathbf{Z_{\bar{m}}})$ or on all $\mathbf{Z_m}$ and $\mathbf{Z_{\bar{m}}}$, and this applies throughout in the rest of the paper. First, $\xi_{m}^{i;F}$ and $\xi_{\bar{m}}^{i;CF}$ are independent conditioning on $\mathbf{Z}$ due to the SUTVA assumption and have the same distribution, i.e.,
\begin{equation*}
\begin{aligned}
\xi_{m}^{i;F}\overset{d}{=}\xi_{\bar{m}}^{i;CF}\mid \mathbf{Z}\quad\text{and}\quad\xi_{m}^{i;F}\perp\!\!\!\perp\xi_{\bar{m}}^{i;CF}\mid \mathbf{Z}.
\end{aligned}
\end{equation*}
From the above assumptions, we have $\mathbb{E}\big[(\xi_{m}^{i;F})^{r}\mid\mathbf{Z}\big]=\mathbb{E}\big[(\xi_{\bar{m}}^{i;CF})^{r} \mid\mathbf{Z}\big]$, $\forall\;r,\;m$, and $\bar{m}$.

From the construction of our RCL estimator, it is easy to find that for $\bar{m}\in\mathscr{I}^c$, we replace $\xi_{\bar{m}}^{i;CF}$ with a $\xi_{m}^{i;F}$ for a random $m$ picked from $\mathscr{I}$. For the convenience of mathematical derivations, we denote an i.i.d. copy of this $\xi_{m}^{i;F}$ as $\xi_{\bar{m}}^{i;F}$. So for any $\bar{m}\in\mathscr{I}^c$, both $\xi_{\bar{m}}^{i;CF}$ and $\xi_{\bar{m}}^{i;F}$ are well-defined now. Moreover, $\xi_{\bar{m},u}^{i;F}$ is the $u^{\text{th}}$ i.i.d. copy of $\xi_{m}^{i;F}$ similarly.

In the remaining sequel, we investigate the consistency of our RCL estimator based on some basics of orthogonal machine learning theory. To start with, we state the assumptions on the consistency rates for the nuisance parameters in Assumption \ref{RCL assumption}. Only the assumptions that are helpful in studying the consistency of our RCL estimator are stated. Other assumptions that concentrate on the conditions of the scores, including orthogonality, identifiability, non-degeneracy, smoothness, and the regularity of moments can be found in \cite{mackey2018orthogonal} and references therein.

\begin{assumption}\label{RCL assumption}
Given that the nuisance parameters and the true nuisance parameters are $(\hat{g}^{i},\hat{\pi}^{i})$ and $(g^{i},\pi^{i})$, and $S_k=\{\bm{\alpha}=(\alpha_{1},\alpha_{2})\in\mathbb{Z}^{2}_{\geq 0}:\left\|\bm{\alpha}\right\|_{1}\leq k\}$, we have
%\begin{enumerate}
%\item 
\begin{align*}
\text{a) } \quad & \mathbb{E}\big[\left|\hat{g}^{i}(\mathbf{Z})-g^{i}(\mathbf{Z})\right|^{4\alpha_{1}}\left|\hat{\pi}^{i}(\mathbf{Z})-\pi^{i}(\mathbf{Z})\right|^{4\alpha_{2}}\mid \hat{g}^{i},\hat{\pi}^{i}\big] \\
& \overset{p}{\longrightarrow} 0, \forall\;\bm{\alpha}\in S_k,\\
%\item 
\text{b) } \quad & N^{\frac{1}{2}}\sqrt{\mathbb{E}\big[\left|\hat{g}^{i}(\mathbf{Z})-g^{i}(\mathbf{Z})\right|^{2\alpha_{1}}\left|\hat{\pi}^{i}(\mathbf{Z})-\pi^{i}(\mathbf{Z})\right|^{2\alpha_{2}}\mid \hat{g}^{i},\hat{\pi}^{i}\big]}\\
& \overset{p}{\longrightarrow} 0, \forall\;\bm{\alpha}\in \{\bm{\alpha}\in\mathbb{Z}^{2}_{\geq 0}:\left\|\bm{\alpha}\right\|_{1}\leq k+1\}\backslash S_k.
\end{align*}
%\end{enumerate}
\end{assumption}
%\begin{figure}[h]
%\includegraphics[width=8cm]{Image_20220128120317}
%\end{figure}

\subsection{Consistency Results}\label{sec:Consistency Results}

For notational convenience, we rewrite our RCL estimator $\hat{\theta}^{i}_{RCL}$ and denote it as $\hat{\theta}^{i}_{N}$. Indeed, we have
\begin{equation}
{
\begin{aligned}
\hat{\theta}^{i}_{N}=&\underbrace{\frac{1}{N}\underset{m=1}{\overset{N}{\sum}}\hat{g}^{i}(\mathbf{Z}_{m})}_{(a)}+\underbrace{\frac{1}{N}\underset{m\in \mathscr{I}}{\overset{}{\sum}}(Y^{i;F}_{m}-\hat{g}^{i}(\mathbf{Z}_{m}))\hat{A}_{m}^{i}}_{(b)}\\
&+\underbrace{\frac{1}{R}\overset{R}{\underset{u=1}{\sum}}\left[\frac{1}{N}\underset{m\in \mathscr{I}^{c}}{\overset{}{\sum}}\hat{\xi}_{m,u}^{i;F} \hat{A}_{m}^{i}\right]}_{(c)}.
\end{aligned}
}\label{eqn:theta_i_N}
\end{equation}
Here, $A^i_m=A(D_m,\mathbf{Z}_m;\pi^i)$ and $\hat{A}^i_m$ is the estimate of $A^i_m$. Besides, we define two quantities $\hat{\tilde{\theta}}_{N}^{i}$ and $\hat{\bar{\theta}}_{N}^{i}$. They are
{
\begin{align}
\hat{\tilde{\theta}}_{N}^{i}&=\frac{1}{N}\underset{m=1}{\overset{N}{\sum}}\hat{g}^{i}(\mathbf{Z}_{m})+\frac{1}{N}\underset{m\in \mathscr{I}}{\overset{}{\sum}}(Y^{i;F}_{m}-\hat{g}^{i}(\mathbf{Z}_{m}))\hat{A}_{m}^{i}\nonumber\\
&+\frac{1}{N}\underset{m\in \mathscr{I}^{c}}{\overset{}{\sum}}\hat{\xi}^{i;CF}_{m}\hat{A}_{m}^{i},\label{eqt:higher-order orthogonal condition before estimator in appendix}\\
\hat{\bar{\theta}}_{N}^{i}&=\frac{1}{N}\underset{m=1}{\overset{N}{\sum}}\hat{g}^{i}(\mathbf{Z}_{m})+\frac{1}{N}\underset{m\in \mathscr{I}}{\overset{}{\sum}}(Y^{i;F}_{m}-\hat{g}^{i}(\mathbf{Z}_{m}))\hat{A}_{m}^{i}\nonumber\\
&+\frac{1}{N}\underset{m\in \mathscr{I}^{c}}{\overset{}{\sum}}\hat{\xi}_{m}^{i;F}\hat{A}_{m}^{i}.\label{eqt:higher-order orthogonal condition final estimator in appendix}
\end{align}
}\noindent
We also define
\begin{equation*}
{
%\resizebox{0.49\textwidth}{!}{$
\begin{aligned}
\kappa_{N}^{i;F}&=\frac{1}{N}\underset{m\in \mathscr{I}^{c}}{\overset{}{\sum}}\xi_{m}^{i;F}A_{m}^{i},\quad\hat{\kappa}_{N}^{i;F}=\frac{1}{N}\underset{m\in \mathscr{I}^{c}}{\overset{}{\sum}}\hat{\xi}_{m}^{i;F}\hat{A}_{m}^{i},\\
\kappa_{N}^{i;CF}&=\frac{1}{N}\underset{m\in \mathscr{I}^{c}}{\overset{}{\sum}}\xi_{m}^{i;CF}A_{m}^{i},\quad\hat{\kappa}_{N}^{i;CF}=\frac{1}{N}\underset{m\in \mathscr{I}^{c}}{\overset{}{\sum}}\hat{\xi}_{m}^{i;CF}\hat{A}_{m}^{i},\\
\kappa_{R,N}^{i;F}&=\frac{1}{R}\overset{R}{\underset{u=1}{\sum}} \big[\frac{1}{N}\underset{m\in \mathscr{I}^{c}}{\overset{}{\sum}}\xi_{m,u}^{i;F} A_{m}^{i}\big],\\
\hat{\kappa}_{R,N}^{i;F}&=\frac{1}{R}\underset{u=1}{\overset{R}{\sum}}\big[\frac{1}{N}\underset{m\in \mathscr{I}^{c}}{\overset{}{\sum}}\hat{\xi}_{m,u}^{i;F}\hat{A}_{m}^{i}\big].
\end{aligned}
%$}
}
\end{equation*}
Then \eqref{eqn:theta_i_N}, \eqref{eqt:higher-order orthogonal condition before estimator in appendix}, and \eqref{eqt:higher-order orthogonal condition final estimator in appendix} can be rewritten as
{
\begin{align}
\hat{\theta}_{N}^{i}&=\frac{1}{N}\underset{m=1}{\overset{N}{\sum}}\hat{g}^{i}(\mathbf{Z}_{m})\nonumber
+\frac{1}{N}\underset{m\in \mathscr{I}}{\overset{}{\sum}}(Y^{i;F}_{m}-\hat{g}^{i}(\mathbf{Z}_{m}))\hat{A}_{m}^{i}+\hat{\kappa}_{R,N}^{i;F},\tag{\ref{eqn:theta_i_N}}\\
\hat{\tilde{\theta}}_{N}^{i}&=\frac{1}{N}\underset{m=1}{\overset{N}{\sum}}\hat{g}^{i}(\mathbf{Z}_{m})\nonumber
+\frac{1}{N}\underset{m\in \mathscr{I}}{\overset{}{\sum}}(Y^{i;F}_{m}-\hat{g}^{i}(\mathbf{Z}_{m}))\hat{A}_{m}^{i}+\hat{\kappa}_{N}^{i;CF},\tag{\ref{eqt:higher-order orthogonal condition before estimator in appendix}}\\
\hat{\bar{\theta}}_{N}^{i}&=\frac{1}{N}\underset{m=1}{\overset{N}{\sum}}\hat{g}^{i}(\mathbf{Z}_{m})\nonumber
+\frac{1}{N}\underset{m\in \mathscr{I}}{\overset{}{\sum}}(Y^{i;F}_{m}-\hat{g}^{i}(\mathbf{Z}_{m}))\hat{A}_{m}^{i}+\hat{\kappa}_{N}^{i;F},\tag{\ref{eqt:higher-order orthogonal condition final estimator in appendix}}
\end{align}
}\noindent
for simplicity. In addition, we have to use two lemmas and two propositions to study the consistency of $\hat{\theta}_{N}^{i}$. We state them with the proofs of the two lemmas as follows. The full proofs of the remaining Proposition \ref{lemma:difference from the statistical standpoint} and Proposition \ref{lemma:unbiasedness and consistency} are left in Appendix \ref{appendix:B} and Appendix \ref{appendix:C} respectively.

\begin{lemma}\label{lemma:simple lemma}	
Given two sequences of random variables $(X_{N})_{N=1}^{\infty}$ and $(Y_{N})_{N=1}^{\infty}$ such that $X_{N}\overset{d}{=}Y_{N}$. If $X_{N}\overset{p}{\rightarrow}c$ for some constant $c$, then $Y_{N}\overset{p}{\rightarrow}c$.	
\end{lemma}
\begin{proof}{proof}
Let $f_{X_{N}}(\cdot)$ and $f_{Y_{N}}(\cdot)$ be the density functions of the random variables $X_{N}$ and $Y_{N}$ respectively. Since $X_{N}\overset{d}{=}Y_{N}$, $f_{X_{N}}(\cdot)=f_{Y_{N}}(\cdot)$. Hence, $\mathbb{P}\left\{\left|X_{N}-c\right|\geq \epsilon\right\}=\int_{|z-c|\geq\epsilon}f_{X_{N}}(z)dz=\int_{|z-c|\geq\epsilon}f_{Y_{N}}(z)dz=\mathbb{P}\left\{\left|Y_{N}-c\right|\geq \epsilon\right\}$. Consequently, $X_{N}\overset{p}{\rightarrow}c$ implies $Y_{N}\overset{p}{\rightarrow}c$.
\end{proof}

\begin{lemma}\label{lemma:simple lemma2}	
Given random variables $X$, $Y$, $E$, $Z$, if $(X \overset{d}{=} Y) \mid Z$, $(X \perp \!\!\! \perp E) \mid Z$, $(Y \perp \!\!\! \perp E) \mid Z$, then $Xh(E,Z) \overset{d}{=} Yh(E,Z)$ for any function $h$.
\end{lemma}
\begin{proof}{proof}
	Define $f_Z(z)$ as the density function of $Z$, $f_{X|Z}(x|z)$ as the conditional density function of $X|Z$. The same applies for $f_{Y|Z}(y|z)$ and $f_{E|Z}(e|z)$. In addition,  $f_{X,E|Z}(x,e|z)$ and $f_{Y,E|Z}(y,e|z)$ are the conditional joint density functions of $X,E|Z$ and $Y,E|Z$ respectively. For a measurable set $\mathcal{A}$, we have
		\begin{equation*}
		{\small
		\begin{aligned}
		&\mathbb{P}\{Xh(E,Z)\in \mathcal{A}\}\\
		%=&\int_{\Omega_{Z}}\mathbb{P}\{Xh(E,Z)\in \mathcal{A}| Z=z\}f_{Z}(z)dz\\
		%=&\int_{\Omega_{Z}}\mathbb{P}\{Xh(E,z)\in \mathcal{A}| Z=z\}f_{Z}(z)dz\\
		=&\int_{\Omega_{Z}}\!\!\left\{\iint_{\Omega_{X} \times \Omega_{E}} \!\!\mathbf{1}_{\{xh(e,z)\in \mathcal{A}\}}f_{X,E|Z}(x,e|z)dxde\right\}\!\!f_{Z}(z)dz\\
		\overset{*}{=}&\int_{\Omega_{Z}}\!\!\left\{\iint_{\Omega_{X} \times \Omega_{E}} \!\!\mathbf{1}_{\{xh(e,z)\in \mathcal{A}\}}f_{X|Z}(x|z)f_{E|Z}(e|z)dxde\right\}\!\!f_{Z}(z)dz\\
		\overset{\triangle}{=}&\int_{\Omega_{Z}}\!\!\left\{\iint_{\Omega_{Y} \times \Omega_{E}} \!\!\mathbf{1}_{\{yh(e,z)\in \mathcal{A}\}}f_{Y|Z}(y|z)f_{E|Z}(e|z)dyde\right\}\!\!f_{Z}(z)dz\\
		\overset{\square}{=}&\int_{\Omega_{Z}}\!\!\left\{\iint_{\Omega_{Y} \times \Omega_{E}} \!\!\mathbf{1}_{\{yh(e,z)\in \mathcal{A}\}}f_{Y,E|Z}(y,e|z)dyde\right\}\!\!f_{Z}(z)dz\\
		=&\mathbb{P}\{Yh(E,Z)\in \mathcal{A}\}.
		\end{aligned}	
		}
		\end{equation*}
$*$ holds since $(X \perp \!\!\! \perp E) \mid Z$, $\triangle$ holds since $(X \overset{d}{=} Y) \mid Z$, and $\square$ holds since $(Y \perp \!\!\! \perp E) \mid Z$.
\end{proof}

\begin{proposition}\label{lemma:difference from the statistical standpoint}
Suppose {$\mathbb{E}\left[(\xi^{i;F}_{m})^{2}\mid\mathbf{Z}\right]$} and {$(A^{i}_m)^2$} exist such that $\mathbb{E}\left[(A_m^{i})^2\mathbb{E}\left[(\xi_m^{i;F})^{2}\mid\mathbf{Z}\right]\right]$ is finite for all $m$. We have {$\kappa_{N}^{i;CF}-\kappa_{N}^{i;F}\overset{p}{\rightarrow} 0$} when {$N\rightarrow \infty$}.
\end{proposition}

\begin{proposition}\label{lemma:unbiasedness and consistency}
Suppose that, conditioning on {$\mathbf{Z}$}, {$\xi_{m,u}^{i;F}$} are i.i.d. of {$\xi_{m}^{i;F}$} and {$\xi_{m,u}^{i;F}$} are i.i.d. of {$\xi_{m,\bar{u}}^{i;F}$}, {$\forall u,\;\bar{u}\in\{1,2,\dots,R\}$}. We have
\begin{equation*}
\kappa_{N}^{i;F}-\kappa_{R,N}^{i;F}\overset{p}{\longrightarrow}0\quad \textit{when}\quad N\rightarrow \infty.\label{item:statistical difference 2}
\end{equation*}
\end{proposition}

Now, we are ready to investigate if the estimator $\hat{\theta}_{N}^{i}$ is a consistent estimator of $\theta^{i}$. Our goal is to show that $\mathbb{P}_{\hat{\rho}}\left\{\left|\hat{\theta}_{N}^{i}-\theta^{i}\right| \geq \epsilon\right\}\overset{p}{\rightarrow} 0$, where $\hat{\rho}=(\hat{g}^i, \hat{\pi}^i)$ are the estimated nuisance parameters. Before presenting the proof, we notice that for $m\in \mathscr{I}^{c}$, $\left(\hat{\xi}^{i;F}\overset{d}{=}\hat{\xi}^{i;CF}\right)\mid\mathbf{Z}$  when $\hat{g}^{i}$ and $g^{i}$ satisfy Assumption \ref{RCL assumption}. This is because that $\hat{\xi}^{i;F}=Y^{i;F}-\hat{g}^{i}(\mathbf{Z})=\xi^{i;F} + g^{i}(\mathbf{Z})-\hat{g}^{i}(\mathbf{Z})$, $\hat{\xi}^{i;CF}=Y^{i;CF}-\hat{g}^{i}(\mathbf{Z})=\xi^{i;CF} + g^{i}(\mathbf{Z})-\hat{g}^{i}(\mathbf{Z})$, {$(\xi^{i;F} \perp \!\!\! \perp D) \mid \mathbf{Z}$}, {$(\xi^{i;CF} \perp \!\!\! \perp D) \mid \mathbf{Z}$}, and $\left(\xi^{i;F}\overset{d}{=}\xi^{i;CF}\right)\mid\mathbf{Z}$. Our concluding result is that under the assumptions stated before, the estimator $\hat{\theta}_{N}^{i}$ is indeed a consistent estimator of $\theta^{i}$. Due to the space limit, we defer the full proof of consistency in Appendix \ref{appendix:D}.

\section{Numerical Studies}\label{sec:experiment}
In this section, we compare the performances of our RCL estimators with the DR, IPW, DML estimators and their variants through simulation and empirical experiments. In both experiments, we consider three types of regressors for $g^i$: Lasso, Random Forests (RF), and Multi-layer Perceptron (MLP); and three types of classifiers for $\pi^i$: Logistic Regression (LR), RF, and MLP. We combine the regression model A and the classification model B to estimate {$g^{i}$} and {$\pi^i$} respectively, and denote the combination as A+B, e.g., Lasso+LR. In the empirical experiments, we consider two additional state-of-the-art neural network models in causal learning: TARNet \cite{shalit2017estimating} and Dragonnet \cite{shi2019adapting}. All the experiments are run on Dell 3640 with an Intel(R) Xeon(R) W-1290P CPU at 3.70GHz, and a NVIDIA GeForce RTX 2080Ti GPU.

The RCL estimator with the values of {$r$} and {$k$} (see Theorem \ref{thm:$k^{th}$-order orthogonal condition score function}) is denoted as RCL{$_{r,k}$}.
For all the experiments throughout the paper, we use the following two metrics to evaluate the performance:
\vspace{-0.1cm}
\begin{subequations}
\begin{equation}
{
\begin{gathered}
\epsilon_{ATE}=\frac{1}{M}\underset{m=1}{\overset{M}{\sum}}\epsilon_{ATE;m};
\end{gathered} \label{eqt:metric}
}
\end{equation}
\vspace{-0.2cm}
%and 
\begin{equation}
{
\begin{gathered}
\sigma_{ATE}=\sqrt{\frac{1}{M-1}\underset{m=1}{\overset{M}{\sum}}\left(\epsilon_{ATE;m}-\epsilon_{ATE}\right)^{2}}.
\end{gathered} \label{eqt:metric variance}
}
\end{equation}
\end{subequations}
Here, {$\epsilon_{ATE;m}$} is the relative error in the $m^{\mathrm{th}}$ experiment or simulation trial where 
$$\epsilon_{ATE;m}=\frac{\underset{1 \leq i<j \leq n}{\overset{}{\sum}}\left|\hat{\theta}^{i,j;m}-\theta^{i,j;m}\right|}{\underset{1 \leq i<j \leq n}{\overset{}{\sum}}\left|\theta^{i,j;m}\right|}$$ 
with {$\theta^{i,j;m}$} and {$\hat{\theta}^{i,j;m}$} being the true ATE and the estimated ATE between the treatment {$d^i$} and the treatment {$d^j$} in the {$m^{\mathrm{th}}$} experiment. {$n$} is the number of treatments and {$M$} is the number of experiments or simulation trials. {$\epsilon_{ATE}$} and {$\sigma_{ATE}$} are the average and standard deviation of these relative errors.

\subsection{Numerical Studies on Simulation Datasets}\label{sec:Numerical Studies on Simulated Data}

We first introduce the data generating process (DGP) for the simulation experiments. Given the covariates {$\mathbf{Z}=(Z_{1},...,Z_{p})^{T}$} which follow a standard multivariate Gaussian distribution, the treatment variable {$D$} has the treatment space {$\{d^{1}, d^{2}, d^{3}\}$} with the corresponding probability
\begin{equation}\label{eqt:prob_i}
{
\begin{aligned}
\pi^{i}(\mathbf{Z}) = \mathbb{P}\{D=d^{i} | \mathbf{Z}\} = \frac{\exp\left(\underset{u=1}{\overset{\lfloor p\cdot r_{c}\rfloor}{\sum}}\beta_{iu}Z_u\right)}{\underset{j=1}{\overset{3}{\sum}}\exp\left(\underset{u=1}{\overset{\lfloor p\cdot r_{c}\rfloor}{\sum}}\beta_{ju}Z_u\right)},
\end{aligned}
}
\end{equation}
where the values of coefficients {$\beta_{iu}$} are randomly picked from the uniform distribution {$\mathcal{U}(-0.1, 0.1)$}. {$r_c$} is the confounding ratio ranging from {$0$} to {$1$}, and the number of covariates in {$\mathbf{Z}$} used to generate {$D$} is {$\lfloor p\cdot r_{c}\rfloor\in\mathbb{N}$}. For example, if {$p=10$} and {$r_{c} = 0.56$}, then {$p\cdot r_c = 5.6$} and {$\lfloor p\cdot r_{c}\rfloor= 5$}. We generate the potential outcome {$Y^{i}$} for treatment indices {$i \in \{1,2,3\}$} as
\begin{equation}
{
\begin{aligned}
	Y^{i}& = g^{i}(\mathbf{Z})+\xi^{i} = e^{\sqrt{d^{i}}}\left(\bm{a}_{i}^{T}\mathbf{Z}+1\right)^2+\xi^{i},
	%		Y^{1}&=\left(0.1\left(-Z_1+1.5Z_2\right)^2+e^{Z_1}-e^{0.5Z_2}\right)e^{\sqrt{0.1}}+\xi^{1}\\
	%		Y^{2}&=\left(0.1\left(-Z_1+1.5Z_2\right)^2+e^{Z_1}-e^{0.5Z_2}\right)e^{\sqrt{0.5}}+\xi^{2}\\
	%		Y^{3}&=\left(0.1\left(-Z_1+1.5Z_2\right)^2+e^{Z_1}-e^{0.5Z_2}\right)e^{\sqrt{1.0}}+\xi^{3}
\end{aligned}
}
\end{equation}
where {$\bm{a}_{i}$} is a {$p\times 1$} constant vector whose elements are randomly chosen from {$\mathcal{U}(0.1, 0.5)$}. We also set {$d^{1}=0.1$}, {$d^{2}=0.5$}, {$d^{3}=1$}, {$\xi^{1} \sim \mathcal{N}(0,9)$}, {$\xi^{2} \sim \mathcal{N}(0,4)$}, and {$\xi^{3} \sim \mathcal{N}(0,1)$}. Next, we generate {$N$} i.i.d. observations based on the DGP. Suppose the realized covariates of the {$m^{\mathrm{th}}$} individual are {$\mathbf{z}_{m}$}, then the actual treatment {$d_{m}$} will be $d^{k}$, where $k$ is determined by {$k=\underset{u \in \{1,2,3\}}{\arg\max} \; \pi^{u}(\mathbf{z}_{m})$}. Under the actual treatment $d^{k}$, the observed factual outcome {$y_{m}$} will correspondingly be a realization of {$Y^{k}$}.

For the simulation experiments, we compute the DR, DML, and our RCL estimators with different values of {$r$} and {$k$} (see Theorem \ref{thm:$k^{th}$-order orthogonal condition score function}), which is denoted by RCL{$_{r,k}$}. We then use {$\epsilon_{ATE}$} in \eqref{eqt:metric} with {$n=3$} and {$M=100$} to evaluate the performance of different estimators for each combination of the regressor {$g^{i}$} and the classifier {$\pi^{i}$} (denoted as regressor+classifier). We split every dataset by the ratio {$56\%/14\%/30\%$} as training/validation/test sets.

\textbf{Consistency of RCL estimators.}  In this part, we first set {$r_c=1$}, {$p=5$}, and let the number of observations {$N$} vary in {$\{1, 2, 4, 8, 16\}\times 10000$}. We check the consistency of RCL estimators through simulations and report {$\epsilon_{ATE}$} in Fig. \ref{Figure:ATE_convergency_r=4}. The result indicates that the error reduces when the sample size increases for our RCL estimators. Besides, we also find that when {$g^{i}$} is fitted well, {$\forall i$} (e.g., when the regressor is chosen as Lasso or RF), RCL{$_{2,2}$} performs better than DR, DML, and other RCL estimators. On the other hand, when {$g^{i}$} is not fitted well for some {$i$}, e.g., when the regressor is chosen as MLP, the DML and the RCL estimators with $k=1$ can significantly correct the bias thanks to the doubly robust property. In this case, despite similar performances produced by RCL{$_{2,1}$} estimator and the DML estimator, RCL{$_{2,1}$} still has a smaller {$\epsilon_{ATE}$}.

\begin{figure}[t]
	\centering	
	\includegraphics[width=0.9\columnwidth]{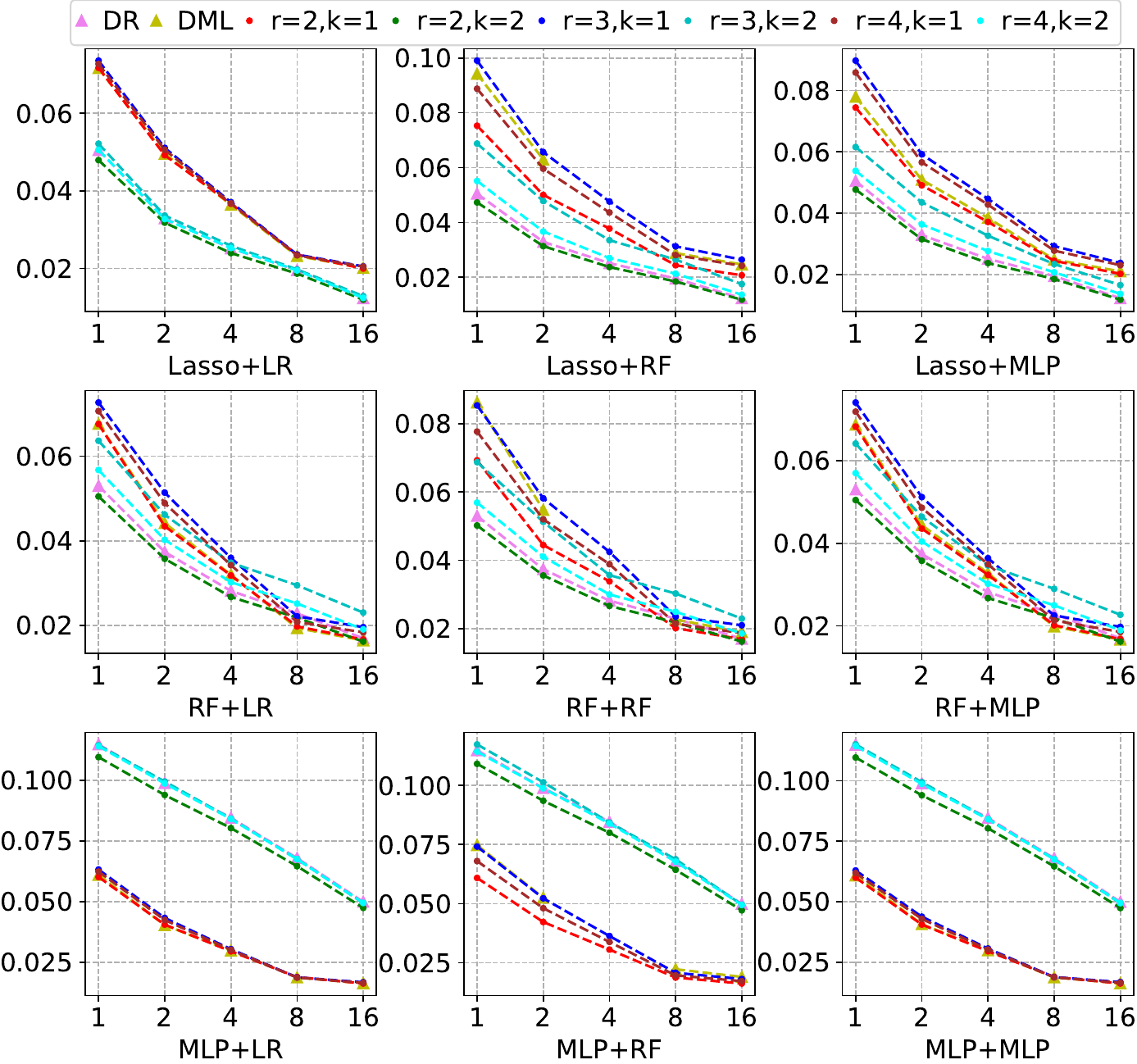}
	\caption{Plots of {{$\epsilon_{ATE}$}} versus the varying {$N$}: DR v.s. DML v.s. RCL.}
	\label{Figure:ATE_convergency_r=4}
\end{figure}

\begin{figure}[t]
\centering
\includegraphics[width=0.9\columnwidth]{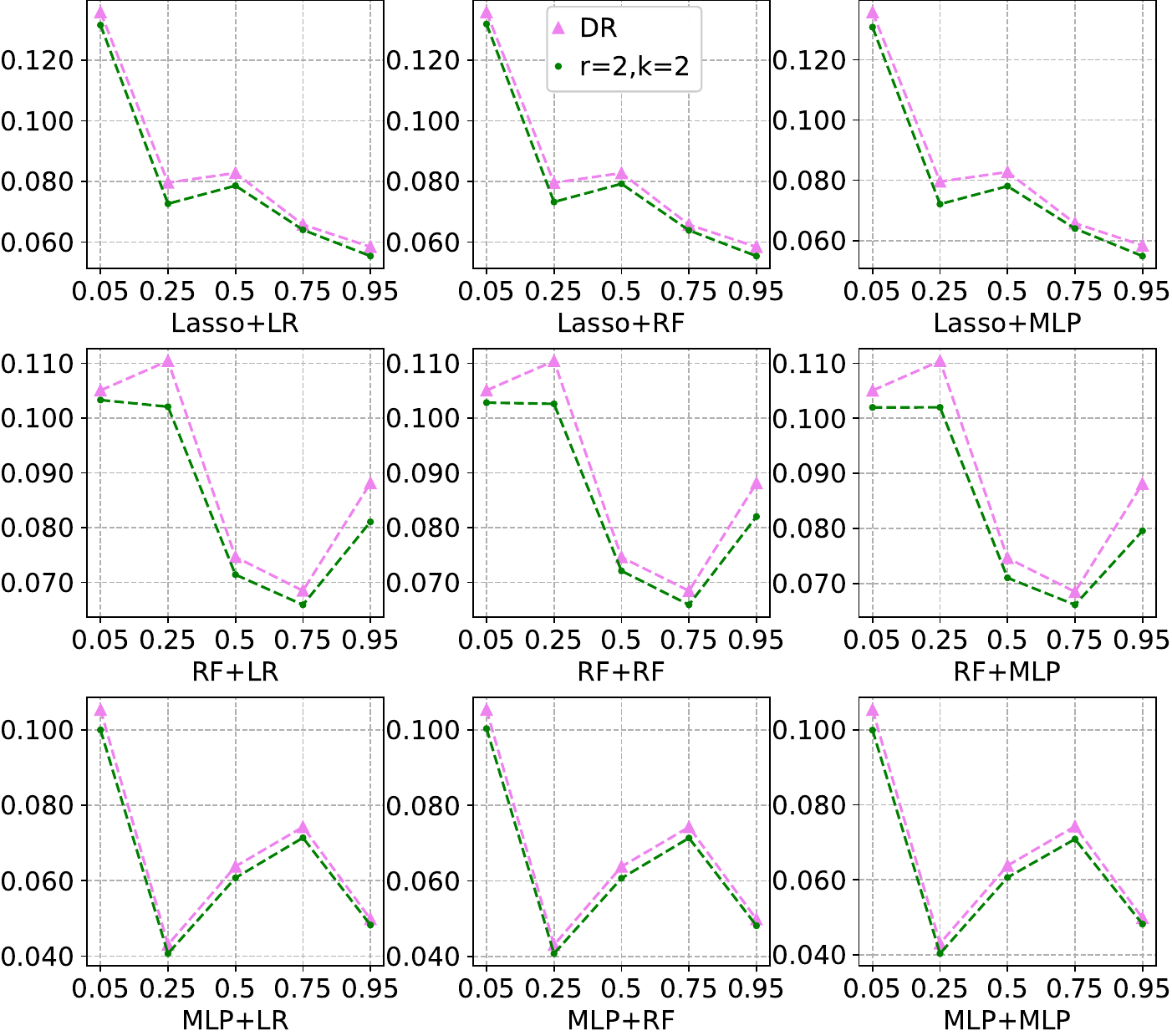}
\caption{Plots of {$\epsilon_{ATE}$} versus the varying {$r_c$}: DR v.s. RCL{$_{2,2}$}.}
\label{Figure:confounding_ratio_DRvsROL}
\end{figure}

\begin{figure}[t]
\centering
\includegraphics[width=0.9\columnwidth]{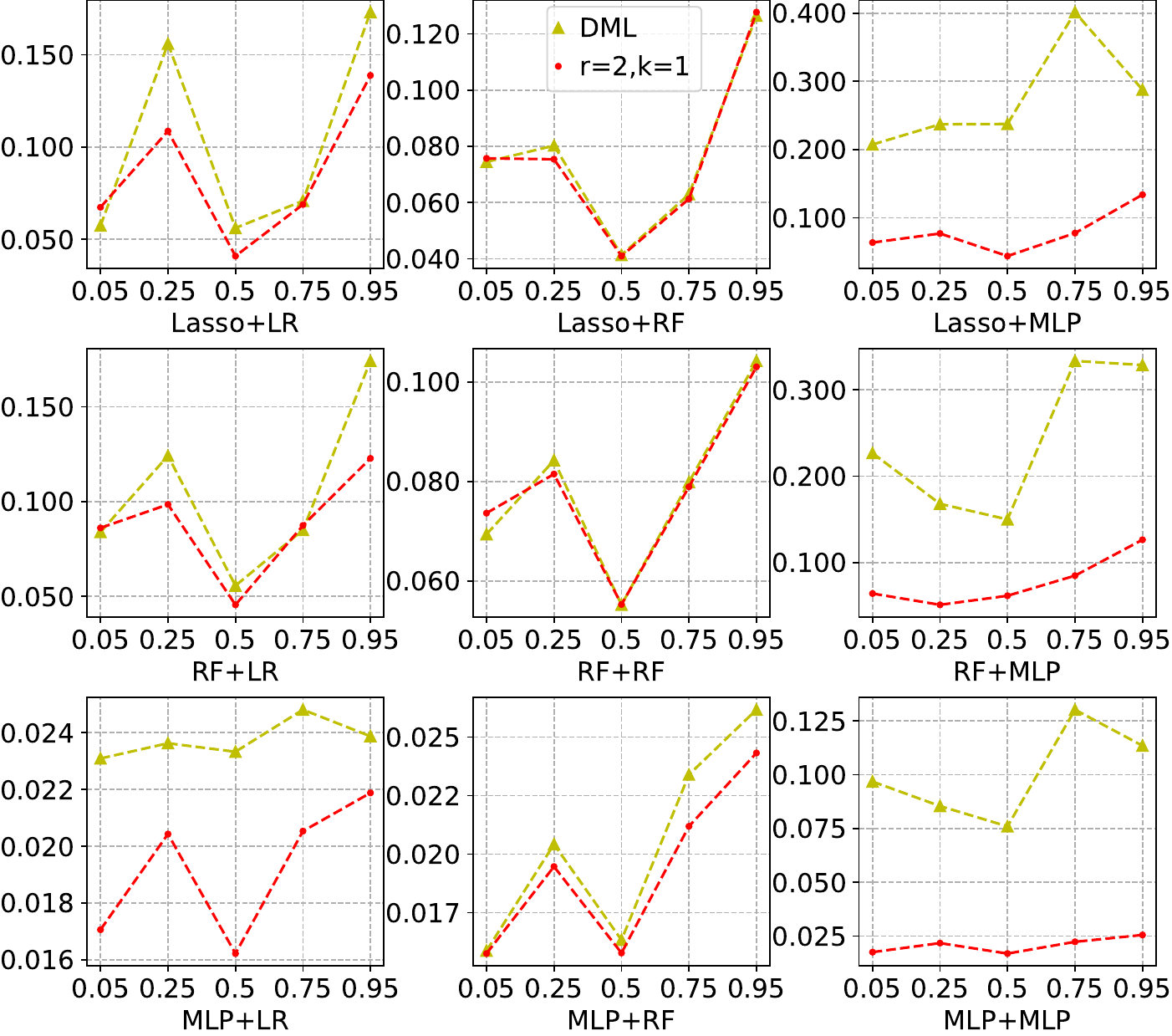}
\caption{Plots of {$\epsilon_{ATE}$} versus the varying {$r_c$}: DML v.s. RCL{$_{2,1}$}.}
\label{Figure:confounding_ratio_DMLvsROL}
\end{figure}

\begin{figure}[t]
\centering
\includegraphics[width=0.9\columnwidth]{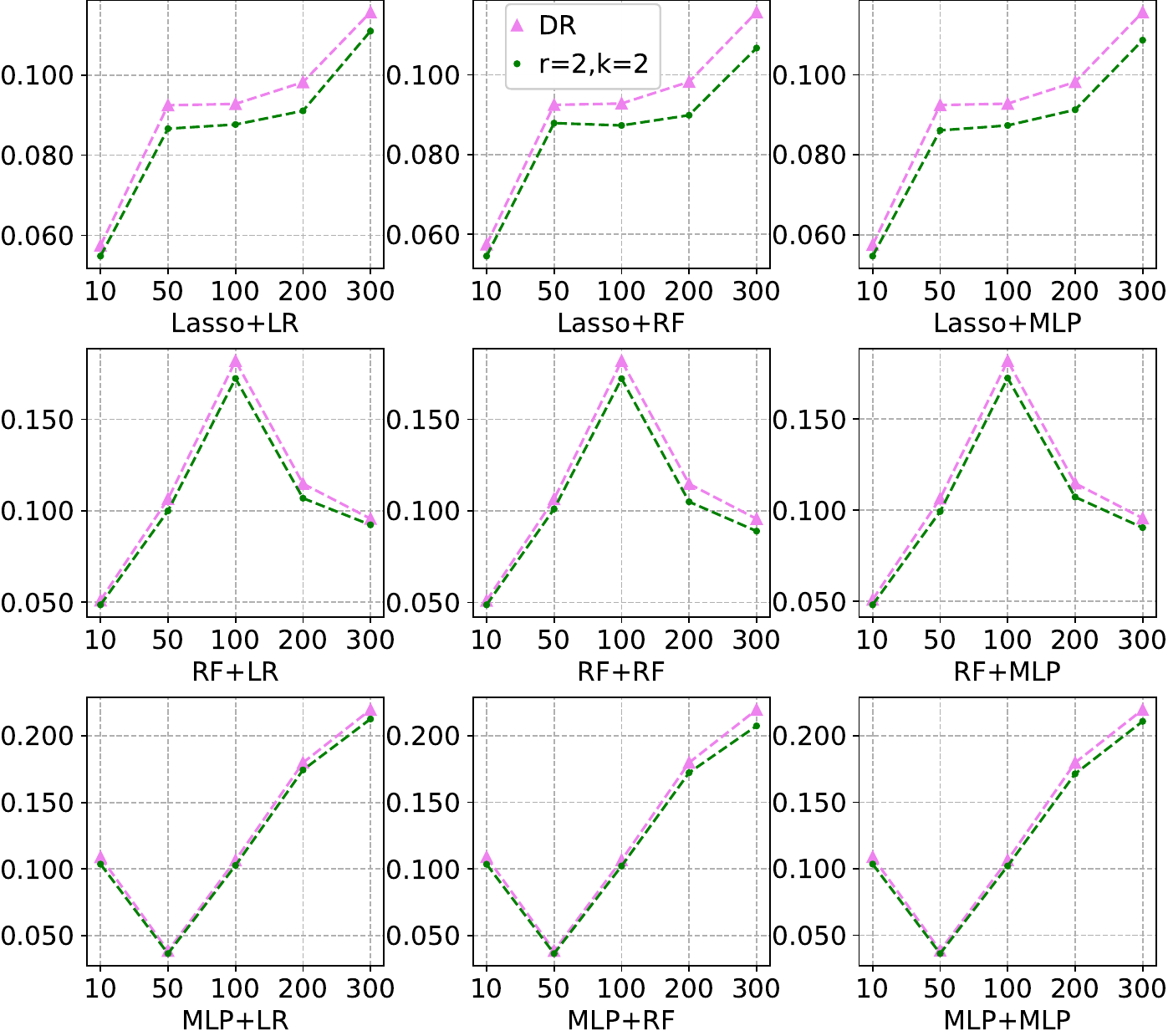}
\caption{Plots of {$\epsilon_{ATE}$} versus the varying {$p$}: DR v.s. RCL{$_{2,2}$}.}
\label{Figure:dim_DRvsROL}
\end{figure}

\begin{figure}[t]
\centering
\includegraphics[width=0.9\columnwidth]{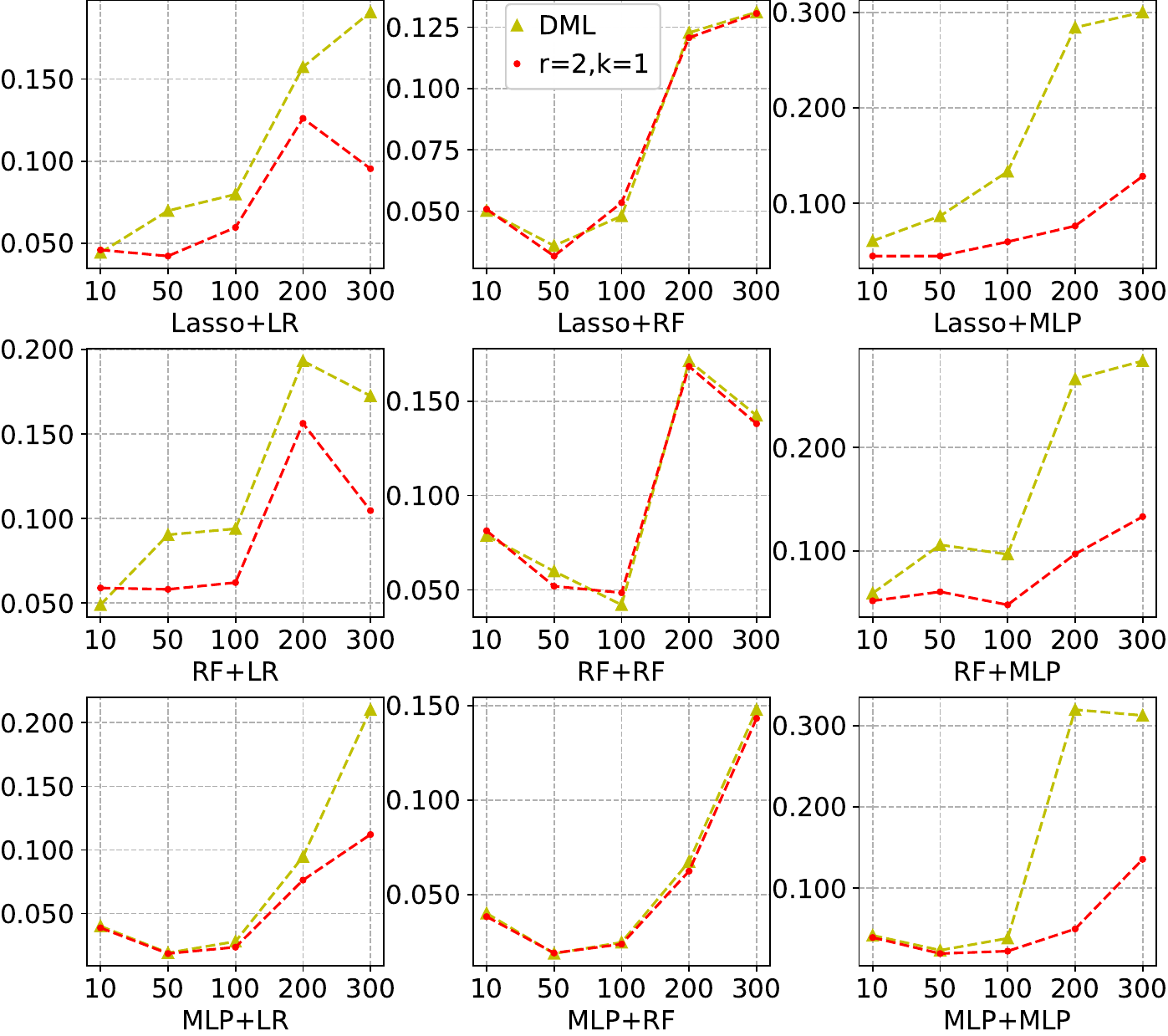}
\caption{Plots of {$\epsilon_{ATE}$} versus the varying {$p$}: DML v.s. RCL{$_{2,1}$}.}
\label{Figure:dim_DMLvsROL}
\end{figure}

\textbf{Varying {$r_{c}$} and {$p$}.} In the following experiments, we mainly compare RCL{$_{2,2}$} with DR, and RCL{$_{2,1}$} with DML since the above simulation experiments indicate that RCL{$_{2,2}$} and DR perform similarly, and RCL{$_{2,1}$} has similar trends to DML. We set {$N=10000$} and plot {$\epsilon_{ATE}$} produced by each model combination A+B versus i) different {$r_{c}$} with {$p = 100$} in Fig. \ref{Figure:confounding_ratio_DRvsROL} and Fig. \ref{Figure:confounding_ratio_DMLvsROL}; ii) different {$p$} with {$r_{c}=1$} in Fig. \ref{Figure:dim_DRvsROL} and Fig. \ref{Figure:dim_DMLvsROL}. From all the four figures, we observe that the DML estimator is sensitive to the change of {$r_{c}$} and {$p$}, especially when the classifier is MLP. As analyzed before, if the estimation error of {$\pi^{i}(\mathbf{z})$} is non-negligible for some $\mathbf{z}$, the term {$\frac{1}{\hat{\pi}^i(\cdot)}$} of the DML estimator often gives extreme values especially when {$\hat{\pi}^{i}(\mathbf{z})$} is small, leading to pronounced estimation error of ATE. Indeed, any ATE estimator that involves the inverse propensity score term might face this error-compounding issue. By contrast, our RCL estimators are less volatile to the variation of {$r_{c}$} and {$p$} regardless of the choice of the classifier. For example, in Fig. \ref{Figure:dim_DMLvsROL}, we notice that when the classifier is MLP, the error of DML rises dramatically as {$p$} increases, while RCL{$_{2,1}$} performs more steadily. In addition, our RCL{$_{2,2}$} estimator overall has a smaller {$\epsilon_{ATE}$} than the DR estimator no matter how {$r_{c}$} or {$p$} varies.

%\begin{landscape}
\begin{table*}[t]
       %\vspace{3em}
	\centering
	\caption{The performance comparisons ({$\epsilon_{ATE} \pm \sigma_{ATE}$}) on the test sets of 1000 \textbf{IHDP} experiments. Smaller {$\epsilon_{ATE}$} is better.}
	\resizebox{2\columnwidth}{!}{
		\begin{tabular}{cccc|cccccc}
			\toprule
			Model/Estimator & DR    & RCL$_{2,2}$ & $R_{DR}$ & IPW   & AIPW  & DML   & DML-trim & RCL$_{2,1}$ & $R_{DML}$ \\
			\midrule
			LASSO+LR & \num{0.0957016592574597}$\pm$\num{0.640267801173915} & \num{0.0922338093233415}$\pm$\num{0.577391472448973} & -3.6\% & \num{1.51697356015605}$\pm$\num{4.68598260312737} & \num{0.105763367167141}$\pm$\num{0.300194083971267} & \num{0.105763367167142}$\pm$\num{0.300194083971274} & \num{0.105763364379096}$\pm$\num{0.300194084194018} & \num{0.104851958331204}$\pm$\num{0.651543659112282} & -0.9\% \\
			\midrule
			LASSO+RF & \num{0.0957016592574597}$\pm$\num{0.640267801173915} & \num{0.0919418871606084}$\pm$\num{0.564172928180374} & -3.9\% & $\infty$ & \num{0.109315513248708}$\pm$\num{0.280292014664635} & $\infty$ & \num{0.11131590918084}$\pm$\num{0.291181180228278} & \num{0.10262849397032}$\pm$\num{0.630798484505913} & -6.1\% \\
			\midrule
			LASSO+MLP & \num{0.0957016592574597}$\pm$\num{0.640267801173915} & \num{0.0927439264967579}$\pm$\num{0.593285296206187} & -3.1\% & \num{4.4812490222709}$\pm$\num{9.85644139084187} & \num{0.224440190445999}$\pm$\num{0.439720169984648} & \num{0.224440190445998}$\pm$\num{0.439720169984648} & \num{0.19311183721892}$\pm$\num{0.323636374244496} & \num{0.102944659987956}$\pm$\num{0.590220193929016} & -47\% \\
			\midrule
			RF+LR & \num{0.0979732966141551}$\pm$\num{0.553484192555486} & \num{0.0932579681403953}$\pm$\num{0.439817975315819} & -4.8\% & \num{1.51697356015605}$\pm$\num{4.68598260312737} & \num{0.128811621589012}$\pm$\num{0.724585718194816} & \num{0.128811621589012}$\pm$\num{0.724585718194816} & \num{0.128811615676305}$\pm$\num{0.724585718507295} & \num{0.128842821124574}$\pm$\num{1.07488187681154} & 0.0\% \\
			\midrule
			RF+RF & \num{0.0979732966141551}$\pm$\num{0.553484192555486} & \num{0.0942932571289985}$\pm$\num{0.458204054417463} & -3.8\% & $\infty$ & \num{0.126178712647616}$\pm$\num{0.638434948404866} & $\infty$ & \num{0.128012877019247}$\pm$\num{0.641268017780346} & \num{0.124621883086877}$\pm$\num{1.01221269744655} & -1.2\% \\
			\midrule
			RF+MLP & \num{0.0979732966141551}$\pm$\num{0.553484192555486} & \num{0.0950802814635367}$\pm$\num{0.506860794218963} & -3.0\% & \num{4.4812490222709}$\pm$\num{9.85644139084187} & \num{0.240480141381655}$\pm$\num{0.756511151700561} & \num{0.240480141381655}$\pm$\num{0.756511151700565} & \num{0.211482597532412}$\pm$\num{0.714411783365695} & \num{0.120590099634166}$\pm$\num{0.883571527750748} & -43\% \\
			\midrule
			MLP+LR & \num{0.0803597106682557}$\pm$\num{0.229579144869116} & \num{0.0790235103312214}$\pm$\num{0.229472621327462} & -1.7\% & \num{1.51697356015605}$\pm$\num{4.68598260312737} & \num{0.141389283589274}$\pm$\num{0.242181566583048} & \num{0.141389283589274}$\pm$\num{0.242181566583048} & \num{0.14138927504712}$\pm$\num{0.242181564569142} & \num{0.104578167591645}$\pm$\num{0.194193628272569} & -26\% \\
			\midrule
			MLP+RF & \num{0.0803597106682557}$\pm$\num{0.229579144869116} & \num{0.0792146226254808}$\pm$\num{0.23321971636867} & -1.4\% & $\infty$ & \num{0.140930737288581}$\pm$\num{0.244936557918274} & $\infty$ & \num{0.142145985158482}$\pm$\num{0.253716565669282} & \num{0.104382624931009}$\pm$\num{0.192382428814327} & -26\% \\
			\midrule
			MLP+MLP & \num{0.0803597106682557}$\pm$\num{0.229579144869116} & \num{0.0790382014019174}$\pm$\num{0.230289370683907} & -1.6\% & \num{4.4812490222709}$\pm$\num{9.85644139084187} & \num{0.388796191412648}$\pm$\num{1.11801289407659} & \num{0.388796144909208}$\pm$\num{1.11801240474788} & \num{0.340278452668023}$\pm$\num{0.951125010966888} & \num{0.111926344793105}$\pm$\num{0.235810638754905} & -67\% \\
			\midrule
			TARNet & \num{0.0542958025852373}$\pm$\num{0.0942507273076831} & \num{0.0533817393804141}$\pm$\num{0.0918769235152564} & -1.7\% & \num{1.27631249863221}$\pm$\num{4.1830027255246} & \num{0.0888336751009626}$\pm$\num{0.16250044416974} & \num{0.0888336437847067}$\pm$\num{0.162499466513899} & \num{0.0888336437847067}$\pm$\num{0.162499466513899} & \num{0.0827653474603322}$\pm$\num{0.197713461439485} & -6.8\% \\
			\midrule
			Dragonnet & \num{0.0562447309384025}$\pm$\num{0.092229254884046} & \num{0.0557504174151612}$\pm$\num{0.0952493269346849} & -0.9\% & \num{1.71597732398186}$\pm$\num{4.81418695573581} & \num{0.14127353877066}$\pm$\num{0.218639997523957} & \num{0.141273601602494}$\pm$\num{0.218640208253127} & \num{0.134702721565834}$\pm$\num{0.175778675670746} & \num{0.0816722418781141}$\pm$\num{0.105565159717243} & -39\% \\
			\bottomrule
		\end{tabular}%
	}
	\label{Table:ihdp}%
\end{table*}%
%\end{landscape}

%\begin{landscape}
\begin{table*}[t]
       %\vspace{3em}
	\centering
	\caption{The performance comparisons ({$\epsilon_{ATE} \pm \sigma_{ATE}$}) on the test sets of 100 \textbf{Twins} experiments. Smaller {$\epsilon_{ATE}$} is better.}
	\resizebox{2\columnwidth}{!}{
		\begin{tabular}{cccc|cccccc}
			\toprule
			Model/Estimator & DR    & RCL$_{2,2}$ & $R_{DR}$ & IPW   & AIPW  & DML   & DML-trim & RCL$_{2,1}$ & $R_{DML}$ \\
			\midrule
			LASSO+LR & \num{0.667161417975565}$\pm$\num{0.436943790888348} & \num{0.646031944788284}$\pm$\num{0.357810609635846} & -3.2\% & \num{0.997186654546707}$\pm$\num{1.0282479839264} & \num{0.863413170171111}$\pm$\num{0.867510368220328} & \num{0.863413170171112}$\pm$\num{0.867510368220329} & \num{0.863413170171112}$\pm$\num{0.867510368220329} & \num{0.860539637381624}$\pm$\num{0.861111151796158} & -0.3\% \\
			\midrule
			LASSO+RF & \num{0.667161417975565}$\pm$\num{0.436943790888348} & \num{0.640423458872207}$\pm$\num{0.346180238288157} & -4.0\% & \num{1.06774689631041}$\pm$\num{1.06667674961721} & \num{0.912542741177314}$\pm$\num{1.0324487837971} & \num{0.912542741177313}$\pm$\num{1.03244878379709} & \num{0.912541530559238}$\pm$\num{1.03244888082876} & \num{0.850930835731041}$\pm$\num{0.857651696249796} & -6.8\% \\
			\midrule
			LASSO+MLP & \num{0.667161417975565}$\pm$\num{0.436943790888348} & \num{0.652263208887379}$\pm$\num{0.361908866758876} & -2.2\% & \num{2.71695259473322}$\pm$\num{2.53524699605525} & \num{0.941193658246498}$\pm$\num{0.969680994312327} & \num{0.941193658246499}$\pm$\num{0.969680994312327} & \num{0.924807382693669}$\pm$\num{0.950933057524706} & \num{0.883284524889292}$\pm$\num{0.851073779438786} & -4.5\% \\
			\midrule
			RF+LR & \num{0.604186638495974}$\pm$\num{0.533120559180304} & \num{0.576393675529002}$\pm$\num{0.460608409580317} & -4.6\% & \num{0.997186654546707}$\pm$\num{1.0282479839264} & \num{0.782618974579244}$\pm$\num{0.863295302774404} & \num{0.782618974579244}$\pm$\num{0.863295302774404} & \num{0.782618974579244}$\pm$\num{0.863295302774404} & \num{0.774451947731414}$\pm$\num{0.842482685707567} & -1.0\% \\
			\midrule
			RF+RF & \num{0.604186638495974}$\pm$\num{0.533120559180304} & \num{0.573526624706884}$\pm$\num{0.444114614484407} & -5.1\% & \num{1.06774689631041}$\pm$\num{1.06667674961721} & \num{0.898143875599452}$\pm$\num{1.04811245118099} & \num{0.898143875599452}$\pm$\num{1.04811245118099} & \num{0.898142228357805}$\pm$\num{1.04811284592583} & \num{0.808639013083756}$\pm$\num{0.860979316949181} & -10\% \\
			\midrule
			RF+MLP & \num{0.604186638495974}$\pm$\num{0.533120559180304} & \num{0.582058275375756}$\pm$\num{0.467041836393486} & -3.7\% & \num{2.71695259473322}$\pm$\num{2.53524699605525} & \num{0.878704531177447}$\pm$\num{0.939717406064745} & \num{0.878704531177446}$\pm$\num{0.939717406064746} & \num{0.863391078384253}$\pm$\num{0.925297128540437} & \num{0.816129289317978}$\pm$\num{0.828611304610332} & -5.5\% \\
			\midrule
			MLP+LR & \num{0.660213653748403}$\pm$\num{0.643159826971596} & \num{0.624421249388836}$\pm$\num{0.561873977426024} & -5.4\% & \num{0.997186654546707}$\pm$\num{1.0282479839264} & \num{0.822347815959208}$\pm$\num{0.831239772507} & \num{0.822347815959207}$\pm$\num{0.831239772507} & \num{0.822347815959207}$\pm$\num{0.831239772507} & \num{0.817012134151291}$\pm$\num{0.827875248022536} & -0.6\% \\
			\midrule
			MLP+RF & \num{0.660213653748403}$\pm$\num{0.643159826971596} & \num{0.618517938691176}$\pm$\num{0.548179257855073} & -6.3\% & \num{1.06774689631041}$\pm$\num{1.06667674961721} & \num{0.939531974914234}$\pm$\num{1.00613870291003} & \num{0.939531974914234}$\pm$\num{1.00613870291003} & \num{0.939531328888188}$\pm$\num{1.00613877340099} & \num{0.84527502195103}$\pm$\num{0.827317719691973} & -10\% \\
			\midrule
			MLP+MLP & \num{0.660213653748403}$\pm$\num{0.643159826971596} & \num{0.629679914590112}$\pm$\num{0.562780660147385} & -4.6\% & \num{2.71695259473322}$\pm$\num{2.53524699605525} & \num{0.904776867847457}$\pm$\num{0.954993073594451} & \num{0.904776832924196}$\pm$\num{0.95499336557208} & \num{0.899217557920781}$\pm$\num{0.953636899261841} & \num{0.849724793635334}$\pm$\num{0.833052100933173} & -5.5\% \\
			\midrule
			TARNet & \num{0.656097905805353}$\pm$\num{0.600217082337399} & \num{0.621069087185189}$\pm$\num{0.509983143308584} & -5.3\% & \num{2.46916252879544}$\pm$\num{2.98410306499364} & \num{0.937674888447282}$\pm$\num{1.30252582133794} & \num{0.937675061452057}$\pm$\num{1.30252599182589} & \num{0.937675061452057}$\pm$\num{1.30252599182589} & \num{0.864644737634207}$\pm$\num{1.04331803418323} & -7.8\% \\
			\midrule
			Dragonnet & \num{0.677258215220963}$\pm$\num{0.634940870296755} & \num{0.642171326254504}$\pm$\num{0.561434959073984} & -5.2\% & \num{1.66850913554347}$\pm$\num{1.65120065182147} & \num{0.794897423744431}$\pm$\num{0.740918933191085} & \num{0.794897437630666}$\pm$\num{0.740918949177126} & \num{0.794897437630666}$\pm$\num{0.740918949177126} & \num{0.79024433520866}$\pm$\num{0.773968125856086} & -0.6\% \\
			\bottomrule
		\end{tabular}%
	}
	\label{Table:twins}%
\end{table*}%
%\end{landscape}

\subsection{Numerical Studies on Benchmark Datasets}\label{sec:Numerical Studies on Benchmark Datasets}

\textbf{Models.} Similar to the simulation experiments, we choose Lasso, RF, and MLP as the regressors and choose LR, RF, and MLP as the classifiers. Additionally, two prevalent neural network models, TARNet and Dragonnet, are also considered for learning the nuisance parameters. According to \cite{shi2019adapting}, these two neural network structures can incorporate the estimations of both {$g^{i}$} and {$\pi^{i}$} using the representation learning technique.

\textbf{Settings.} We implement the above methods on two widely adopted benchmark datasets for causal inference, i.e., \textbf{IHDP} and \textbf{Twins}, and then compare RCL estimators with DR, IPW, DML, and their variants AIPW and DML-trim estimators. Mathematically, both the AIPW estimator and the DML-trim estimator are the same as the DML estimator. However, empirically, AIPW and DML-trim are less prone to suffer from the extreme values of inverse propensity scores. To be precise, AIPW decomposes the estimator into two parts that both contain the IPW term (see \cite{linden2016estimating}), while DML-trim trims estimated propensity scores at the cutoff points of $0.01$ and $0.99$ (see \cite{chernozhukov2018double}).

We take RCL{$_{2,1}$} and RCL{$_{2,2}$} as the representatives of the general RCL estimators because for the real datasets with a relatively small sample size and a large dimension of features, the second moment estimation of {$\nu^i$} is more reliable compared to higher moment estimations. %改到这
We use grid search to adjust the hyperparameters on the validation set for those machine learning models. For TARNet and Dragonnet, we use the same network structures (layers, units, regularization, batch size, learning rate, and stopping criterion) as suggested in \cite{shalit2017estimating} and \cite{shi2019adapting}.

\textbf{IHDP.} It is a widely used benchmark dataset for causal inference introduced by \cite{hill2011bayesian}. IHDP dataset is constructed based on the randomized controlled experiment conducted by Infant Health and Development Program. The collected 25-dimensional confounders from the 747 samples are associated with the properties of infants and their mothers, such as birth weight and mother's age. Our aim is to study the treatment effect of the specialist visits (binary treatment) on the cognitive scores (continuous-valued outcome). By removing a subset of the treated group, the selection bias in the IHDP dataset occurs. 
There are 1000 IHDP datasets given in \cite{hill2011bayesian}. Each dataset is split by the ratio of {$63\%/27\%/10\%$} as training/validation/test sets, which keeps consistent with \cite{shalit2017estimating}.

\textbf{Twins.} Twins dataset is introduced by \cite{louizos2017causal} and it collects twin births in USA between 1989 and 1991. The treatment {$D=1$} indicates the heavier twin while {$D=0$} indicates the lighter one; the outcome {$Y$} is a binary variable defined as the mortality in the first year; the covariates {$\mathbf{Z}$} include 30 features relevant to the parents, the pregnancy, and the birth. Similar to \cite{yoon2018ganite}, we only select twins that have the same gender and both weights less than 2kg. Finally, we have {$11440$} pairs of twins whose mortality rates are {$17.7\%$} for lighter twin and {$16.1\%$} for heavier twin. To simulate an observational dataset with selection bias, we selectively choose one from the twins as the observed sample based on the covariates of the $m^{\text{th}}$ individual: {$D_m|\mathbf{Z}_m \sim$} Bernoulli(Sigmoid({$\mathbf{w}^T\mathbf{Z}_m+b$})), where {$\mathbf{w}^T \sim \mathcal{U}_{30}((-0.01,0.01)^{30})$} and {$b \sim \mathcal{N}(0,0.01)$}. We repeat this process {$100$} times, and each of the generated {$100$} datasets is split by the ratio of {$64\%/16\%/20\%$} as training/validation/test sets, which keeps consistent with \cite{yoon2018ganite}.

\textbf{Analysis.} In Table \ref{Table:ihdp} and Table \ref{Table:twins}, we report the performance of every model combination, measured by {$\epsilon_{ATE} \; (\pm \sigma_{ATE})$}, for IHDP and Twins experiments respectively. The smaller {$\epsilon_{ATE}$}, the better. The metric {$R_{DR}=\text{RCL}_{2,2}/\text{DR}-1$}  is used to evaluate the reduction ratio in {$\epsilon_{ATE}$} of RCL{$_{2,2}$} relative to DR, and {$R_{DML}=\text{RCL}_{2,1}/\min(\text{IPW, AIPW, DML, DML-trim})-1$} is to evaluate the reduction ratio in RCL{$_{2,1}$} relative to the best estimator among IPW, AIPW, DML, and DML-trim. The negative {$R_{DR}$} ({$R_{DML}$}) indicates that the RCL estimator has a smaller {$\epsilon_{ATE}$} than the DR (IPW, AIPW, DML, and DML-trim) estimator. The smaller {$R_{DR}$} ({$R_{DML}$}), the better for our estimators.

Table \ref{Table:ihdp} reports the experimental results on IHDP datasets. It illustrates that although the DR estimator produces reasonable estimates, the RCL{$_{2,2}$} estimator has a smaller {$\epsilon_{ATE}$} than the DR estimator, with the error reduced relatively by {$0.9\%-4.8\%$}. Simultaneously, RCL{$_{2,2}$} achieves the best performance among all the estimators across all the model combinations. We also notice that even though AIPW and DML-trim avoid extreme values encountered by DML (e.g., when the classifier is chosen as RF, the inverse propensity score is estimated with an infinity value for some data points), the RCL{$_{2,1}$} estimator is still at most $67\%$ better than the best of IPW, AIPW, DML, and DML-trim estimators. More importantly, when the variance of inverse propensity scores is large (e.g., when the classifier is MLP), the improvement of RCL{$_{2,1}$} relative to DML becomes more substantial.

Table \ref{Table:twins} presents the experimental results on Twins datasets. It can be observed that the RCL{$_{2,2}$} estimator has a significantly smaller {$\epsilon_{ATE}$} compared with other estimators for all model combinations, and it can reduce the estimation error relatively by {$2.2\%-6.3\%$} compared with the DR method. Besides, the RCL{$_{2,1}$} estimator can reduce the estimation error by {$0.3\%-10\%$} relative to the best of IPW, AIPW, DML, and DML-trim estimators. It is also noticeable that when {$\pi^{i}$} is well specified (e.g., the case on using Dragonnet in Table \ref{Table:twins}), our RCL{$_{2,1}$} estimator still outperforms the DML estimator even though the error {$\epsilon_{ATE}$} produced by the DML estimator is small enough.

\begin{table}[t]
       %\renewcommand\arraystretch{0.8}
       %\vspace{6em}
	\centering
	\caption{The comparisons of mean and standard deviation between the generated (fake) data and the original (real) data.}
	\begin{tabular}{ccccc}
		\toprule
		\multirow{3}[6]{*}{Variables} & \multicolumn{4}{c}{\textbf{Mean}} \\
		\cmidrule{2-5}          & \multicolumn{2}{c}{Controlled} & \multicolumn{2}{c}{Treated} \\
		\cmidrule{2-5}          & Fake  & Real  & Fake  & Real \\
		\midrule
		Age   & 31.02 & 30.99 & 30.54 & 30.00 \\
		Using days  & 950.10 & 921.28 & 874.61 & 875.62 \\
		Status & 0.05  & 0.05  & 0.05  & 0.04 \\
		Credit line $(\times 100)$ & 93.05 & 91.24 & 94.34 & 93.56 \\
		Order amount $(\times 100)$ & 11.53 & 10.64 & 9.58  & 10.77 \\
		Borrow amount $(\times 100)$ & 8.58  & 9.01  & 13.26 & 13.69 \\
		Repayment amount $(\times 100)$ & 6.83  & 6.95  & 12.69 & 11.37 \\
		Numer of orders  & 3.82  & 3.87  & 4.09  & 4.55 \\
		Numer of borrowings & 1.96  & 1.99  & 2.94  & 3.00 \\
		Numer of repayments & 0.97  & 0.94  & 1.13  & 1.11 \\
		Handling fee & 15.87 & 16.47 & 14.09 & 16.29 \\
		Loan term & 2.61  & 2.58  & 2.22  & 2.20 \\
		Monthly consumption & 4.19  & 3.63  & 4.78  & 3.98 \\
		\midrule
		\multirow{3}[6]{*}{Variables} & \multicolumn{4}{c}{\textbf{Standard Deviation}} \\
		\cmidrule{2-5}          & \multicolumn{2}{c}{Controlled} & \multicolumn{2}{c}{Treated} \\
		\cmidrule{2-5}          & Fake  & Real  & Fake  & Real \\
		\midrule
		Age   & 7.09  & 6.68  & 7.41  & 6.92 \\
		Using days  & 324.99 & 310.12 & 291.13 & 292.99 \\
		Status & 0.23  & 0.22  & 0.21  & 0.20 \\
		Credit line $(\times 100)$ & 41.13 & 39.10 & 36.78 & 38.52 \\
		Order amount $(\times 100)$ & 15.97 & 18.48 & 12.85 & 15.86 \\
		Borrow amount $(\times 100)$ & 15.86 & 15.51 & 18.83 & 21.88 \\
		Repayment amount $(\times 100)$ & 13.93 & 12.44 & 18.49 & 19.20 \\
		Numer of orders  & 3.51  & 3.37  & 2.85  & 4.46 \\
		Numer of borrowings & 2.38  & 2.22  & 3.35  & 4.33 \\
		Numer of repayments & 0.51  & 0.58  & 0.53  & 0.68 \\
		Handling fee & 63.26 & 67.60 & 51.66 & 59.90 \\
		Loan term & 3.76  & 3.61  & 2.82  & 2.98 \\
		Monthly consumption & 2.68  & 2.65  & 2.9   & 2.73 \\
		\bottomrule
	\end{tabular}%
	\label{tab:comparison_credit}%
\end{table}%

%\begin{landscape}
\begin{table*}[t]
	%\vspace{4em}
	\centering
	\caption{The performance comparisons ({$\epsilon_{ATE} \pm \sigma_{ATE}$}) on the test sets of 100 \textbf{Credit} experiments. Smaller {$\epsilon_{ATE}$} is better.}
	\resizebox{2\columnwidth}{!}{
	\begin{tabular}{cccc|cccccc}
		\toprule
		Model/Estimator & DR    & RCL$_{2,2}$ & $R_{DR}$ & IPW   & AIPW  & DML   & DML-trim & RCL$_{2,1}$ & $R_{DML}$ \\
		\midrule
		LASSO+LR & \num{0.690049682648229}$\pm$\num{1.74286942570469} & \num{0.680621014767257}$\pm$\num{1.71390836906064} & -1.4\% & \num{2.40049493034954}$\pm$\num{2.43906039404566} & \num{1.06582155320181}$\pm$\num{1.35433356217323} & \num{1.06582155320181}$\pm$\num{1.35433356217323} & \num{0.981344691098694}$\pm$\num{1.11035516976907} & \num{0.644504319600636}$\pm$\num{0.742954972954954} & -34.3\% \\
		\midrule
		LASSO+RF & \num{0.690049682648229}$\pm$\num{1.74286942570469} & \num{0.677515449676811}$\pm$\num{1.70959815758643} & -1.8\% & $\infty$ & \num{1.10207031590086}$\pm$\num{1.53817075764853} & $\infty$ & \num{1.09388472155565}$\pm$\num{1.51711131131176} & \num{0.656311328830662}$\pm$\num{0.771858804552922} & -40.0\% \\
		\midrule
		LASSO+MLP & \num{0.690049682648229}$\pm$\num{1.74286942570469} & \num{0.681864438093791}$\pm$\num{1.71273737358402} & -1.2\% & \num{2.95960275392154}$\pm$\num{3.33831524399723} & \num{0.758444112456858}$\pm$\num{1.07656628542571} & \num{0.758444112456858}$\pm$\num{1.07656628542571} & \num{0.758442275558753}$\pm$\num{1.07656116355827} & \num{0.627041651277892}$\pm$\num{0.74385615774205} & -17\% \\
		\midrule
		RF+LR & \num{0.696297961691107}$\pm$\num{1.57368830898018} & \num{0.687614131424038}$\pm$\num{1.55188911864512} & -1.2\% & \num{2.40049493034954}$\pm$\num{2.43906039404566} & \num{0.924612372641937}$\pm$\num{1.02095913566655} & \num{0.924612372641936}$\pm$\num{1.02095913566655} & \num{0.867388079898693}$\pm$\num{0.912322979977973} & \num{0.61057535360992}$\pm$\num{0.564389362000155} & -29.6\% \\
		\midrule
		RF+RF & \num{0.696297961691107}$\pm$\num{1.57368830898018} & \num{0.684494509499265}$\pm$\num{1.55123229409802} & -1.7\% & $\infty$ & \num{0.975900144259669}$\pm$\num{1.02256678006144} & $\infty$ & \num{0.971531420886165}$\pm$\num{1.00187788492556} & \num{0.618625253765665}$\pm$\num{0.541004187985663} & -36.3\% \\
		\midrule
		RF+MLP & \num{0.696297961691107}$\pm$\num{1.57368830898018} & \num{0.688845154528303}$\pm$\num{1.55268489033691} & -1.1\% & \num{2.95960275392154}$\pm$\num{3.33831524399723} & \num{0.719553453941674}$\pm$\num{0.902552630701806} & \num{0.719553453941674}$\pm$\num{0.902552630701805} & \num{0.719551776939198}$\pm$\num{0.902550823715124} & \num{0.588533929086706}$\pm$\num{0.540938097571692} & -18\% \\
		\midrule
		MLP+LR & \num{1.08982756220863}$\pm$\num{1.52194526688617} & \num{1.07738875264661}$\pm$\num{1.50712340006535} & -1.1\% & \num{2.40049493034954}$\pm$\num{2.43906039404566} & \num{1.02478540400446}$\pm$\num{1.27741600501819} & \num{1.02478540400446}$\pm$\num{1.27741600501819} & \num{0.954183808635399}$\pm$\num{1.00662340985542} & \num{0.701366944808164}$\pm$\num{0.685348188169263} & -26\% \\
		\midrule
		MLP+RF & \num{1.08982756220863}$\pm$\num{1.52194526688617} & \num{1.07445253896869}$\pm$\num{1.50182501980458} & -1.4\% & $\infty$ & \num{1.04344713492224}$\pm$\num{1.11876673246206} & $\infty$ & \num{1.03148248845676}$\pm$\num{1.11141079234105} & \num{0.715433125507879}$\pm$\num{0.649140861372895} & -31\% \\
		\midrule
		MLP+MLP & \num{1.08982756220863}$\pm$\num{1.52194526688617} & \num{1.07686100378118}$\pm$\num{1.50557050426281} & -1.2\% & \num{2.95960275392154}$\pm$\num{3.33831524399723} & \num{0.725530731034892}$\pm$\num{0.911923302583978} & \num{0.725530635877883}$\pm$\num{0.9119222830062} & \num{0.725531463221722}$\pm$\num{0.911922898190708} & \num{0.660601946230215}$\pm$\num{0.620231002844582} & -9\% \\
		\midrule
		TARNet & \num{0.726858779195232}$\pm$\num{0.952506228476958} & \num{0.716647907722652}$\pm$\num{0.938860651315584} & -1.4\% & $\infty$ & $\infty$ & $\infty$ & \num{13.3481589508833}$\pm$\num{17.3186060153251} & \num{0.73443354388927}$\pm$\num{0.855793299810166} & -94.5\% \\
		\midrule
		Dragonnet & \num{0.740349980069469}$\pm$\num{1.04028015334969} & \num{0.728776295887394}$\pm$\num{1.01955782175376} & -1.6\% & \num{53.017501038504}$\pm$\num{389.318542500602} & \num{76.5859063068816}$\pm$\num{512.821144459242} & \num{76.5859134813654}$\pm$\num{512.821195315032} & \num{1.89449626198146}$\pm$\num{3.41067161924036} & \num{0.798366903867514}$\pm$\num{0.748310321240025} & -58\% \\
		\bottomrule
	\end{tabular}%
}
	\label{tab:credit_result}%
\end{table*}%
%\end{landscape}

\subsection{Numerical Studies on Credit dataset}

Causal inference methods are popularly assessed with simulation datasets which are typically generated by a parametric data generating process. Though the ground truth of treatment effects is accessible in this way, such synthetic datasets fail to resemble the original real datasets. Some benchmark real datasets are used for evaluating methods by hiding some part of the full dataset with a selecting formula for some specific purpose, like IHDP and Twins. The resulting  semi-synthetic datasets are used in many existing works. But such original datasets are difficult and expensive to generate. Recently, \cite{athey2021using}  suggest using Wasserstein Generative Adversarial Network (WGAN) to generate datasets to evaluate various treatment effect estimators. Such artificial data can closely mimic the original real-world data, and the non-parametric generating way provides a fairer evaluation to alleviate concerns of researchers choosing some specific parametric generating processes to support the proposed methods.

\textbf{Credit.} Following the data generating strategy in \cite{athey2021using}, we generate the Credit dataset based on a consumer loan dataset collected from an online consumer credit lending platform. 
In the Credit dataset, the treatment is whether to increase the credit limit for an individual in June or not, and the outcome is the monthly consumption in June. The covariates of the Credit dataset are age, using days, credit status in the last month; and the average credit line, order amount, borrowing amount, repayment amount, number of orders, number of borrowings, number of repayments, handling fee, and loan term in the last three months. We randomly shuffle the samples to obtain $100$ different Credit datasets with $2000$ samples generated by WGAN, and then split by the ratio of $64\%/16\%/20\%$ as training/validation/test sets. The comparisons between the generated data and the original data are shown in Table \ref{tab:comparison_credit}.

\textbf{Result analysis.} We keep the same setting as in the IHDP experiments and report results on test sets in Table \ref{tab:credit_result}. It is obvious that RCL$_{2,2}$ and RCL$_{2,1}$ are overwhelmingly superior compared with DR and variants of DML, respectively. The improvements are especially pronounced for DML variants, with the estimation error reduced by $94.5\%$ at most. Moreover, our estimators protect the estimates from high variance or infinity values caused by the error-compounding issue. Table \ref{tab:credit_result} reveals that our estimators can significantly improve the ATE estimations compared to traditional estimators on datasets generated in a nonparametric way suggested by \cite{athey2021using}.

In summary, DML is recognized as a better method than DR and IPW because when the DR (IPW) estimator has a notable bias due to the misspecification on {$g^{i}$} (or {$\pi^{i}$}), the DML estimator can reduce the bias if {$\pi^{i}$} (or {$g^{i}$}) is well estimated. However, the advantages of DML are not easy to achieve in practice. First, the DML estimator, which incorporates the inverse propensity score term, may give a very large estimation or even infinite value of ATE, reflecting that the DML estimator is volatile to the estimation error of propensity scores. Second, if {$g^{i}$} is approximated well enough, DML will not assuredly perform better than DR due to the high variance of the IPW term. By contrast, our RCL estimators are more practical since i) they can stabilize the error caused by the extremes of propensity scores; ii) if {$g^{i}$} is well approximated, the RCL{$_{2,2}$} estimator will outperform the DR estimator owing to the RCL scores are orthogonal scores; iii) if {$g^{i}$} is not well approximated, but {$\pi^{i}$} is correctly specified, the RCL estimator with {$k=1$} performs better than the DML estimator with smaller estimation error and slighter volatility to the estimation error of propensity scores.

\section{Conclusions}
This paper constructs the RCL scores and establishes the RCL estimators for the ATE estimation. Theoretically, we prove that the RCL scores are orthogonal scores and the RCL estimators are consistent. Numerically, the comprehensive experiments have shown that our estimators outperform the commonly used methods such as DR, IPW, AIPW, DML, and DML-trim estimators. In addition, the proposed RCL estimators have the same merit, i.e., the doubly robust property, as the DML estimator. However, unlike the DML estimator, the RCL estimators are more stable to the extremes of propensity scores than the DML estimator and its variants. 

To be more specific, RCL{$_{2,2}$} always outperforms DR and RCL{$_{2,1}$} always works better than the IPW-based and DML-based methods. The comparison result between RCL{$_{2,2}$} and RCL{$_{2,1}$} varies across datasets and different combinations of machine learning regressor and classifier. Notice that there is also no consistent superiority between DR and DML (it may depend on the severity of the confounding effect or the difference between the distributions of the treated group and controlled group). We suggest choosing RCL{$_{2,2}$} or RCL{$_{2,1}$} based on the performance difference between DR and DML. For RCL{$_{r,k}$} with {$r>2$}, in the current stage it is not superior than the case with {$r=2$}.
In the future research, we will investigate the optimal values of {$(r,k)$} for the RCL estimator.

Future research also includes providing interpretability for deep learning models in causal inference context using the RCL method, since deep learning has become the most attractive methodology and technique in the big data era and the interpretability is critical for researchers. In addition, the applications of RCL method in  decision makings are worthwhile future works, because better causality models can answer well more {\textit{what if}} problems and thus are inherently linked and helpful to decision makings.

%% file: appendix1.tex
\section{Proof of Theorem \ref{thm:$k^{th}$-order orthogonal condition score function}}
\label{appendix:A}

Given the nuisance parameters $\varrho=(\mathcal{g}^{i},a_{i})$ and the true nuisance parameters $\rho=(g^{i},\pi^{i})$, we find out the RCL score $\psi^{i}(W,\vartheta,\varrho)$ w.r.t. the nuisance parameters $\varrho=(\mathcal{g}^{i},a_{i})$ which can be used to construct the estimator of the causal parameter $\theta^{i}:=\mathbb{E}\left[g^{i}(\mathbf{Z})\right]$. We try an ansatz of $\psi^{i}(W,\vartheta,\varrho)$ such that
\begin{equation}
{
\begin{aligned}\label{eqt:higher-order orthogonal condition score function-supp}
\psi^{i}(W,\vartheta,\varrho)&=\vartheta-\mathcal{g}^{i}(\mathbf{Z})-(Y^{i}-\mathcal{g}^{i}(\mathbf{Z}))A(D,\mathbf{Z};a_{i}),
\end{aligned}
}
\end{equation}
where 
\begin{equation}
{
\begin{aligned}\label{eqt:higher-order orthogonal condition score function 2-supp}
A(D,\mathbf{Z};a_{i})&=\bar{b}_{r}\left[\mathbf{1}_{\{D=d^{i}\}}-a_{i}(\mathbf{Z})\right]^{r}\\
&+\underset{q=1}{\overset{k-1}{\sum}}b_{q}\big(\big[\mathbf{1}_{\{D=d^{i}\}}-a_{i}(\mathbf{Z})\big]^{q}-\mathbb{E}\big[(\nu^{i})^{q}\mid\mathbf{Z}\big]\big).
\end{aligned}
}
\end{equation}
Here, the coefficients $b_{1},\dots,b_{k-1},\bar{b}_{r}$ depend on $\mathbf{Z}$ and the moments of $\nu^{i}$ only. Using the ansatz, we notice that $\psi^{i}(W,\vartheta,\varrho)$ satisfies the moment condition, i.e., $\mathbb{E}\left[\psi^{i}(W,\vartheta,\varrho)\mid_{\vartheta=\theta^{i},\;\varrho=\rho}\right]=0$. Indeed, we have
\begin{equation*}
{
\begin{aligned}
&\mathbb{E}\left[\psi^{i}(W,\vartheta,\varrho)\mid_{\vartheta=\theta^{i},\;\varrho=\rho}\right]\\
=&\mathbb{E}\left[\theta^{i}-g^{i}(\mathbf{Z})-(Y^{i}-g^{i}(\mathbf{Z}))A(D,\mathbf{Z};\pi^{i})\right]\\
=&\mathbb{E}\left[\theta^{i}-g^{i}(\mathbf{Z})\right]-\mathbb{E}\left[(Y^{i}-g^{i}(\mathbf{Z}))A(D,\mathbf{Z};\pi^{i})\right]\\
=&-\mathbb{E}\left[\xi^{i}\times A(D,\mathbf{Z};\pi^{i})\right]\\
=&-\mathbb{E}\left[\mathbb{E}\left[\xi^{i}\times A(D,\mathbf{Z};\pi^{i})\mid D, \mathbf{Z}\right]\right]\\
=&-\mathbb{E}\left[A(D,\mathbf{Z};\pi^{i})\mathbb{E}\left[\xi^{i}\mid D,\mathbf{Z}\right]\right]=0.\\
%=&-\mathbb{E}\left[A(D,\mathbf{Z};\pi^{i})\mathbb{E}\left[\xi^{i}\mid \mathbf{Z}\right]\right]=0.
\end{aligned}
}
\end{equation*}
The second last equality comes from the fact that $A(D,\mathbf{Z};\pi^{i})$ is a function of $(D,\mathbf{Z})$. The last equality comes from the fact that $(\xi^{i} \perp \!\!\! \perp D) \mid \mathbf{Z}$. Now, we aim to find out the coefficients $b_{1},\dots,b_{k-1},\bar{b}_{r}$ such that the score \eqref{eqt:higher-order orthogonal condition score function-supp} satisfies the orthogonal score condition in Definition \ref{def:orthogonal condition} for $\alpha\in S_k$ (abbreviated as a $k^{\mathrm{th}}$ score). Indeed, we need to have $\mathbb{E}\left[\partial_{\mathcal{g}^{i}}^{\alpha_{1}}\partial_{a_{i}}^{\alpha_{2}}\psi^{i}(W,\vartheta,\varrho)\mid_{\vartheta=\theta^{i},\;\varrho=\rho}\mid\mathbf{Z}\right]=0$ for all $\alpha_{1}$ and $\alpha_{2}$ which are non-negative integers such that $1\leq \alpha_{1}+\alpha_{2}\leq k$. Since $\partial_{\mathcal{g}^{i}}^{\alpha_{1}}\partial_{a_{i}}^{\alpha_{2}}\psi^{i}(W,\vartheta,\varrho)=0$ when $\alpha_{1}\geq 2$, we only need to solve the coefficients $b_{1},\dots,b_{k-1},\bar{b}_{r}$ from
\begin{subequations}
{
\begin{empheq}[left=\empheqlbrace]{align}
0&=\mathbb{E}\left[\partial_{a_{i}}^{k}\psi^{i}(W,\vartheta,\varrho)\mid_{\vartheta=\theta^{i},\;\varrho=\rho}\mid\mathbf{Z}\right],\label{eqt:score condition 1}\\
0&=\mathbb{E}\left[\partial_{\mathcal{g}^{i}}^{1}\partial_{a_{i}}^{q}\psi^{i}(W,\vartheta,\varrho)\mid_{\vartheta=\theta^{i},\;\varrho=\rho}\mid\mathbf{Z}\right],\label{eqt:score condition 2}
\end{empheq}
}\noindent
\end{subequations}
$\forall q=0,\dots,k-1$. However, \eqref{eqt:score condition 1} always holds since
\begin{equation*}
{
\begin{aligned}
&\mathbb{E}\left[\partial_{a_{i}}^{k}\psi^{i}(W,\vartheta,\varrho)\mid_{\vartheta=\theta^{i},\;\varrho=\rho}\mid\mathbf{Z}\right]\\
=&\mathbb{E}\left[(Y^{i}-g^{i}(\mathbf{Z}))\times\partial_{a_{i}}^{k}A(D,\mathbf{Z};a_{i})\mid_{a_{i}=\pi^{i}}\mid\mathbf{Z}\right]\\
=&\mathbb{E}\left[\mathbb{E}\left[(Y^{i}-g^{i}(\mathbf{Z}))\times\partial_{a_{i}}^{k}A(D,\mathbf{Z};a_{i})\mid_{a_{i}=\pi^{i}}\mid D,\mathbf{Z}\right] \mid \mathbf{Z}\right]\\										=&\mathbb{E}\left[\partial_{a_{i}}^{k}A(D,\mathbf{Z};a_{i})\mid_{a_{i}=\pi^{i}}\mathbb{E}\left[(Y^{i}-g^{i}(\mathbf{Z}))\mid D,\mathbf{Z}\right] \mid \mathbf{Z}\right]\\
=&\mathbb{E}\left[\partial_{a_{i}}^{k}A(D,\mathbf{Z};a_{i})\mid_{a_{i}=\pi^{i}}\mathbb{E}\left[\xi^{i}\mid D, \mathbf{Z}\right] \mid \mathbf{Z}\right]=0.
\end{aligned}
}
\end{equation*}
Consequently, we need to find out the coefficients $b_{1}$, $b_{2}$, $\dots$, $b_{k-1}$, $\bar{b}_{r}$ from
{
\begin{align}\tag{\ref{eqt:score condition 2}}
\mathbb{E}\left[\partial_{\mathcal{g}^{i}}^{1}\partial_{a_{i}}^{q}\psi^{i}(W,\vartheta,\varrho)\mid_{\vartheta=\theta^{i},\;\varrho=\rho}\mid\mathbf{Z}\right]&=0,
\end{align}
}\noindent
$\forall q=0,\dots,k-1$. From \eqref{eqt:score condition 2}, there are $k$ equations and we need to solve the $k$ unknowns $b_{1},\dots,b_{k-1},\bar{b}_{r}$ from the $k$ equations. Generally, the $k$ unknowns could be solved uniquely.

To start with, we compute $\partial_{\mathcal{g}^{i}}^{1}\partial_{a_{i}}^{q}\psi^{i}(W,\vartheta,\varrho)$ for $q=0,\dots,k-1$. Note that
\begin{equation*}
{
\begin{aligned}
\partial_{\mathcal{g}^{i}}^{1}\partial_{a_{i}}^{q}\psi^{i}(W,\vartheta,\varrho)&=-1+A(D,\mathbf{Z};a_{i})
\end{aligned}
}
\end{equation*}
when $q=0$ and
\begin{equation*}
{
\begin{aligned}
\partial_{\mathcal{g}^{i}}^{1}\partial_{a_{i}}^{q}\psi^{i}(W,\vartheta,\varrho)&=\bar{b}_{r}\frac{r!(-1)^{q}[\mathbf{1}_{\{D=d^{i}\}}-a_{i}(\mathbf{Z})]^{r-q}}{(r-q)!}\\
&\quad+\overset{k-1}{\underset{u=q}{\sum}}b_{u}\frac{u!(-1)^{q}[\mathbf{1}_{\{D=d^{i}\}}-a_{i}(\mathbf{Z})]^{u-q}}{(u-q)!}
\end{aligned}
}
\end{equation*}
%\begin{equation*}
%{
%\begin{aligned}
%\partial_{\mathcal{g}^{i}}^{1}\partial_{a_{i}}^{q}\psi^{i}(W,\vartheta,\varrho)&=\begin{cases}
%\begin{aligned}
%&-1+A(D,\mathbf{Z};a_{i}),
%& \; & q=0\\
%&\bar{b}_{r}\frac{r!(-1)^{q}[\mathbf{1}_{\{D=d^{i}\}}-a_{i}(\mathbf{Z})]^{r-q}}{(r-q)!}\\
%&+\overset{k-1}{\underset{u=q}{\sum}}b_{u}\frac{u!(-1)^{q}[\mathbf{1}_{\{D=d^{i}\}}-a_{i}(\mathbf{Z})]^{u-q}}{(u-q)!},
%& \; 1\leq & q\leq k-1.
%\end{aligned}
%\end{cases}
%\end{aligned}
%}
%\end{equation*}
when $1\leq q\leq k-1$. Consequently, we need to solve for $b_{1},\dots,b_{k-1}$ and $\bar{b}_{r}$ simultaneously from
\begin{subequations}
\begin{equation}
{
\begin{aligned}\label{eqt:simultaneous equation 1}
1&=\mathbb{E}\left[A(D,\mathbf{Z};\pi^{i})\mid\mathbf{Z}\right]
\end{aligned}
}
\end{equation}
and 
\begin{equation}
{
\begin{aligned}\label{eqt:simultaneous equation 2}
0&=\mathbb{E}\left[\bar{b}_{r}\frac{r!(-1)^{q}[\mathbf{1}_{\{D=d^{i}\}}-\pi^{i}(\mathbf{Z})]^{r-q}}{(r-q)!}\mid\mathbf{Z}\right]\\
&\quad+\mathbb{E}\left[\overset{k-1}{\underset{u=q}{\sum}}b_{u}\frac{u!(-1)^{q}[\mathbf{1}_{\{D=d^{i}\}}-\pi^{i}(\mathbf{Z})]^{u-q}}{(u-q)!}\mid\mathbf{Z}\right].
\end{aligned}
}
\end{equation}
\end{subequations}
%\begin{subequations}
%{
%\begin{empheq}[left=\empheqlbrace]{align}
%1&=\mathbb{E}\left[A(D,\mathbf{Z};\pi^{i})\mid\mathbf{Z}\right],\label{eqt:simultaneous equation 1}\\
%0&=\mathbb{E}\left[\bar{b}_{r}\frac{r!(-1)^{q}[\mathbf{1}_{\{D=d^{i}\}}-\pi^{i}(\mathbf{Z})]^{r-q}}{(r-q)!}+\overset{k-1}{\underset{u=q}{\sum}}b_{u}\frac{u!(-1)^{q}[\mathbf{1}_{\{D=d^{i}\}}-\pi^{i}(\mathbf{Z})]^{u-q}}{(u-q)!}\mid\mathbf{Z}\right].\label{eqt:simultaneous equation 2}
%\end{empheq}
%}\noindent
%\end{subequations}
From \eqref{eqt:simultaneous equation 1}, we have
\begin{equation}\tag{\ref{eqt:simultaneous equation 1}*}\label{eqt:simultaneous equation 1-simplified}
{
\begin{aligned}
1&=\bar{b}_{r}\mathbb{E}\left[\left(\mathbf{1}_{\{D=d^{i}\}}-\pi^{i}(\mathbf{Z})\right)^{r}\mid\mathbf{Z}\right]\\
&\quad+\underset{q=1}{\overset{k-1}{\sum}}b_{q}\mathbb{E}\left[\left(\mathbf{1}_{\{D=d^{i}\}}-\pi^{i}(\mathbf{Z})\right)^{q}\mid\mathbf{Z}\right]\\
&\quad-\underset{q=1}{\overset{k-1}{\sum}}b_{q}\mathbb{E}\left[\mathbb{E}\left[(\nu^{i})^{q}\mid\mathbf{Z}\right]\mid\mathbf{Z}\right].
\end{aligned}
}
\end{equation}
Since 
\begin{equation*}
\begin{aligned}
\mathbb{E}\left[\left(\mathbf{1}_{\{D=d^{i}\}}-\pi^{i}(\mathbf{Z})\right)^{q}\mid\mathbf{Z}\right]&=\mathbb{E}\left[(\nu^{i})^{q}\mid\mathbf{Z}\right]\quad\text{and}\\
\mathbb{E}\left[\mathbb{E}\left[(\nu^{i})^{q}\mid\mathbf{Z}\right]\mid\mathbf{Z}\right]&=\mathbb{E}\left[(\nu^{i})^{q}\mid\mathbf{Z}\right],
\end{aligned}
\end{equation*}
we understand that 
\begin{equation*}
\begin{aligned}
\mathbb{E}\left[\left(\mathbf{1}_{\{D=d^{i}\}}-\pi^{i}(\mathbf{Z})\right)^{q}\mid\mathbf{Z}\right]-\mathbb{E}\left[\mathbb{E}\left[(\nu^{i})^{q}\mid\mathbf{Z}\right]\mid\mathbf{Z}\right]=0.
\end{aligned}
\end{equation*}
As such, \eqref{eqt:simultaneous equation 1-simplified} can be reduced as
%
%Since {$\mathbb{E}\left[\left(\mathbf{1}_{\{D=d^{i}\}}-\pi^{i}(\mathbf{Z})\right)^{q}\mid\mathbf{Z}\right]=\mathbb{E}\left[(\nu^{i})^{q}\mid\mathbf{Z}\right]$} and {$\mathbb{E}\left[\mathbb{E}\left[(\nu^{i})^{q}\mid\mathbf{Z}\right]\mid\mathbf{Z}\right]=\mathbb{E}\left[(\nu^{i})^{q}\mid\mathbf{Z}\right]$}, we understand that $\mathbb{E}\left[\left(\mathbf{1}_{\{D=d^{i}\}}-\pi^{i}(\mathbf{Z})\right)^{q}\mid\mathbf{Z}\right]-\mathbb{E}\left[\mathbb{E}\left[(\nu^{i})^{q}\mid\mathbf{Z}\right]\mid\mathbf{Z}\right]=0$. The above equality can therefore be reduced as
\begin{equation*}
{
\begin{aligned}
& \bar{b}_{r}\mathbb{E}\left[\left(\mathbf{1}_{\{D=d^{i}\}}-\pi^{i}(\mathbf{Z})\right)^{r}\mid\mathbf{Z}\right]=1\\
\Rightarrow\quad & \bar{b}_{r}\mathbb{E}\left[(\nu^{i})^{r}\mid\mathbf{Z}\right]=1.
\end{aligned}
}
\end{equation*}
%\begin{equation*}
%{
%\begin{aligned}
%\bar{b}_{r}\mathbb{E}\left[\left(\mathbf{1}_{\{D=d^{i}\}}-\pi^{i}(\mathbf{Z})\right)^{r}\mid\mathbf{Z}\right]+\underset{q=1}{\overset{k-1}{\sum}}b_{q}\mathbb{E}\left[\left(\mathbf{1}_{\{D=d^{i}\}}-\pi^{i}(\mathbf{Z})\right)^{q}\mid\mathbf{Z}\right]-\underset{q=1}{\overset{k-1}{\sum}}b_{q}\mathbb{E}\left[\mathbb{E}\left[(\nu^{i})^{q}\mid\mathbf{Z}\right]\mid\mathbf{Z}\right]&=1\\
%\Rightarrow \bar{b}_{r}\mathbb{E}\left[(\nu^{i})^{r}\mid\mathbf{Z}\right]+\underset{q=1}{\overset{k-1}{\sum}}b_{q}\mathbb{E}\left[\left((\nu^{i})^{q}-\mathbb{E}\left[(\nu^{i})^{q}\mid\mathbf{Z}\right]\right)\mid\mathbf{Z}\right]&=1\\
%\Rightarrow \bar{b}_{r}\mathbb{E}\left[(\nu^{i})^{r}\mid\mathbf{Z}\right]+\underset{q=1}{\overset{k-1}{\sum}}b_{q}\left(\mathbb{E}\left[(\nu^{i})^{q}\mid\mathbf{Z}\right]-\mathbb{E}\left[(\nu^{i})^{q}\mid\mathbf{Z}\right]\right)&=1\\
%\Rightarrow \bar{b}_{r}\mathbb{E}\left[(\nu^{i})^{r}\mid\mathbf{Z}\right]&=1.
%\end{aligned}
%}
%\end{equation*}
Hence, we can solve for $\bar{b}_{r}$ such that
\begin{equation*}
\begin{aligned}
\bar{b}_{r}=\frac{1}{\mathbb{E}\left[(\nu^{i})^{r}\mid\mathbf{Z}\right]}.
\end{aligned}
\end{equation*}

It remains to find out $b_{1},\dots,b_{k-1}$ from \eqref{eqt:simultaneous equation 2}. Indeed, we can simplify \eqref{eqt:simultaneous equation 2} as
\begin{equation}
{
\begin{aligned}\label{eqt:iteration equation of finding coeff of b}
\bar{b}_{r}\mathbb{E}\left[\frac{r!(\nu^{i})^{r-q}}{(r-q)!}\mid\mathbf{Z}\right]+\overset{k-1}{\underset{u=q}{\sum}}b_{u}\mathbb{E}\left[\frac{u!(\nu^{i})^{u-q}}{(u-q)!}\mid\mathbf{Z}\right] &=0\\
\Rightarrow \bar{b}_{r}\binom{r}{q}\mathbb{E}\left[(\nu^{i})^{r-q}\mid\mathbf{Z}\right]+\overset{k-1}{\underset{u=q}{\sum}}b_{u}\binom{u}{q}\mathbb{E}\left[(\nu^{i})^{u-q}\mid\mathbf{Z}\right] &=0,
\end{aligned}
}
\end{equation}
$\forall 1\leq q\leq k-1$. Now, we solve $b_{1},\dots,b_{k-1}$. We start with finding out $b_{k-1}$, followed by $b_{k-2},\;b_{k-3},\dots,b_{1}$ iteratively. When $q=k-1$, \eqref{eqt:iteration equation of finding coeff of b} becomes
\begin{equation*}
{
\begin{aligned}
 0=&~\bar{b}_{r}\binom{r}{k-1}\mathbb{E}\left[(\nu^{i})^{r-k+1}\mid\mathbf{Z}\right]\\
& +b_{k-1}\binom{k-1}{k-1}\mathbb{E}\left[(\nu^{i})^{0}\mid\mathbf{Z}\right]\\
\Rightarrow & \quad b_{k-1}=-\bar{b}_{r}\binom{r}{k-1}\mathbb{E}\left[(\nu^{i})^{r-k+1}\mid\mathbf{Z}\right].
\end{aligned}
}
\end{equation*}
Now, when $q=k-2$, \eqref{eqt:iteration equation of finding coeff of b} becomes
\begin{equation*}
{
\begin{aligned}
&0=\bar{b}_{r}{r \choose k-2}\mathbb{E}\left[(\nu^{i})^{r-k+2}\mid\mathbf{Z}\right]\\
&+b_{k-1}{k-1 \choose k-2}\mathbb{E}\left[(\nu^{i})^{(k-1)-(k-2)}\mid\mathbf{Z}\right]
+b_{k-2}\mathbb{E}\left[(\nu^{i})^{0}\mid\mathbf{Z}\right]\\
&\Rightarrow b_{k-2}=-b_{k-1}{k-1 \choose k-2}\mathbb{E}\left[(\nu^{i})^{1}\mid\mathbf{Z}\right]\\
&\qquad\qquad\quad\!-\bar{b}_{r}{r \choose k-2}\mathbb{E}\left[(\nu^{i})^{r-k+2}\mid\mathbf{Z}\right].
\end{aligned}
}
\end{equation*}
Now, suppose $b_{q+1},\dots,b_{k-1}$ are known and we want to find out what $b_{q}$ is. We have to solve it from
\begin{equation*}
{
\begin{aligned}
0=&~b_{q}\mathbb{E}\left[(\nu^{i})^{0}\mid\mathbf{Z}\right]+\bar{b}_{r}{r \choose q}\mathbb{E}\left[(\nu^{i})^{r-q}\mid\mathbf{Z}\right]\\
&+\overset{k-1}{\underset{u=q+1}{\sum}}b_{u}{u \choose q}\mathbb{E}\left[(\nu^{i})^{u-q}\mid\mathbf{Z}\right].
\end{aligned}
}
\end{equation*}
%\begin{equation*}
%{
%\begin{aligned}
%\bar{b}_{r}{r \choose q}\mathbb{E}\left[(\nu^{i})^{r-q}\mid\mathbf{Z}\right]+\overset{k-1}{\underset{u=q+1}{\sum}}b_{u}{u \choose q}\mathbb{E}\left[(\nu^{i})^{u-q}\mid\mathbf{Z}\right]+b_{q}\mathbb{E}\left[(\nu^{i})^{0}\mid\mathbf{Z}\right]&=0.
%\end{aligned}
%}
%\end{equation*}
We can obtain $b_{q}$ from the above equation, which gives
\begin{equation*}
{
\begin{aligned}
b_{q}=&-\overset{k-1}{\underset{u=q+1}{\sum}}b_{u}{u \choose q}\mathbb{E}\left[(\nu^{i})^{u-q}\mid\mathbf{Z}\right]\\
&-\bar{b}_{r}{r \choose q}\mathbb{E}\left[(\nu^{i})^{r-q}\mid\mathbf{Z}\right]\\
\Rightarrow b_{q}=&-\overset{k-1-q}{\underset{u=1}{\sum}}b_{q+u}{q+u \choose q}\mathbb{E}\left[(\nu^{i})^{u}\mid\mathbf{Z}\right]\\
&-\bar{b}_{r}{r \choose q}\mathbb{E}\left[(\nu^{i})^{r-q}\mid\mathbf{Z}\right].
\end{aligned}
}
\end{equation*}
The proof is completed.

%% file: appendix.tex
Now we present the theoretical proofs of some theorems, propositions, and the consistency results given in the paper. Appendix \ref{appendix:A} is in the main body of the paper.

%\begin{proof}{\textbf{Proof of Corollary \ref{corollary:$k^{th}$-order orthogonal condition estimator}}}\ 
%We have discussed the way to obtain the estimator in the main paper.
%\end{proof}

\section{Proof of Proposition \ref{lemma:difference from the statistical standpoint}}
\label{appendix:B}

$\forall \epsilon>0$, we consider $\mathbb{P}\left\{\left|\kappa_{N}^{i;CF}-\kappa_{N}^{i;F}\right|\geq\epsilon\right\}$. Indeed, we have
\begin{equation*}
{
\begin{aligned}
\mathbb{P}\left\{\left|\kappa_{N}^{i;CF}-\kappa_{N}^{i;F}\right|\geq\epsilon\right\}\leq & \frac{\mathbb{E}\left[\left(\kappa_{N}^{i;CF}-\kappa_{N}^{i;F}\right)^{2}\right]}{\epsilon^{2}}\\
=&\frac{\mathbb{E}\left[\left(\underset{m\in \mathscr{I}^{c}}{\sum}(\xi_{m}^{i;CF}-\xi_{m}^{i;F})A_{m}^{i}\right)^{2}\right]}{N^2 \epsilon^{2}}.\\
%\\
%=&\frac{\frac{1}{N^{2}}\underset{m\in \mathscr{I}^{c}}{\sum}\mathbb{E}\left[\left(\xi_{m}^{i;F}\right)^{2}\left(A(D_{m},\mathbf{Z}_{m};\pi^{i})\right)^{2}\right]}{\epsilon^{2}},
\end{aligned}
}
\end{equation*}
Denoting $\xi_{m}^{i;CF}-\xi_{m}^{i;F}$ as $\Xi_{m}^{i}$, we have
\begin{equation*}
{
\begin{aligned}
&\mathbb{E}\left[\big(\underset{m\in \mathscr{I}^{c}}{\sum}(\xi_{m}^{i;CF}-\xi_{m}^{i;F})A_{m}^{i}\big)^{2}\right]=\mathbb{E}\left[\big(\underset{m\in \mathscr{I}^{c}}{\sum}\Xi_{m}^{i}A_{m}^{i}\big)^{2}\right]\\
=&\mathbb{E}\left[\underset{m,\bar{m}\in \mathscr{I}^{c}}{\sum}\Xi_{m}^{i}A_{m}^{i}\Xi_{\bar{m}}^{i}A_{\bar{m}}^{i}\right]=\underset{m=\bar{m}\in \mathscr{I}^{c}}{\sum}\mathbb{E}\left[\Xi_{m}^{i}A_{m}^{i}\Xi_{\bar{m}}^{i}A_{\bar{m}}^{i}\right]\\
%=&\mathbb{E}\left[\underset{m,\bar{m}\in \mathscr{I}^{c}}{\sum}\Xi_{m}^{i}A_{m}^{i}\Xi_{\bar{m}}^{i}A_{\bar{m}}^{i}\right]=\underset{m,\bar{m}\in \mathscr{I}^{c}}{\sum}\mathbb{E}\left[\Xi_{m}^{i}A_{m}^{i}\Xi_{\bar{m}}^{i}A_{\bar{m}}^{i}\right]=\underset{m,\bar{m}\in \mathscr{I}^{c}}{\sum}\mathbb{E}\left[A_{m}^{i}A_{\bar{m}}^{i}\mathbb{E}\left[\Xi_{m}^{i}\Xi_{\bar{m}}^{i}\mid D, \mathbf{Z}\right]\right]\\
%=&\underset{m\in \mathscr{I}^{c}}{\sum}\mathbb{E}\left[(A_{m}^{i})^{2}\mathbb{E}\left[(\Xi_{m}^{i})^{2}\mid D,\mathbf{Z}\right]\right]\\
&\;+\underset{\substack{m,\bar{m}\in \mathscr{I}^{c} \\ m\neq \bar{m}}}{\sum}\mathbb{E}\left[A_{m}^{i}A_{\bar{m}}^{i}\mathbb{E}\left[\Xi_{m}^{i}\Xi_{\bar{m}}^{i}\mid D,\mathbf{Z}\right]\right]\\
=&\underset{m\in \mathscr{I}^{c}}{\sum}\mathbb{E}\left[(A_{m}^{i})^{2}\mathbb{E}\left[(\Xi_{m}^{i})^{2}\mid D,\mathbf{Z}\right]\right]\\
&\;+\underset{\substack{m,\bar{m}\in \mathscr{I}^{c} \\ m\neq \bar{m}}}{\sum}\mathbb{E}\left[A_{m}^{i}A_{\bar{m}}^{i}\mathbb{E}\left[\Xi_{m}^{i}\mid D,\mathbf{Z}\right]\mathbb{E}\left[\Xi_{\bar{m}}^{i}\mid D,\mathbf{Z}\right]\right]\\
=&\underset{m\in \mathscr{I}^{c}}{\sum}\mathbb{E}\left[(A_{m}^{i})^{2}\mathbb{E}\left[(\Xi_{m}^{i})^{2}\mid \mathbf{Z}\right]\right]=2\underset{m\in \mathscr{I}^{c}}{\sum}\mathbb{E}\left[(A_{m}^{i})^{2}\mathbb{E}\left[(\xi^{i;F}_{m})^{2}\mid \mathbf{Z}\right]\right]\\
\leq& 2N\mathbb{E}\left[(A^{i})^{2}\mathbb{E}\left[(\xi^{i;F})^{2}\mid \mathbf{Z}\right]\right].
\end{aligned}
}
\end{equation*}
The last equality follows from
\begin{equation*}
{
\begin{aligned}
&\mathbb{E}\left[(\Xi^{i}_{m})^{2}\mid\mathbf{Z}\right]=\mathbb{E}\left[(\xi^{i;CF}_{m}-\xi^{i;F}_{m})^{2}\mid\mathbf{Z}\right]\\
=&\mathbb{E}\left[(\xi^{i;CF}_{m})^{2}\mid\mathbf{Z}\right]-2\mathbb{E}\left[\xi^{i;F}_{m}\xi^{i;CF}_{m}\mid\mathbf{Z}\right]+\mathbb{E}\left[(\xi^{i;F}_{m})^{2}\mid\mathbf{Z}\right]\\
=&\mathbb{E}\left[(\xi^{i;CF}_{m})^{2}\mid\mathbf{Z}\right]-2\mathbb{E}\left[\xi^{i;F}_{m}\mid\mathbf{Z}\right]\mathbb{E}\left[\xi^{i;CF}_{m}\mid\mathbf{Z}\right]+\mathbb{E}\left[(\xi^{i;F}_{m})^{2}\mid\mathbf{Z}\right]\\
=&\mathbb{E}\left[(\xi^{i;CF}_{m})^{2}\mid\mathbf{Z}\right]+\mathbb{E}\left[(\xi^{i;F}_{m})^{2}\mid\mathbf{Z}\right]=2\mathbb{E}\left[(\xi^{i;F}_{m})^{2}\mid\mathbf{Z}\right].
\end{aligned}
}
\end{equation*}
%where $\sigma^{i}=\mathbb{E}\left[(\xi^{i;CF})^{2}\right]=\mathbb{E}\left[(\xi^{i;F})^{2}\right]$.
%The derivations are valid since $(\xi^{i;CF}\overset{d}{=}\xi^{i;F})\mid\mathbf{Z}$.
As a consequence, we have
\begin{equation*}
{
\begin{aligned}
\mathbb{P}\left\{\left|\kappa_{N}^{i;CF}-\kappa_{N}^{i;F}\right|\geq\epsilon\right\}&\leq \frac{2N\mathbb{E}\left[(A^{i})^{2}\mathbb{E}\left[(\xi^{i;F})^{2}\mid\mathbf{Z}\right]\right]}{N^{2}\epsilon^{2}}\\
&=\frac{2\mathbb{E}\left[(A^{i})^{2}\mathbb{E}\left[(\xi^{i;F})^{2}\mid\mathbf{Z}\right]\right]}{N\epsilon^{2}}\rightarrow 0
\end{aligned}
}
\end{equation*}
when $N\rightarrow \infty$. As a result, we have {$\kappa_{N}^{i;CF}-\kappa_{N}^{i;F}\overset{p}{\rightarrow} 0$}. The proof is completed.

\section{Proof of Proposition \ref{lemma:unbiasedness and consistency}}
\label{appendix:C}

First, we write
%Write {\color{blue}{$\frac{1}{R}\overset{R}{\underset{u=1}{\sum}}\xi_{m,u}^{i;F}=\mathscr{E}_{m}^{i}$}. Then we have}
\begin{equation*}
{
\begin{aligned}
\kappa_{R,N}^{i;F}&=\frac{1}{R}\overset{R}{\underset{u=1}{\sum}}\bigg[\frac{1}{N}\underset{m\in \mathscr{I}^{c}}{\overset{}{\sum}}\xi_{m,u}^{i;F} A_{m}^{i}\bigg]
=\frac{1}{N}\underset{m\in \mathscr{I}^{c}}{\overset{}{\sum}}\left(\frac{1}{R}\overset{R}{\underset{u=1}{\sum}}\xi_{m,u}^{i;F}\right) A_{m}^{i}=\frac{1}{N}\underset{m\in \mathscr{I}^{c}}{\overset{}{\sum}}\mathscr{E}_{m}^{i} A_{m}^{i}.
\end{aligned}
}
\end{equation*}
%\begin{equation*}
%{
%\begin{aligned}
%\kappa_{R,N}^{i;F}&=\frac{1}{R}\overset{R}{\underset{u=1}{\sum}} \left[\frac{1}{N}\underset{m\in \mathscr{I}^{c}}{\overset{}{\sum}}\xi_{m,u}^{i;F} A_{m}^{i}\right]\\
%&=\frac{1}{N}\underset{m\in \mathscr{I}^{c}}{\overset{}{\sum}}\left(\frac{1}{R}\overset{R}{\underset{u=1}{\sum}}\xi_{m,u}^{i;F}\right) A_{m}^{i}=\frac{1}{N}\underset{m\in \mathscr{I}^{c}}{\overset{}{\sum}}\mathscr{E}_{m}^{i} A_{m}^{i}
%\end{aligned}
%}
%\end{equation*}
$\forall\epsilon>0$, we have
\begin{equation*}
{
\begin{aligned}
\mathbb{P}\left\{\left|\kappa_{N}^{i;F}-\kappa_{R,N}^{i;F}\right|\geq\epsilon\right\}\leq  \frac{\mathbb{E}\left[\left(\frac{1}{N}\underset{m\in \mathscr{I}^{c}}{\overset{}{\sum}}\left[\mathscr{E}_{m}^{i}-\xi_{m}^{i;F}\right]A_{m}^{i}\right)^{2}\right]}{\epsilon^{2}}.
%=&\mathbb{P}\left\{\left|\frac{1}{N}\underset{m\in \mathscr{I}^{c}}{\overset{}{\sum}}\left[\mathscr{E}_{m}^{i}-\xi_{m}^{i;F}\right]A_{m}^{i}\right|\geq\epsilon\right\}\\
\end{aligned}
}
\end{equation*}
We simplify the numerator term. Note that
\begin{equation*}
%{
\resizebox{0.8\textwidth}{!}{$
\begin{aligned}
&\mathbb{E}\left[\left(\frac{1}{N}\underset{m\in \mathscr{I}^{c}}{\overset{}{\sum}}\left[\mathscr{E}_{m}^{i}-\xi_{m}^{i;F}\right]A_{m}^{i}\right)^{2}\right]
=\frac{1}{N^{2}}\:\:\sum_{\mathclap{\substack{m,\bar{m}\in \mathscr{I}^{c}}}}\mathbb{E}\left[\left(\mathscr{E}_{m}^{i}-\xi_{m}^{i;F}\right)\left(\mathscr{E}_{\bar{m}}^{i}-\xi_{\bar{m}}^{i;F}\right)A_{\bar{m}}^{i}A_{m}^{i}\right]\\
=&\frac{1}{N^{2}}\:\:\sum_{\mathclap{\substack{m\in \mathscr{I}^{c}}}}\mathbb{E}\left[\left(\mathscr{E}_{m}^{i}-\xi_{m}^{i;F}\right)^{2}(A_{m}^{i})^{2}\right]
+\frac{1}{N^{2}}\:\:\sum_{\mathclap{\substack{m,\bar{m}\in \mathscr{I}^{c}\\ m\neq\bar{m}}}}\mathbb{E}\left[\left(\mathscr{E}_{m}^{i}-\xi_{m}^{i;F}\right)\left(\mathscr{E}_{\bar{m}}^{i}-\xi_{\bar{m}}^{i;F}\right)A_{\bar{m}}^{i}A_{m}^{i}\right]\\
=&\frac{1}{N^{2}}\:\:\sum_{\mathclap{\substack{m\in \mathscr{I}^{c}}}}\mathbb{E}\left[(A_{m}^{i})^{2} \mathbb{E}\left[ \left(\mathscr{E}_{m}^{i}-\xi_{m}^{i;F}\right)^{2} \mid D,\mathbf{Z}\right] \right]
+\frac{1}{N^{2}}\:\:\sum_{\mathclap{\substack{m,\bar{m}\in \mathscr{I}^{c}\\ m\neq\bar{m}}}}\mathbb{E}\left[A_{\bar{m}}^{i}A_{m}^{i}\mathbb{E}\left[\left(\mathscr{E}_{m}^{i}-\xi_{m}^{i;F}\right)\left(\mathscr{E}_{\bar{m}}^{i}-\xi_{\bar{m}}^{i;F}\right)\mid D,\mathbf{Z}\right]\right]\\
=&\frac{1}{N^{2}}\:\:\sum_{\mathclap{\substack{m\in \mathscr{I}^{c}}}}\mathbb{E}\left[(A_{m}^{i})^{2} \mathbb{E}\left[ \left(\mathscr{E}_{m}^{i}-\xi_{m}^{i;F}\right)^{2} \mid\mathbf{Z}\right] \right]
+\frac{1}{N^{2}}\:\:\sum_{\mathclap{\substack{m,\bar{m}\in \mathscr{I}^{c}\\ m\neq\bar{m}}}}\mathbb{E}\left[A_{\bar{m}}^{i}A_{m}^{i}\mathbb{E}\left[\left(\mathscr{E}_{m}^{i}-\xi_{m}^{i;F}\right)\mid\mathbf{Z}\right]\mathbb{E}\left[\left(\mathscr{E}_{\bar{m}}^{i}-\xi_{\bar{m}}^{i;F}\right)\mid\mathbf{Z}\right]\right]\\
=&\frac{1}{N^{2}}\:\:\sum_{\mathclap{\substack{m\in \mathscr{I}^{c}}}}\mathbb{E}\left[(A_{m}^{i})^{2} \mathbb{E}\left[ \left(\mathscr{E}_{m}^{i}-\xi_{m}^{i;F}\right)^{2} \mid\mathbf{Z}\right] \right].
\end{aligned}$
}
%}
\end{equation*}
The last equality in the above derivation follows from the fact that, conditioning on {$\mathbf{Z}$}, {$\xi_{m,u}^{i;F}$} are i.i.d. of {$\xi_{m}^{i;F}$} for any {$u\in\{1,2,\cdots,R\}$}. Indeed, we have $\mathbb{E}\left[\xi_{m,u}^{i;F}\mid\mathbf{Z}\right]=\mathbb{E}\left[\xi_{m}^{i;F}\mid\mathbf{Z}\right]$ for any $m$ and $u\in\{1,2,\cdots,R\}$. Consequently, we have $\mathbb{E}\left[\left(\mathscr{E}_{m}^{i}-\xi_{m}^{i;F}\right)\mid\mathbf{Z}\right]=0$.
%\begin{equation*}
%%{
%\resizebox{0.49\textwidth}{!}{$\begin{aligned}
%&\mathbb{E}\left[\left(\mathscr{E}_{m}^{i}-\xi_{m}^{i;F}\right)\mid\mathbf{Z}\right]=\mathbb{E}\left[\left(\frac{1}{R}\overset{R}{\underset{u=1}{\sum}} \xi_{m,u}^{i;F}-\xi_{m}^{i;F}\right)\mid\mathbf{Z}\right]\\
%=&\frac{1}{R}\overset{R}{\underset{u=1}{\sum}} \mathbb{E}\left[\xi_{m,u}^{i;F}\mid\mathbf{Z}\right]-\mathbb{E}\left[\xi_{m}^{i;F}\mid\mathbf{Z}\right]=\frac{1}{R}\overset{R}{\underset{u=1}{\sum}} \mathbb{E}\left[\xi_{m}^{i;F}\mid\mathbf{Z}\right]-\mathbb{E}\left[\xi_{m}^{i;F}\mid\mathbf{Z}\right]\\
%=& \mathbb{E}\left[\xi_{m}^{i;F}\mid\mathbf{Z}\right]-\mathbb{E}\left[\xi_{m}^{i;F}\mid\mathbf{Z}\right]=0.
%\end{aligned}
%$}
%%}
%\end{equation*}

%\begin{equation*}
%{
%\begin{aligned}
%&\mathbb{E}\left[\left(\mathscr{E}_{m}^{i}-\xi_{m}^{i;F}\right)\mid\mathbf{Z}\right]=\mathbb{E}\left[\left(\frac{1}{R}\overset{R}{\underset{u=1}{\sum}} \xi_{m,u}^{i;F}-\xi_{m}^{i;F}\right)\mid\mathbf{Z}\right]\\
%=&\frac{1}{R}\overset{R}{\underset{u=1}{\sum}} \mathbb{E}\left[\xi_{m,u}^{i;F}\mid\mathbf{Z}\right]-\mathbb{E}\left[\xi_{m}^{i;F}\mid\mathbf{Z}\right]=\frac{1}{R}\overset{R}{\underset{u=1}{\sum}} \mathbb{E}\left[\xi_{m}^{i;F}\mid\mathbf{Z}\right]-\mathbb{E}\left[\xi_{m}^{i;F}\mid\mathbf{Z}\right]\\
%=& \mathbb{E}\left[\xi_{m}^{i;F}\mid\mathbf{Z}\right]-\mathbb{E}\left[\xi_{m}^{i;F}\mid\mathbf{Z}\right]=0.
%\end{aligned}
%}
%\end{equation*}
In addition, we simplify the quantity {$\mathbb{E}\left[\left(\mathscr{E}_{m}^{i}-\xi_{m}^{i;F}\right)^{2}\mid\mathbf{Z}\right]$}. Note that {$\mathscr{E}_{m}^{i}-\xi_{m}^{i;F}=\frac{1}{R}\underset{u=1}{\overset{R}{\sum}}[\xi_{m,u}^{i;F}-\xi_{m}^{i;F}]$}. We therefore have
\begin{equation*}
{
\begin{aligned}
&\mathbb{E}\left[\left(\overset{R}{\underset{u=1}{\sum}}\left[\xi_{m,u}^{i;F}-\xi_{m}^{i;F}\right]\right)^{2}\mid\mathbf{Z}\right]
=\overset{R}{\underset{u,\bar{u}=1}{\sum}}\mathbb{E}\left[\left(\xi_{m,u}^{i;F}-\xi_{m}^{i;F}\right)\left(\xi_{m,\bar{u}}^{i;F}-\xi_{m}^{i;F}\right)\mid\mathbf{Z}\right]\\
=&\overset{R}{\underset{u,\bar{u}=1}{\sum}}\left\{\mathbb{E}\left[\xi_{m,u}^{i;F}\xi_{m,\bar{u}}^{i;F}\mid\mathbf{Z}\right]-\mathbb{E}\left[\xi_{m,u}^{i;F}\xi_{m}^{i;F}\mid\mathbf{Z}\right]\right.
\left.-\mathbb{E}\left[\xi_{m}^{i;F}\xi_{m,\bar{u}}^{i;F}\mid\mathbf{Z}\right]+\mathbb{E}\left[\xi_{m}^{i;F}\xi_{m}^{i;F}\mid\mathbf{Z}\right]\right\}\\
=&\overset{R}{\underset{u=1}{\sum}}\mathbb{E}\left[(\xi_{m,u}^{i;F})^{2}\mid\mathbf{Z}\right]-2R\overset{R}{\underset{u=1}{\sum}}\mathbb{E}\left[\xi_{m,u}^{i;F}\xi_{m}^{i;F}\mid\mathbf{Z}\right]+R^{2}\mathbb{E}\left[(\xi_{m}^{i;F})^{2}\mid\mathbf{Z}\right]
+\overset{R}{\underset{\substack{u,\bar{u}=1 \\ u\neq\bar{u}}}{\sum}}\mathbb{E}\left[\xi_{m,u}^{i;F}\xi_{m,\bar{u}}^{i;F}\mid\mathbf{Z}\right]\\
=&\left[R^{2}+R\right]\mathbb{E}\left[\left(\xi_{m}^{i;F}\right)^{2}\mid\mathbf{Z}\right].
\end{aligned}
}
\end{equation*}
We justify the last equality. The last equality follows from the fact that, conditioning on {$\mathbf{Z}$}, {$\xi_{m,u}^{i;F}$} are i.i.d. of {$\xi_{m}^{i;F}$} and {$\xi_{m,u}^{i;F}$} are i.i.d. of {$\xi_{m,\bar{u}}^{i;F}$} for any {$u,\;\bar{u}\in\{1,2,\cdots,R\}$}. Indeed, under the given fact, we have
\begin{equation*}
{
\begin{aligned}
&\mathbb{E}\left[\xi_{m,u}^{i;F}\xi_{m}^{i;F}\mid \mathbf{Z}\right]=\mathbb{E}\left[\xi_{m,u}^{i;F}\mid \mathbf{Z}\right]\mathbb{E}\left[\xi_{m}^{i;F}\mid \mathbf{Z}\right]\\
=&\mathbb{E}\left[\mathbb{E}\left[\xi_{m,u}^{i;F}\mid D, \mathbf{Z}\right]\mid\mathbf{Z}\right]\mathbb{E}\left[\mathbb{E}\left[\xi_{m}^{i;F}\mid D,\mathbf{Z}\right]\mid \mathbf{Z}\right]=0
\end{aligned}
}
\end{equation*}
and
\begin{equation*}
{
\begin{aligned}
&\mathbb{E}\left[\xi_{m,u}^{i;F}\xi_{m,\bar{u}}^{i;F}\mid \mathbf{Z}\right]=\mathbb{E}\left[\xi_{m,u}^{i;F}\mid \mathbf{Z}\right]\mathbb{E}\left[\xi_{m,\bar{u}}^{i;F}\mid \mathbf{Z}\right]\qquad (u\ne \bar{u})\\
=&\mathbb{E}\left[\mathbb{E}\left[\xi_{m,u}^{i;F}\mid D,\mathbf{Z}\right]\mid \mathbf{Z}\right]\mathbb{E}\left[\mathbb{E}\left[\xi_{m,\bar{u}}^{i;F}\mid D,\mathbf{Z}\right]\mid \mathbf{Z}\right]=0.
\end{aligned}
}
\end{equation*}
%\begin{equation*}
%{
%\begin{aligned}
%&\mathbb{E}\left[\xi_{m,u}^{i;F}\xi_{m}^{i;F}\right]=\mathbb{E}\left[\mathbb{E}\left[\xi_{m,u}^{i;F}\xi_{m}^{i;F}\mid D,\mathbf{Z}\right]\right]\\
%=&\mathbb{E}\left[\mathbb{E}\left[\xi_{m,u}^{i;F}D, \mid\mathbf{Z}\right]\mathbb{E}\left[\xi_{m}^{i;F}\mid D,\mathbf{Z}\right]\right]=0
%\end{aligned}
%}
%\end{equation*}
%and
%\begin{equation*}
%{
%\begin{aligned}
%&\mathbb{E}\left[\xi_{m,u}^{i;F}\xi_{m,\bar{u}}^{i;F}\right]=\mathbb{E}\left[\mathbb{E}\left[\xi_{m,u}^{i;F}\xi_{m,\bar{u}}^{i;F}\mid D,\mathbf{Z}\right]\right]\\
%=&\mathbb{E}\left[\mathbb{E}\left[\xi_{m,u}^{i;F}\mid D,\mathbf{Z}\right]\mathbb{E}\left[\xi_{m,\bar{u}}^{i;F}\mid D,\mathbf{Z}\right]\right]=0.
%\end{aligned}
%}
%\end{equation*}
Consequently, we have
\begin{equation*}
{
\begin{aligned}
\mathbb{E}\left[\left(\mathscr{E}_{m}^{i}-\xi_{m}^{i;F}\right)^{2}\mid\mathbf{Z}\right]=\left(1+\frac{1}{R}\right)\mathbb{E}\left[\left(\xi_{m}^{i;F}\right)^{2}\mid\mathbf{Z}\right].
\end{aligned}
}
\end{equation*}
Thus, we have
\begin{equation}
{
\begin{aligned}\label{eqt:summation R bound}
\mathbb{P}\left\{\left|\kappa_{N}^{i;F}-\kappa_{R,N}^{i;F}\right|\geq\epsilon\right\}&\leq \frac{\frac{1}{N^{2}}\underset{m\in \mathscr{I}^{c}}{\overset{}{\sum}}\mathbb{E}\left[(A_{m}^{i})^{2}
\left(1+\frac{1}{R}\right)\mathbb{E}\left[\left(\xi_{m}^{i;F}\right)^{2}\mid\mathbf{Z}\right]\right]}{\epsilon^{2}}\\
&\leq \frac{\left(1+\frac{1}{R}\right)\mathbb{E}\left[(A^{i})^{2}\mathbb{E}\left[(\xi^{i;F})^{2}\mid\mathbf{Z}\right]\right]}{N\epsilon^{2}}.
\end{aligned}
}
\end{equation}
We notice that no matter we set $R\rightarrow \infty$ followed by $N\rightarrow \infty$, or we fix $R$ but let $N\rightarrow \infty$, we see that {$\mathbb{P}\left\{\left|\kappa_{N}^{i;F}-\kappa_{R,N}^{i;F}\right|\geq\epsilon\right\} \rightarrow 0$}. The proof is completed.
%where $\sigma^{i}:=\mathbb{E}\left[\left(\xi_{m}^{i;F}\right)^{2}\right]$.
%We notice that no matter we set $R\rightarrow \infty$ followed by $N\rightarrow \infty$ or vice versa, or we fix $R$ but letting $N\rightarrow \infty$, we see that {$\mathbb{P}\left\{\left|\kappa_{N}^{i;F}-\kappa_{N}^{i;CF}\right|\geq\epsilon\right\}$} converges to $0$ from \eqref{eqt:statistical difference chebyshev}.

\section{Proof of The Consistency of $\hat{\theta}_{N}^{i}$}
\label{appendix:D}

$\forall\epsilon>0$, we have
\begin{equation*}
{
\begin{aligned}
&\mathbb{P}_{\hat{\rho}}\left\{\left|\hat{\theta}_{N}^{i}-\theta^{i}\right| \geq \epsilon\right\}
=\mathbb{P}_{\hat{\rho}}\left\{\left|\hat{\theta}_{N}^{i}-\hat{\bar{\theta}}_{N}^{i}+\hat{\bar{\theta}}_{N}^{i}-\theta^{i}\right| \geq \epsilon\right\}\\
\leq&\mathbb{P}_{\hat{\rho}}\left\{\left|\hat{\theta}_{N}^{i}-\hat{\bar{\theta}}_{N}^{i}\right| \geq \frac{\epsilon}{2}\right\}+\mathbb{P}_{\hat{\rho}}\left\{\left|\hat{\bar{\theta}}_{N}^{i}-\theta^{i}\right| \geq \frac{\epsilon}{2}\right\}.
\end{aligned}
}
\end{equation*}
Since {$\left(\xi^{i;F}\overset{d}{=}\xi^{i;CF}\right)\mid\mathbf{Z}$} and {$\left(\hat{\xi}^{i;F}\overset{d}{=}\hat{\xi}^{i;CF}\right)\mid\mathbf{Z}$}, we have {$\hat{\bar{\theta}}_{N}^{i}\overset{d}{=}\hat{\tilde{\theta}}_{N}^{i}$} by Lemma \ref{lemma:simple lemma2}. Moreover, we know that {$\mathbb{P}_{\hat{\rho}}\left\{\left|\hat{\tilde{\theta}}_{N}^{i}-\theta^{i}\right| \geq \frac{\epsilon}{2}\right\}\overset{p}{\rightarrow}0$} under the assumptions given in \cite{mackey2018orthogonal}. Together with the fact that {$\hat{\bar{\theta}}_{N}^{i}\overset{d}{=}\hat{\tilde{\theta}}_{N}^{i}$}, we have {$\mathbb{P}_{\hat{\rho}}\left\{\left|\hat{\bar{\theta}}_{N}^{i}-\theta^{i}\right| \geq \frac{\epsilon}{2}\right\}\overset{p}{\rightarrow}0$} by Lemma \ref{lemma:simple lemma}. We turn to consider the quantity {$\mathbb{P}_{\hat{\rho}}\left\{\left|\hat{\theta}_{N}^{i}-\hat{\bar{\theta}}_{N}^{i}\right| \geq \frac{\epsilon}{2}\right\}$}, and we aim to show that {$\mathbb{P}_{\hat{\rho}}\left\{\left|\hat{\theta}_{N}^{i}-\hat{\bar{\theta}}_{N}^{i}\right| \geq \frac{\epsilon}{2}\right\}\overset{p}{\rightarrow} 0$}. Notice that
\begin{equation}
{
\begin{aligned}\label{eqt:consistent proof estimator decomp}
&\mathbb{P}_{\hat{\rho}}\left\{\left|\hat{\theta}_{N}^{i}-\hat{\bar{\theta}}_{N}^{i}\right| \geq \frac{\epsilon}{2}\right\}=\mathbb{P}_{\hat{\rho}}\left\{\left|\hat{\kappa}_{R,N}^{i;F}-\hat{\kappa}_{N}^{i;F}\right| \geq \frac{\epsilon}{2}\right\}\\
\leq&\underbrace{\mathbb{P}_{\hat{\rho}}\left\{\left|\hat{\kappa}_{R,N}^{i;F}-\kappa_{R,N}^{i;F}\right| \geq \frac{\epsilon}{8}\right\}}_{(a)}+\underbrace{\mathbb{P}_{\hat{\rho}}\left\{\left|\kappa_{R,N}^{i;F}-\kappa_{N}^{i;F}\right| \geq \frac{\epsilon}{8}\right\}}_{(b)}\\
&\;+\underbrace{\mathbb{P}_{\hat{\rho}}\left\{\left|\kappa_{N}^{i;F}-\kappa_{N}^{i;CF}\right| \geq \frac{\epsilon}{8}\right\}}_{(c)}+\underbrace{\mathbb{P}_{\hat{\rho}}\left\{\left|\kappa_{N}^{i;CF}-\hat{\kappa}_{N}^{i;F}\right| \geq \frac{\epsilon}{8}\right\}}_{(d)}.
\end{aligned}
}
\end{equation}
Note that {$\left|\kappa_{R,N}^{i;F}-\kappa_{N}^{i;F}\right|$} and {$\left|\kappa_{N}^{i;F}-\kappa_{N}^{i;CF}\right|$} do not incorporate any terms related to the estimated {$\hat{\rho}$}. From Proposition \ref{lemma:difference from the statistical standpoint} and Proposition \ref{lemma:unbiasedness and consistency}, we conclude that (\ref{eqt:consistent proof estimator decomp}b) and (\ref{eqt:consistent proof estimator decomp}c) converge to {$0$} in probability respectively. 
It remains to show the convergence of (\ref{eqt:consistent proof estimator decomp}a) and (\ref{eqt:consistent proof estimator decomp}d).
%It remains to show the convergence of {$\mathbb{P}_{\hat{\rho}}\left\{\left|\hat{\kappa}_{R,N}^{i;F}-\kappa_{R,N}^{i;F}\right| \geq \frac{\epsilon}{8}\right\}$} and {$\mathbb{P}_{\hat{\rho}}\left\{\left|\kappa_{N}^{i;CF}-\hat{\kappa}_{N}^{i;F}\right| \geq \frac{\epsilon}{8}\right\}$}.
Consider (\ref{eqt:consistent proof estimator decomp}d) first. Since
\begin{equation*}
{
\begin{aligned}
&\left|\kappa_{N}^{i;CF}-\hat{\kappa}_{N}^{i;F}\right|=\left|\frac{1}{N}\underset{m\in \mathscr{I}^{c}}{\sum}\xi_{m}^{i;CF}A_{m}^{i}-\frac{1}{N}\underset{m\in \mathscr{I}^{c}}{\sum}\hat{\xi}_{m}^{i;F}\hat{A}_{m}^{i}\right|\\
\leq&\underbrace{\left|\frac{1}{N}\underset{m\in \mathscr{I}^{c}}{\sum}(\xi_{m}^{i;CF}A_{m}^{i}-\xi_{m}^{i;F}A_{m}^{i})\right|}_{\Gamma_{1}}\\
&\;+\underbrace{\left|\frac{1}{N}\underset{m\in \mathscr{I}^{c}}{\sum}(\xi_{m}^{i;F}A_{m}^{i}-\hat{\xi}_{m}^{i;F}A_{m}^{i})\right|}_{\Gamma_{2}}+\underbrace{\left|\frac{1}{N}\underset{m\in \mathscr{I}^{c}}{\sum}\hat{\xi}_{m}^{i;F}(A_{m}^{i}-\hat{A}_{m}^{i})\right|}_{\Gamma_{3}},
\end{aligned}
}
\end{equation*}
(\ref{eqt:consistent proof estimator decomp}d) is bounded above by
\begin{equation}
{
\begin{aligned}
&\underbrace{\mathbb{P}_{\hat{\rho}}\left\{\Gamma_{1} \geq \frac{\epsilon}{24}\right\}}_{(a)}+\underbrace{\mathbb{P}_{\hat{\rho}}\left\{\Gamma_{2} \geq \frac{\epsilon}{24}\right\}}_{(b)}+\underbrace{\mathbb{P}_{\hat{\rho}}\left\{\Gamma_{3} \geq \frac{\epsilon}{24}\right\}}_{(c)}.\label{eqt:consistency check equation}
\end{aligned}
}\\
\end{equation}
(\ref{eqt:consistency check equation}a) converges to $0$ in probability due to Proposition \ref{lemma:difference from the statistical standpoint}. We study the quantities (\ref{eqt:consistency check equation}b) and (\ref{eqt:consistency check equation}c). 
	
(\ref{eqt:consistency check equation}b) can be further bounded. If {$N^{c}$} is the size of {$\mathscr{I}^{c}$}, then we have
%(\ref{eqt:consistency check equation}b) can be further bounded. If {$N^{c}$} is the size of {$\mathfrak{D}_{i}^{c}\cap\mathscr{I}$}, then we have
\begin{equation*}
{
\begin{aligned}
\Gamma_{2}=&\left|\frac{1}{N}\underset{m\in \mathscr{I}^{c}}{\sum}(\xi_{m}^{i;F}A_{m}^{i}-\hat{\xi}_{m}^{i;F}A_{m}^{i})\right|
%=&\left|\frac{1}{N}\underset{m\in \mathscr{I}^{c}}{\sum}\xi_{m}^{i;F}A_{m}^{i}-\frac{N^{c}}{N}\mathbb{E}_{\hat{\rho}}\left[\xi^{i;F}A^{i}\right]+\frac{N^{c}}{N}\mathbb{E}_{\hat{\rho}}\left[\xi^{i;F}A^{i}\right]-\frac{N^{c}}{N}\mathbb{E}_{\hat{\rho}}\left[\hat{\xi}^{i;F}A^{i}\right]+\frac{N^{c}}{N}\mathbb{E}_{\hat{\rho}}\left[\hat{\xi}^{i;F}A^{i}\right]-\frac{1}{N}\underset{m\in \mathscr{I}^{c}}{\sum}\hat{\xi}_{m}^{i;F}A_{m}^{i}\right|\\
\leq\underbrace{\left|\frac{1}{N}\underset{m\in \mathscr{I}^{c}}{\sum}\xi_{m}^{i;F}A_{m}^{i}-\frac{N^{c}}{N}\mathbb{E}_{\hat{\rho}}\left[\xi^{i;F}A^{i}\right]\right|}_{\Gamma_{2;1}}\\
&\;+\underbrace{\left|\frac{N^{c}}{N}\mathbb{E}_{\hat{\rho}}\left[\xi^{i;F}A^{i}\right]-\frac{N^{c}}{N}\mathbb{E}_{\hat{\rho}}\left[\hat{\xi}^{i;F}A^{i}\right]\right|}_{\Gamma_{2;2}}
+\underbrace{\left|\frac{N^{c}}{N}\mathbb{E}_{\hat{\rho}}\left[\hat{\xi}^{i;F}A^{i}\right]-\frac{1}{N}\underset{m\in \mathscr{I}^{c}}{\sum}\hat{\xi}_{m}^{i;F}A_{m}^{i}\right|}_{\Gamma_{2;3}}.
\end{aligned}
}
\end{equation*}
We see that (\ref{eqt:consistency check equation}b) can be further bounded by
\begin{equation}
{
\begin{aligned}
&\underbrace{\mathbb{P}_{\hat{\rho}}\left\{\Gamma_{2;1} \geq \frac{\epsilon}{72}\right\}}_{(a)}+\underbrace{\mathbb{P}_{\hat{\rho}}\left\{\Gamma_{2;2} \geq \frac{\epsilon}{72}\right\}}_{(b)}+\underbrace{\mathbb{P}_{\hat{\rho}}\left\{\Gamma_{2;3} \geq \frac{\epsilon}{72}\right\}}_{(c)}.\label{eqt:term 1 check}
\end{aligned}
}
\end{equation}
We investigate if (\ref{eqt:term 1 check}a), (\ref{eqt:term 1 check}b), and (\ref{eqt:term 1 check}c) converge to $0$ in probability. We consider (\ref{eqt:term 1 check}a) first. Recall the assumptions that {$(\xi^{i;F} \perp \!\!\! \perp D) \mid \mathbf{Z}$}, {$\left(\xi^{i;F}\overset{d}{=}\xi^{i;CF}\right) \mid \mathbf{Z}$}, and {$(\xi^{i;CF} \perp \!\!\! \perp D) \mid \mathbf{Z}$}, we have {$\frac{1}{N}\underset{m\in \mathscr{I}^{c}}{\sum}\xi_{m}^{i;F}A_{m}^{i}\overset{d}{=}\frac{1}{N}\underset{m\in \mathscr{I}^{c}}{\sum}\xi_{m}^{i;CF}A_{m}^{i}$} by Lemma \ref{lemma:simple lemma2}.	Since {$\Gamma_{2;1}=\frac{N^{c}}{N}\left|\frac{1}{N^{c}}\underset{m\in \mathscr{I}^{c}}{\sum}\xi_{m}^{i;F}A_{m}^{i}-\mathbb{E}_{\hat{\rho}}\left[\xi^{i;F}A^{i}\right]\right|$}, we have
\begin{equation*}
{
\begin{aligned}
&\mathbb{P}_{\hat{\rho}}\left\{\Gamma_{2;1}\geq\frac{\epsilon}{72}\right\}
%=&\mathbb{P}_{\hat{\rho}}\left\{\frac{N^{c}}{N}\left|\frac{1}{N^{c}}\underset{m\in \mathscr{I}^{c}}{\sum}\xi_{m}^{i;F}A_{m}^{i}-\mathbb{E}_{\hat{\rho}}\left[\xi^{i;F}A^{i}\right]\right|\geq\frac{\epsilon}{72}\right\}\\
=\mathbb{P}_{\hat{\rho}}\left\{\left|\frac{1}{N^{c}}\underset{m\in \mathscr{I}^{c}}{\sum}\xi_{m}^{i;F}A_{m}^{i}-\mathbb{E}_{\hat{\rho}}\left[\xi^{i;F}A^{i}\right]\right|\geq\frac{\epsilon}{72}\cdot\frac{N}{N^{c}}\right\}\\
\leq& \mathbb{P}_{\hat{\rho}}\left\{\left|\frac{1}{N^{c}}\underset{m\in \mathscr{I}^{c}}{\sum}\xi_{m}^{i;F}A_{m}^{i}-\mathbb{E}_{\hat{\rho}}\left[\xi^{i;F}A^{i}\right]\right|\geq\frac{\epsilon}{72}\right\}
\leq\frac{\mathbb{E}_{\hat{\rho}}\left[\left|\frac{1}{N^{c}}\underset{m\in \mathscr{I}^{c}}{\sum}\xi_{m}^{i;F}A_{m}^{i}-\mathbb{E}_{\hat{\rho}}\left[\xi^{i;F}A^{i}\right]\right|^{2}\right]}{\left(\frac{\epsilon}{72}\right)^{2}}.
\end{aligned}
}
\end{equation*}
Consider the numerator term, %{$\mathbb{E}_{\hat{\rho}}\left[\left|\frac{1}{N^{c}}\underset{m\in \mathscr{I}^{c}}{\sum}\xi_{m}^{i;F}A_{m}^{i}-\mathbb{E}_{\hat{\rho}}\left[\xi^{i;F}A^{i}\right]\right|^{2}\right]$}. 
note that it equals
\begin{equation*}
%{
\resizebox{0.96\textwidth}{!}{$
\begin{aligned}
&\frac{1}{(N^{c})^{2}}\underset{m\in \mathscr{I}^{c}}{\sum}\mathbb{E}_{\hat{\rho}}\left[\left|\xi_{m}^{i;F}A_{m}^{i}-\mathbb{E}_{\hat{\rho}}\left[\xi^{i;F}A^{i}\right]\right|^{2}\right]
%=&\frac{1}{\left(N^{c}\right)^{2}}\underset{m\in \mathscr{I}^{c}}{\sum}\mathbb{E}_{\hat{\rho}}\left[\left|\xi_{m}^{i;F}A_{m}^{i}-\mathbb{E}_{\hat{\rho}}\left[\xi^{i;F}A^{i}\right]\right|^{2}\right]\\
+\frac{1}{\left(N^{c}\right)^{2}}\underset{\substack{m,\bar{m}\in \mathscr{I}^{c}\\m\neq\bar{m}}}{\sum}\mathbb{E}_{\hat{\rho}}\left[(\xi_{m}^{i;F}A_{m}^{i}-\mathbb{E}_{\hat{\rho}}[\xi^{i;F}A^{i}])(\xi_{\bar{m}}^{i;F}A_{\bar{m}}^{i}-\mathbb{E}_{\hat{\rho}}[\xi^{i;F}A^{i}])\right]\\
=&\frac{1}{\left(N^{c}\right)^{2}}\underset{m\in \mathscr{I}^{c}}{\sum}\mathbb{E}_{\hat{\rho}}\left[\left(A_{m}^{i}\right)^{2}\mathbb{E}_{\hat{\rho}}\left[\left(\xi_{m}^{i;F}\right)^{2}\mid D,\mathbf{Z}\right]\right]
+\frac{1}{\left(N^{c}\right)^{2}}\underset{\substack{m,\bar{m}\in \mathscr{I}^{c}\\m\neq\bar{m}}}{\sum}\mathbb{E}_{\hat{\rho}}\left[A_{m}^{i}A_{\bar{m}}^{i}\mathbb{E}_{\hat{\rho}}\left[\xi_{m}^{i;F}\mid D,\mathbf{Z}\right]\mathbb{E}_{\hat{\rho}}\left[\xi_{\bar{m}}^{i;F}\mid D,\mathbf{Z}\right]\right]\\
=&\frac{1}{\left(N^{c}\right)^{2}}\underset{m\in \mathscr{I}^{c}}{\sum}\mathbb{E}_{\hat{\rho}}\left[\left(A_{m}^{i}\right)^{2}\mathbb{E}_{\hat{\rho}}\left[\left(\xi_{m}^{i;F}\right)^{2}\mid\mathbf{Z}\right]\right]
=\frac{1}{N^{c}}\mathbb{E}_{\hat{\rho}}\left[\left(A^{i}\right)^{2}\mathbb{E}_{\hat{\rho}}\left[\left(\xi^{i;F}\right)^{2}\mid\mathbf{Z}\right]\right].
\end{aligned}
$}
%}
\end{equation*}
Since {$A^{i}$} and {$\xi^{i;F}$} do not include the estimated nuisance parameters, {$\mathbb{E}_{\hat{\rho}}\left[\left(A^{i}\right)^{2}\mathbb{E}_{\hat{\rho}}\left[\left(\xi^{i;F}\right)^{2}\mid\mathbf{Z}\right]\right]$} is a constant. Moreover, note that {$N^{c}\rightarrow \infty$} when {$N\rightarrow \infty$}, we have
\begin{equation*}
{
\begin{aligned}
\mathbb{P}_{\hat{\rho}}\left\{\Gamma_{2;1}\geq\frac{\epsilon}{72}\right\}\leq\frac{72^2\;\mathbb{E}_{\hat{\rho}}\left[\left(A^{i}\right)^{2}\mathbb{E}_{\hat{\rho}}\left[\left(\xi^{i;F}\right)^{2}\mid\mathbf{Z}\right]\right]}{\epsilon^{2}N^{c}}\overset{p}{\longrightarrow}0.
\end{aligned}
}
\end{equation*}

Now, we consider (\ref{eqt:term 1 check}b). Indeed, we have
\begin{equation*}
{
\begin{aligned}
&\mathbb{P}_{\hat{\rho}}\left\{\Gamma_{2;2}\geq\frac{\epsilon}{72}\right\}
=\mathbb{P}_{\hat{\rho}}\left\{\left|\mathbb{E}_{\hat{\rho}}\left[\xi^{i;F}A^{i}\right]-\mathbb{E}_{\hat{\rho}}\left[\hat{\xi}^{i;F}A^{i}\right]\right|\geq\frac{\epsilon}{72}\cdot\frac{N}{N^{c}}\right\}\\
\leq&\mathbb{P}_{\hat{\rho}}\left\{\left|\mathbb{E}_{\hat{\rho}}\left[(\xi^{i;F}-\hat{\xi}^{i;F})A^{i}\right]\right|\geq\frac{\epsilon}{72}\right\}
\leq\frac{72^2\left\{\mathbb{E}_{\hat{\rho}}\left[(\xi^{i;F}-\hat{\xi}^{i;F})A^{i}\right]\right\}^{2}}{\epsilon^{2}}\\
\leq&\frac{72^2\;\left\{\mathbb{E}_{\hat{\rho}}\left[(\xi^{i;F}-\hat{\xi}^{i;F})^{4q}\right]\right\}^{\frac{1}{2q}}\left\{\mathbb{E}_{\hat{\rho}}\left[(A^{i})^{\frac{4q}{4q-1}}\right]\right\}^{2-\frac{1}{2q}}}{\epsilon^{2}}\overset{p}{\longrightarrow}0.
\end{aligned}
}
\end{equation*}
Here, the last inequality follows from the H\"{o}lders inequality, and thereby the convergence holds $\forall q \in \{1,2,\dots,k\}$ according to  Assumption 1.5 of \cite{mackey2018orthogonal}. Finally, we consider (\ref{eqt:term 1 check}c). We can rewrite {$\Gamma_{2;3}$} as
\begin{equation*}
{
\begin{aligned}
\Gamma_{2;3}=\frac{N^{c}}{N}\left|\mathbb{E}_{\hat{\rho}}\left[\hat{\xi}^{i;F}A^{i}\right]-\frac{1}{N^{c}}\underset{m\in \mathscr{I}^{c}}{\sum}\hat{\xi}_{m}^{i;F}A_{m}^{i}\right|.
\end{aligned}
}
\end{equation*}
Now, we have
\begin{equation*}
{
\begin{aligned}
&\mathbb{P}_{\hat{\rho}}\left\{\Gamma_{2;3}\geq\frac{\epsilon}{72}\right\}
\leq\mathbb{P}_{\hat{\rho}}\left\{\left|\mathbb{E}_{\hat{\rho}}\left[\hat{\xi}^{i;F}A^{i}\right]-\frac{1}{N^{c}}\underset{m\in \mathscr{I}^{c}}{\sum}\hat{\xi}_{m}^{i;F}A_{m}^{i}\right|\geq\frac{\epsilon}{72}\right\}\\
\leq&\frac{72^2\;\mathbb{E}_{\hat{\rho}}\left[\left\{\underset{m\in \mathscr{I}^{c}}{\sum}\left(\mathbb{E}_{\hat{\rho}}\left[\hat{\xi}^{i;F}A^{i}\right]-\hat{\xi}_{m}^{i;F}A_{m}^{i}\right)\right\}^{2}\right]}{\epsilon^{2}\left(N^{c}\right)^{2}}\\
=&\frac{72^2\;\underset{m\in \mathscr{I}^{c}}{\sum}\;\mathbb{E}_{\hat{\rho}}\left[\left(\mathbb{E}_{\hat{\rho}}\left[\hat{\xi}^{i;F}A^{i}\right]-\hat{\xi}_{m}^{i;F}A_{m}^{i}\right)^{2}\right]}{\epsilon^{2}\left(N^{c}\right)^{2}}\\
&\;+\frac{72^2\;\underset{\substack{m,\bar{m}\in \mathscr{I}^{c}\\m\neq\bar{m}}}{\sum}\mathbb{E}_{\hat{\rho}}[(\hat{\xi}_{m}^{i;F}-\xi_{m}^{i;F})A_{m}^{i}]\mathbb{E}_{\hat{\rho}}[(\hat{\xi}_{\bar{m}}^{i;F}-\xi_{\bar{m}}^{i;F})A_{\bar{m}}^{i}]}{\epsilon^{2}\left(N^{c}\right)^{2}}\\
&\;-2\frac{72^2\;(N^c-1)\underset{m\in \mathscr{I}^{c}}{\sum}\mathbb{E}_{\hat{\rho}}[(\hat{\xi}_{m}^{i;F}-\xi_{m}^{i;F})A_{m}^{i}]\mathbb{E}_{\hat{\rho}}[(\hat{\xi}^{i;F}-\xi^{i;F})A^{i}]}{\epsilon^{2}\left(N^{c}\right)^{2}}\\
&\;+\frac{72^2\;\underset{\substack{m,\bar{m}\in \mathscr{I}^{c}\\m\neq\bar{m}}}{\sum}\left\{\mathbb{E}_{\hat{\rho}}\left[\left(\hat{\xi}^{i;F}-\xi^{i;F}\right)A^{i}\right]\right\}^{2}}{\epsilon^{2}\left(N^{c}\right)^{2}}.
\end{aligned}
}
\end{equation*}
Using Assumption 1.5 of \cite{mackey2018orthogonal}, we can conclude that {$\mathbb{P}_{\hat{\rho}}\left\{\Gamma_{2;3}\geq\frac{\epsilon}{72}\right\}\overset{p}{\rightarrow}0$}. 

Next, we come to bound (\ref{eqt:consistency check equation}c). Since {$N^{c}$} is the size of {$\mathscr{I}^{c}$} and
%Using Assumption 1.5 of \cite{mackey2018orthogonal}, we can conclude that {$\mathbb{P}_{\hat{\rho}}\left\{\Gamma_{2;3}\geq\frac{\epsilon}{72}\right\}\overset{p}{\rightarrow}0$.} Next, we come to bound (\ref{eqt:consistency check equation}c). Again, we denote {$N^{c}$} as the size of {$\mathfrak{D}_{i}^{c}\cap\mathscr{I}$}. Since
\begin{equation*}
{
\begin{aligned}
\Gamma_{3}=&\left|\frac{1}{N}\underset{m\in \mathscr{I}^{c}}{\sum}\hat{\xi}_{m}^{i;F}A_{m}^{i}-\frac{1}{N}\underset{m\in \mathscr{I}^{c}}{\sum}\hat{\xi}_{m}^{i;F}\hat{A}_{m}^{i}\right|
%=&\left|\frac{1}{N}\underset{m\in \mathscr{I}^{c}}{\sum}\hat{\xi}_{m}^{i;F}A_{m}^{i}-\frac{N^{c}}{N}\mathbb{E}_{\hat{\rho}}\left[\hat{\xi}^{i;F}A^{i}\right]+\frac{N^{c}}{N}\mathbb{E}_{\hat{\rho}}\left[\hat{\xi}^{i;F}A^{i}\right]-\frac{N^{c}}{N}\mathbb{E}_{\hat{\rho}}\left[\hat{\xi}^{i;F}\hat{A}^{i}\right]+\frac{N^{c}}{N}\mathbb{E}_{\hat{\rho}}\left[\hat{\xi}^{i;F}\hat{A}^{i}\right]-\frac{1}{N}\underset{m\in \mathscr{I}^{c}}{\sum}\hat{\xi}_{m}^{i;F}\hat{A}_{m}^{i}\right|\\
\leq\left|\frac{1}{N}\underset{m\in \mathscr{I}^{c}}{\sum}\hat{\xi}_{m}^{i;F}A_{m}^{i}-\frac{N^{c}}{N}\mathbb{E}_{\hat{\rho}}\left[\hat{\xi}^{i;F}A^{i}\right]\right|\\
&+\left|\frac{N^{c}}{N}\mathbb{E}_{\hat{\rho}}\left[\hat{\xi}^{i;F}A^{i}\right]-\frac{N^{c}}{N}\mathbb{E}_{\hat{\rho}}\left[\hat{\xi}^{i;F}\hat{A}^{i}\right]\right|
+\left|\frac{N^{c}}{N}\mathbb{E}_{\hat{\rho}}\left[\hat{\xi}^{i;F}\hat{A}^{i}\right]-\frac{1}{N}\underset{m\in \mathscr{I}^{c}}{\sum}\hat{\xi}_{m}^{i;F}\hat{A}_{m}^{i}\right|\\
=&\frac{N^{c}}{N}\underbrace{\left|\frac{1}{N^{c}}\underset{m\in \mathscr{I}^{c}}{\sum}\hat{\xi}_{m}^{i;F}A_{m}^{i}-\mathbb{E}_{\hat{\rho}}\left[\hat{\xi}^{i;F}A^{i}\right]\right|}_{\Gamma_{3;1}}
+\frac{N^{c}}{N}\underbrace{\left|\mathbb{E}_{\hat{\rho}}\left[\hat{\xi}^{i;F}A^{i}\right]-\mathbb{E}_{\hat{\rho}}\left[\hat{\xi}^{i;F}\hat{A}^{i}\right]\right|}_{\Gamma_{3;2}}\\
&+\frac{N^{c}}{N}\underbrace{\left|\mathbb{E}_{\hat{\rho}}\left[\hat{\xi}^{i;F}\hat{A}^{i}\right]-\frac{1}{N^{c}}\underset{m\in \mathscr{I}^{c}}{\sum}\hat{\xi}_{m}^{i;F}\hat{A}_{m}^{i}\right|}_{\Gamma_{3;3}},
\end{aligned}
}
\end{equation*}
we see that (\ref{eqt:consistency check equation}c) can be further bounded by
\begin{equation}
{
\begin{aligned}
&\underbrace{\mathbb{P}_{\hat{\rho}}\left\{\Gamma_{3;1} \geq \frac{\epsilon}{72}\right\}}_{(a)}+\underbrace{\mathbb{P}_{\hat{\rho}}\left\{\Gamma_{3;2} \geq \frac{\epsilon}{72}\right\}}_{(b)}+\underbrace{\mathbb{P}_{\hat{\rho}}\left\{\Gamma_{3;3} \geq \frac{\epsilon}{72}\right\}}_{(c)}.\label{eqt:term 2 check}
\end{aligned}
}
\end{equation}
Similarly, we can prove that (\ref{eqt:term 2 check}a), (\ref{eqt:term 2 check}b), and (\ref{eqt:term 2 check}c) converge to $0$ in probability when $N\rightarrow \infty$ using the arguments in proving that (\ref{eqt:term 1 check}a), (\ref{eqt:term 1 check}b), and (\ref{eqt:term 1 check}c) converge to $0$. As a result, the quantity (\ref{eqt:consistent proof estimator decomp}d) converges to $0$ in probability when $N\rightarrow \infty$.
	
Lastly, we turn to consider the quantity (\ref{eqt:consistent proof estimator decomp}a). In fact, we have
\begin{subequations}
{
\begin{align}
&\mathbb{P}_{\hat{\rho}}\left\{\left|\hat{\kappa}_{R,N}^{i;F}-\kappa_{R,N}^{i;F}\right|\geq\frac{\epsilon}{8}\right\}\nonumber\\
\leq&\mathbb{P}_{\hat{\rho}}\left\{\left|\frac{1}{N}\underset{m\in \mathscr{I}^{c}}{\overset{}{\sum}}\frac{1}{R}\underset{u=1}{\overset{R}{\sum}}\hat{\xi}_{m,u}^{i;F}(\hat{A}_{m}^{i}-A_{m}^{i})\right|\geq\frac{\epsilon}{16}\right\}\label{eqt:Chebyshev Inequality quantity 2 term 1}\\
&\;+\mathbb{P}_{\hat{\rho}}\left\{\left|\frac{1}{N}\underset{m\in \mathscr{I}^{c}}{\overset{}{\sum}}A_{m}^{i}\frac{1}{R}\underset{u=1}{\overset{R}{\sum}}\left(\hat{\xi}_{m,u}^{i;F}-\xi_{m,u}^{i;F}\right)\right|\geq\frac{\epsilon}{16}\right\}.\label{eqt:Chebyshev Inequality quantity 2 term 2}
\end{align}
}\noindent
\end{subequations}
We can argue that \eqref{eqt:Chebyshev Inequality quantity 2 term 1} converges to $0$ in probability as $N\rightarrow \infty$ using similar arguments when we prove that (\ref{eqt:consistency check equation}c) converges to $0$ in probability. Simultaneously, we can argue \eqref{eqt:Chebyshev Inequality quantity 2 term 2} converges to $0$ in probability as $N\rightarrow \infty$ using similar arguments when we prove that (\ref{eqt:consistency check equation}b) converges to $0$ in probability. Consequently, we have $\hat{\kappa}_{R,N}^{i;F}-\kappa_{R,N}^{i;F}$ converges to $0$ in probability.

The proof is completed.